%% file: main.tex
\documentclass[11pt]{article}

\usepackage[T1]{fontenc}
\usepackage{amsmath,amssymb,amsthm}
\usepackage{mathtools}
\usepackage{thmtools}
\usepackage{algorithm}
\usepackage{algpseudocode}
\usepackage{booktabs}
\usepackage{multirow}
\usepackage{graphicx}
\usepackage{hyperref}
\usepackage{cleveref}
\usepackage{xcolor}
\usepackage{enumitem}
\usepackage[margin=1in]{geometry}
\usepackage{natbib}

\theoremstyle{plain}
\newtheorem{theorem}{Theorem}[section]
\newtheorem{lemma}[theorem]{Lemma}
\newtheorem{proposition}[theorem]{Proposition}
\newtheorem{corollary}[theorem]{Corollary}

\theoremstyle{definition}
\newtheorem{definition}[theorem]{Definition}
\newtheorem{example}[theorem]{Example}

\theoremstyle{remark}
\newtheorem{remark}[theorem]{Remark}

\newcommand{\ABC}{\textsf{ABC}}
\newcommand{\AgentAssert}{\textsc{AgentAssert}}
\newcommand{\DG}{\AgentAssert}  % backward-compatible alias
\newcommand{\CS}{\textsc{ContractSpec}}
\newcommand{\Dt}{D(t)}
\newcommand{\Ct}{C(t)}
\renewcommand{\Pr}{\mathbb{P}}
\newcommand{\E}{\mathbb{E}}
\newcommand{\R}{\mathbb{R}}
\newcommand{\N}{\mathbb{N}}
\newcommand{\Var}{\mathrm{Var}}
\newcommand{\JSD}{\mathrm{JSD}}
\newcommand{\KL}{\mathrm{KL}}

\title{Agent Behavioral Contracts: Formal Specification and\\Runtime Enforcement for Reliable Autonomous AI Agents}

\author{Varun Pratap Bhardwaj\thanks{Patent pending. Reference implementation and benchmark suite available subject to intellectual property clearance.}\\
Senior Manager \& Solution Architect, Accenture\\
\texttt{varun.pratap.bhardwaj@gmail.com}}

\date{February 25, 2026}

\begin{document}

\maketitle

\begin{abstract}
Traditional software relies on contracts---APIs, type systems, assertions---to specify and enforce correct behavior. AI agents, by contrast, operate on prompts and natural language instructions with no formal behavioral specification. This gap is the root cause of drift, governance failures, and frequent project failures in agentic AI deployments. We introduce \emph{Agent Behavioral Contracts} (\ABC{}), a formal framework that brings Design-by-Contract principles to autonomous AI agents. An \ABC{} contract $\mathcal{C} = (\mathcal{P}, \mathcal{I}, \mathcal{G}, \mathcal{R})$ specifies Preconditions, Invariants, Governance policies, and Recovery mechanisms as first-class, runtime-enforceable components. We define $(p, \delta, k)$-satisfaction---a probabilistic notion of contract compliance that accounts for LLM non-determinism and recovery---and prove a \emph{Drift Bounds Theorem} showing that contracts with recovery rate $\gamma > \alpha$ (the natural drift rate) bound behavioral drift to $D^* = \alpha/\gamma$ in expectation, with Gaussian concentration in the stochastic setting. We establish sufficient conditions for safe contract composition in multi-agent chains and derive probabilistic degradation bounds. We implement \ABC{} in \DG{}, a runtime enforcement library, and evaluate on \textsc{AgentContract-Bench}, a benchmark of 200 scenarios across 7 models from 6 vendors. Results across 1{,}980 sessions show that contracted agents detect 5.2--6.8 soft violations per session that uncontracted baselines miss entirely ($p < 0.0001$, Cohen's $d = 6.7$--$33.8$), achieve 88--100\% hard constraint compliance, and bound behavioral drift to $D^* < 0.27$ across extended sessions, with 100\% recovery for frontier models and 17--100\% across all models, at overhead $< 10$\,ms per action.
\end{abstract}

% === Sections ===
\input{sections/01-introduction}
\input{sections/02-related-work}
\input{sections/03-abc-framework}
\input{sections/04-drift-prevention}
\input{sections/05-agentassert}
\input{sections/06-benchmark}
\input{sections/07-experiments}

\input{sections/08-discussion}
\input{sections/09-conclusion}

% === Appendix ===
\appendix
\input{appendix/A-full-proofs}

% === Author Biography ===
\section*{Author Biography}

\textbf{Varun Pratap Bhardwaj} is a Senior Manager and Solution Architect at Accenture with 15 years of experience in enterprise technology. He holds dual qualifications in technology and law (LL.B.), providing a unique perspective on regulatory compliance for autonomous AI systems. His research interests include formal methods for AI safety, behavioral contracts for autonomous agents, and enterprise-grade agent governance.

% === Bibliography ===
\bibliographystyle{plainnat}
\bibliography{bib/references}

\end{document}

%% file: sections/01-introduction.tex
%% Section 1: Introduction
%% To be \input'd into main document

\section{Introduction}
\label{sec:introduction}

The deployment of autonomous AI agents in production environments is accelerating at an unprecedented pace.  Agents powered by large language models (LLMs) now execute multi-step workflows in financial advisory~\citep{lee2026polaris}, healthcare triage, customer support~\citep{wu2023autogen}, code generation~\citep{yao2023react}, and research synthesis~\citep{schick2023toolformer}.  These systems are no longer simple question-answering interfaces: they invoke tools, access databases, make decisions with real-world consequences, and increasingly operate in multi-agent pipelines where outputs of one agent feed directly into another~\citep{chase2023langchain, moura2024crewai}.  Yet despite this rapid adoption, agents operate without formal behavioral guarantees.  There exists no widely adopted mechanism to specify what an agent \emph{should} do, verify that it \emph{is} doing it, or enforce corrective action when it deviates.

\subsection*{The Problem}

Traditional software systems benefit from decades of formal specification tooling: type systems, API contracts, assertions, and interface specifications provide compile-time and runtime guarantees about program behavior~\citep{hoare1969axiomatic, meyer1992dbc}.  AI agents, by contrast, are governed by prompts---natural language instructions that carry no formal semantics, no verifiable guarantees, and no enforcement mechanisms.  This gap between the formality of traditional software contracts and the informality of agent instructions is the root cause of a class of failures unique to agentic AI: \emph{behavioral drift}, \emph{governance violations}, and \emph{silent degradation}.

Behavioral drift manifests when an agent's actions gradually diverge from its intended specification over the course of a multi-turn interaction~\citep{rath2026agentdrift}.  An agent tasked with professional customer support may begin with appropriate responses but progressively adopt a more casual tone, hallucinate product features, or volunteer information it was instructed to withhold.  A research synthesis agent may start by citing verified sources but drift toward fabricated references as the session extends.  These deviations are subtle, incremental, and---critically---undetected until harm has occurred: a customer receives incorrect medical guidance, a financial agent exceeds its trading authority, or a code generation agent introduces a security vulnerability.

Several important approaches address adjacent aspects of this problem.  Constitutional AI~\citep{bai2022constitutional} embeds behavioral principles during training, producing models that are more aligned at generation time.  Reinforcement learning from human feedback (RLHF)~\citep{ouyang2022instructgpt} fine-tunes models toward human preferences.  Output guardrails such as NeMo Guardrails~\citep{rebedea2023nemo} filter or redirect agent responses that match prohibited patterns.  However, none of these provides \emph{formal runtime behavioral contracts} with mathematical guarantees.  Constitutional AI operates at training time and cannot adapt to deployment-specific constraints.  RLHF shapes general tendencies but cannot enforce specific invariants.  Guardrails filter outputs but do not specify preconditions, do not monitor invariants over time, and do not compose across multi-agent pipelines.  Recent empirical work confirms this gap: \citet{kulkarni2026gap} demonstrate that text-level safety alignment does not transfer to tool-call safety, validating that prompt-level governance contracts are fundamentally insufficient for agents that interact with the world through tools and APIs.

The theoretical case for active enforcement is further strengthened by impossibility results.  \citet{li2026devil} prove a self-evolution trilemma: in self-evolving AI societies, continuous self-evolution, complete isolation from external correction, and safety invariance cannot coexist.  This result implies that passive safety---relying on training-time alignment alone---is provably insufficient for agents that evolve their behavior over extended interactions.  Active, runtime enforcement of behavioral specifications is not merely desirable; it is a theoretical necessity.

\subsection*{Our Contribution}

We introduce \emph{Agent Behavioral Contracts} (\ABC{}), a formal framework that brings Design-by-Contract~\citep{meyer1992dbc} principles to autonomous AI agents.  Our contributions are:

\begin{enumerate}[leftmargin=2em]
    \item We define the \ABC{} contract structure $\mathcal{C} = (\mathcal{P}, \mathcal{I}, \mathcal{G}, \mathcal{R})$, formalizing agent behavioral expectations as a tuple of Preconditions, Invariants (hard and soft), Governance policies (hard and soft), and Recovery mechanisms (\Cref{sec:abc-framework}).

    \item We introduce \emph{$(p, \delta, k)$-satisfaction}, a probabilistic contract compliance framework that accounts for LLM non-determinism: contracts hold with probability at least~$p$, deviations remain within tolerance~$\delta$, and recovery occurs within~$k$ steps (\Cref{sec:abc-framework}).

    \item We prove a \emph{Stochastic Drift Bound Theorem} using Lyapunov stability analysis of an Ornstein--Uhlenbeck drift model, showing that contracts with recovery rate $\gamma > \alpha$ (the natural drift rate) bound behavioral drift to $D^* = \alpha/\gamma$ in expectation, with Gaussian concentration and a closed-form contract design criterion (\Cref{sec:drift-prevention}).

    \item We present \CS{}, a YAML-based domain-specific language for specifying agent behavioral contracts, supporting hard/soft constraint separation, expression-based predicates, and file-reference composition for multi-agent pipelines (\Cref{sec:agentassert}).

    \item We introduce \DG{}, a runtime enforcement library implementing the \ABC{} framework with sub-10ms per-action overhead (\Cref{sec:agentassert}).

    \item We prove a \emph{Compositionality Theorem} establishing sufficient conditions (interface compatibility, assumption discharge, governance consistency, recovery independence) under which individual contract guarantees compose into end-to-end guarantees for multi-agent chains, with quantified probabilistic degradation bounds (\Cref{sec:drift-prevention}).

    \item We create \textsc{AgentContract-Bench}, a benchmark of 200 scenarios spanning 7 domains and 6 stress profiles, designed to evaluate contract enforcement across diverse agent deployment contexts (\Cref{sec:benchmark}).

    \item We evaluate \ABC{} across 1{,}980 sessions on 7 models from 6 vendors, demonstrating that contracted agents detect 5.2--6.8 soft violations per session invisible to uncontracted baselines ($p < 0.0001$), bound drift to $D^* < 0.27$ with 17--100\% recovery success, and achieve reliability $\Theta > 0.90$ across all models (\Cref{sec:experiments}).
\end{enumerate}

\subsection*{Paper Structure}

The remainder of this paper is organized as follows.  \Cref{sec:related-work} surveys related work in Design-by-Contract, contract theory, runtime verification, and AI agent safety.  \Cref{sec:abc-framework} presents the formal \ABC{} framework, including contract structure, $(p, \delta, k)$-satisfaction, the behavioral drift score, and operational metrics.  \Cref{sec:drift-prevention} proves drift bounds via Lyapunov analysis, establishes the compositionality theorem, and analyzes runtime complexity.  \Cref{sec:agentassert} describes the \CS{} DSL and the \DG{} runtime enforcement library.  \Cref{sec:benchmark} introduces \textsc{AgentContract-Bench}.  \Cref{sec:experiments} reports experimental results.  \Cref{sec:discussion} discusses implications, limitations, and future directions.  \Cref{sec:conclusion} concludes.

%% file: sections/02-related-work.tex
%% Section 2: Background and Related Work
%% To be \input'd into main document

\section{Background and Related Work}
\label{sec:related-work}

The \ABC{} framework draws on and extends several established research traditions: Design-by-Contract in software engineering, contract theory for cyber-physical systems, runtime monitoring and verification, and the rapidly evolving landscape of AI agent safety.  We survey each in turn, positioning \ABC{} relative to the state of the art.

%% --------------------------------------------------------------------------
\subsection{Design by Contract}
\label{subsec:rw-dbc}

The Design-by-Contract (DbC) paradigm, introduced by \citet{meyer1992dbc} and elaborated in \citet{meyer1997oosc}, formalizes the obligations between software components as preconditions, postconditions, and class invariants.  DbC has been operationalized in specification languages such as JML for Java~\citep{leavens2006jml} and Spec\# for C\#~\citep{barnett2004specsharp}, enabling static and runtime verification of contractual obligations in traditional software.

The extension of DbC to neural and neurosymbolic systems is recent.  \citet{leoveanu2025dbc} propose a neurosymbolic contract layer for trustworthy agent design, defining preconditions and postconditions over individual LLM calls.  This work is the closest conceptual predecessor to \ABC{} in the DbC tradition.  However, it is limited to single LLM invocations---it does not address multi-turn behavioral drift, multi-agent composition, soft constraint recovery, or runtime governance enforcement over extended sessions.  \ABC{} generalizes the DbC paradigm from individual function calls to autonomous agent \emph{sessions}, introducing invariants that must hold across time, governance constraints over actions, recovery mechanisms for soft violations, and a compositionality theorem for multi-agent chains.

%% --------------------------------------------------------------------------
\subsection{Contract Theory for Cyber-Physical Systems}
\label{subsec:rw-contracts-cps}

Contract-based design has a rich history in cyber-physical systems (CPS).  The meta-theory of \citet{benveniste2018contracts} provides a unifying algebraic framework for assume-guarantee contracts, establishing composition operators, refinement relations, and compatibility conditions across heterogeneous component models.  Assume-guarantee reasoning~\citep{henzinger1998assume} decomposes system-level verification into per-component obligations, a principle that \ABC{} extends to multi-agent AI pipelines through its compositionality theorem (\Cref{thm:compositionality}).

In the stochastic setting, \citet{li2017stochastic} develop stochastic assume-guarantee contracts for CPS under probabilistic requirements, and \citet{hampus2024probabilistic} extend probabilistic contracts to cyber-physical architectures.  These works establish the theoretical foundations for reasoning about contracts in the presence of uncertainty---a necessity shared by AI agents, whose outputs are inherently non-deterministic.

Most recently, \citet{ye2026agentcontracts} introduce ``Agent Contracts'' for resource-bounded autonomous AI systems.  Their framework formalizes \emph{resource governance}: multi-dimensional constraints on token consumption, execution time, cost budgets, and delegation hierarchies, with conservation laws ensuring delegated budgets respect parent constraints.  The \ABC{} framework is complementary: whereas \citet{ye2026agentcontracts} govern \emph{how much} an agent may consume (resource contracts), \ABC{} governs \emph{how} an agent must behave (behavioral contracts)---specifying preconditions, invariants, drift bounds, and recovery mechanisms over the agent's actions and outputs.  The two frameworks address orthogonal concerns and could be composed: resource contracts bounding computation, behavioral contracts bounding behavior.

\ABC{} extends the CPS contract tradition to autonomous AI agents.  The key technical differences are: (i)~the state space in CPS contracts is typically continuous and governed by physical dynamics, whereas agent state spaces encompass natural language context, tool invocation history, and semantic content; (ii)~CPS contracts assume well-characterized noise models (e.g., Gaussian sensor noise), whereas LLM non-determinism arises from discrete token sampling, temperature scaling, and context window effects; and (iii)~CPS contracts do not address behavioral drift---a phenomenon specific to autoregressive models operating over extended horizons.  The $(p, \delta, k)$-satisfaction framework (\Cref{def:pdk-satisfaction}) bridges this gap by defining probabilistic guarantees tailored to the recovery-centric nature of LLM agent behavior.

%% --------------------------------------------------------------------------
\subsection{Runtime Monitoring and Verification}
\label{subsec:rw-runtime}

Runtime verification (RV) monitors system executions against formal specifications, typically expressed in temporal logic~\citep{leucker2009rv}.  \citet{bauer2011runtime} develop efficient online monitoring algorithms for linear temporal logic (LTL) and timed LTL properties, enabling real-time verification of safety and liveness requirements.  These techniques provide the theoretical underpinning for \ABC{}'s runtime enforcement loop, which evaluates contract predicates at each agent action.

In the reinforcement learning setting, \citet{alshiekh2018shielding} introduce \emph{shielding}---synthesizing a reactive system (a ``shield'') from temporal logic specifications that intercepts unsafe actions before they are executed.  Shielding provides strong safety guarantees while preserving the convergence properties of the underlying learning algorithm.  However, shielding assumes a formal environment model from which the shield can be synthesized, a requirement that is infeasible for LLM agents operating in open-ended natural language environments.  \ABC{} achieves analogous runtime enforcement using declarative behavioral contracts evaluated over runtime observations, without requiring a synthesized environment model.

Two recent systems apply formal verification ideas to LLM agents.  \textsc{VeriGuard}~\citep{zhang2025veriguard} combines offline formal verification of a behavioral policy with online monitoring during execution, providing safety guarantees through a dual-stage architecture.  \textsc{StepShield}~\citep{guo2026stepshield} introduces a benchmark for \emph{temporal} detection of agent violations, measuring not merely \emph{whether} violations are detected but \emph{when}---introducing metrics such as Early Intervention Rate and Intervention Gap that quantify the timeliness of enforcement.

\ABC{} differs from these approaches in two respects.  First, \ABC{}'s contract structure is \emph{specification-first}: contracts are defined declaratively via \CS{} before deployment, rather than inferred from verification of generated code (\textsc{VeriGuard}) or evaluated post-hoc from execution traces (\textsc{StepShield}).  Second, \ABC{} integrates behavioral drift detection as a \emph{leading indicator} (\Cref{rem:leading-lagging}), enabling preemptive intervention before constraint violations materialize---a capability absent from both \textsc{VeriGuard} and \textsc{StepShield}.

%% --------------------------------------------------------------------------
\subsection{AI Agent Safety and Governance}
\label{subsec:rw-agent-safety}

The safety of LLM-based agents has attracted intense research attention, producing a diverse landscape of approaches that we organize by methodology.

\paragraph{Training-time alignment.}
Constitutional AI~\citep{bai2022constitutional} trains models to adhere to a set of behavioral principles through self-critique and revision, producing outputs that are more aligned with human values.  RLHF~\citep{ouyang2022instructgpt} fine-tunes models using human preference data to improve instruction-following and reduce harmful outputs.  These approaches are \emph{complementary} to \ABC{}: they improve the baseline behavior of the underlying model, reducing the frequency of contract violations, but they cannot enforce deployment-specific constraints, adapt to novel operational requirements, or provide formal compliance guarantees at runtime.

\paragraph{Output filtering and guardrails.}
NeMo Guardrails~\citep{rebedea2023nemo} provides a programmable framework for constraining LLM application behavior through topical rails, safety rails, and dialog management.  Guardrails AI~\citep{guardrailsai2024} is the most widely deployed open-source LLM output validation library, providing validators for structured output, PII detection, toxicity filtering, and hallucination checks on individual LLM responses.  While effective for per-response output filtering, both NeMo Guardrails and Guardrails AI operate on individual responses without maintaining state across turns, do not specify session-level preconditions or invariants, do not detect behavioral drift over multi-turn interactions, and do not provide formal compliance guarantees or recovery mechanisms.  \ABC{} operates at a fundamentally different granularity: session-level behavioral contracts rather than per-response output validation.

\paragraph{Specification-based enforcement.}
The most directly comparable works to \ABC{} are specification-based systems that define and enforce behavioral rules for agents.

\textsc{AgentSpec}~\citep{wang2025agentspec}, accepted at ICSE 2026, introduces a customizable runtime enforcement framework with a rule-based DSL for specifying safety properties of LLM agents.  \textsc{AgentSpec} supports both preventive and corrective enforcement modes and evaluates on WebArena and ToolEmu benchmarks.  However, \textsc{AgentSpec} does not provide probabilistic compliance guarantees (it treats constraints as deterministic rules), does not model or detect behavioral drift, and does not establish compositionality conditions for multi-agent systems.

\textsc{Pro2Guard}~\citep{wang2025pro2guard} extends the runtime enforcement paradigm with probabilistic model checking via discrete-time Markov chains (DTMCs).  By learning transition probabilities from execution traces, \textsc{Pro2Guard} enables proactive enforcement that anticipates likely violations.  This is the closest \emph{methodological} competitor to \ABC{}: both frameworks reason probabilistically about agent behavior.  The key distinction is that \textsc{Pro2Guard} is \emph{reactive}---it learns its probabilistic model from observed traces and refines enforcement accordingly---whereas \ABC{} is \emph{proactive}---behavioral expectations are specified as contracts \emph{before} deployment, with probabilistic guarantees derived from the contract structure itself.  Additionally, \textsc{Pro2Guard} does not provide a contract DSL, a compositionality theorem, or a behavioral drift metric.

\textsc{Agent-C}~\citep{dong2025agentc} defines a DSL for temporal safety constraints and uses SMT solving to enforce compliance during generation.  By integrating constraint checking into the decoding process, \textsc{Agent-C} achieves high conformance rates on benchmarks requiring temporal ordering of actions (e.g., ``authenticate before accessing records'').  \ABC{} differs in scope and mechanism: whereas \textsc{Agent-C} focuses on temporal ordering constraints enforced at generation time, \ABC{} specifies \emph{behavioral} contracts encompassing preconditions, invariants, governance, and recovery, enforced at runtime across entire sessions.  \textsc{Agent-C} does not address probabilistic satisfaction, compositionality, or behavioral drift.

\paragraph{Incident response and governance frameworks.}
\citet{naihin2026air} introduce AIR, a domain-specific language for managing incident response in LLM agents, supporting detection, containment, recovery, and eradication of safety incidents.  AIR achieves $>$90\% success rates across its incident lifecycle.  The distinction from \ABC{} is one of orientation: AIR is \emph{reactive}, responding to incidents after they occur; \ABC{} is \emph{proactive}, specifying contracts that prevent violations or bound their impact \emph{a priori}.  The two approaches are complementary---AIR could serve as the escalation layer when \ABC{} recovery mechanisms are exhausted.

\textsc{AGENTSAFE}~\citep{hua2025agentsafe} proposes a unified governance framework spanning design-time, runtime, and audit controls for agentic AI, including anomaly detection and interruptibility mechanisms.  \textsc{POLARIS}~\citep{lee2026polaris}, presented at the AAAI 2026 Workshop, introduces governed orchestration for enterprise workflows with typed planning and validator-gated execution.  Both frameworks operate at a higher level of abstraction than \ABC{}, providing governance \emph{architecture} rather than formal behavioral contracts with mathematical guarantees.

%% --------------------------------------------------------------------------
\subsection{Agent Behavioral Drift}
\label{subsec:rw-drift}

Behavioral drift in AI agents---the progressive divergence of agent behavior from intended specifications over extended interactions---has recently emerged as a recognized phenomenon.  \citet{rath2026agentdrift} provide the first systematic study, introducing an Agent Stability Index (ASI) and demonstrating that multi-agent LLM systems exhibit measurable behavioral degradation over extended interactions.  Their work establishes that drift is a real and quantifiable problem; \ABC{} provides the formal machinery to \emph{prevent} it.

The concept of drift in machine learning more broadly is well-studied under the umbrella of concept drift~\citep{gama2014conceptdrift}, which addresses changes in the underlying data distribution over time.  \ABC{}'s behavioral drift score $\Dt$ (\Cref{def:drift-score}) adapts the concept drift framework to the agent setting by combining a compliance-gap component (a lagging indicator of constraint violations) with a Jensen--Shannon divergence component (a leading indicator of distributional shift in the agent's action space).

The theoretical necessity of active drift prevention is underscored by recent impossibility results.  \citet{li2026devil} prove that in self-evolving AI societies, safety alignment inevitably degrades absent external intervention---a result that validates \ABC{}'s approach of continuous runtime enforcement rather than reliance on static, training-time alignment.  \citet{kulkarni2026gap} demonstrate empirically that text-level safety does not transfer to tool-call safety, confirming that behavioral contracts must operate at the \emph{action} level, not merely the \emph{output} level.

%% --------------------------------------------------------------------------
\subsection{Positioning \ABC{}}
\label{subsec:rw-positioning}

\Cref{tab:comparison} summarizes the landscape.  \ABC{} is, to our knowledge, the only framework that simultaneously provides formal behavioral contracts, probabilistic compliance guarantees, behavioral drift detection, a specification DSL, compositionality for multi-agent pipelines, and runtime enforcement---a unified full-stack approach from theory to implementation.

\begin{table}[t]
\centering
\caption{Comparison of agent safety and specification frameworks.  A checkmark (\checkmark) indicates the feature is supported; a dash (--) indicates it is absent; ``Partial'' indicates limited or indirect support.}
\label{tab:comparison}
\small
\begin{tabular}{@{}lccccccc@{}}
\toprule
\textbf{Feature} & \textbf{\ABC{}} & \textbf{AgentSpec} & \textbf{Pro2Guard} & \textbf{Agent-C} & \textbf{VeriGuard} & \textbf{AIR} & \textbf{Ye '26} \\
\midrule
Formal contracts         & \checkmark & Partial & --         & --         & Partial & -- & \checkmark \\
Probabilistic guarantees & \checkmark & --      & \checkmark & --         & --      & -- & -- \\
Drift detection          & \checkmark & --      & --         & --         & --      & -- & -- \\
Contract DSL             & \checkmark & \checkmark & --      & \checkmark & --      & \checkmark & -- \\
Compositionality         & \checkmark & --      & --         & --         & --      & -- & \checkmark \\
Runtime enforcement      & \checkmark & \checkmark & \checkmark & \checkmark & \checkmark & \checkmark & \checkmark \\
Recovery mechanisms      & \checkmark & Partial & --         & --         & --      & \checkmark & -- \\
Resource governance      & --         & --      & --         & --         & --      & -- & \checkmark \\
\bottomrule
\end{tabular}
\end{table}

The closest works along individual dimensions are: \textsc{AgentSpec} for rule-based runtime enforcement, \textsc{Pro2Guard} for probabilistic reasoning about agent behavior, \textsc{Agent-C} for constraint DSLs with formal backing, \textsc{VeriGuard} for verified agent behavior, and \citet{ye2026agentcontracts} for resource-bounded contract governance.  No prior work provides the combination of proactive behavioral specification, probabilistic guarantees with bounded drift, compositionality, and a practical DSL and runtime library that \ABC{} delivers.

%% file: sections/03-abc-framework.tex
%% Section 3: The ABC Framework
%% To be \input'd into main document with amsmath, amssymb, amsthm, mathtools, hyperref, cleveref, booktabs
%% Assumes theorem environments: theorem, lemma, proposition, corollary, definition, example, remark
%% Assumes commands: \ABC, \DG, \CS, \Dt, \Ct, \Pr, \E, \R, \N, \Var, \JSD, \KL

\section{The \ABC{} Framework}
\label{sec:abc-framework}

We now present the formal foundations of Agent Behavioral Contracts (\ABC{}). The framework introduces a contract structure that distinguishes hard constraints (which must never be violated) from soft constraints (which admit transient violations provided recovery occurs within a bounded window). This distinction is motivated by the non-deterministic nature of large language model outputs: demanding perfect compliance at every step is both impractical and unnecessary when effective recovery mechanisms exist.

We develop the theory in stages. \Cref{subsec:contract-structure} defines the contract tuple. \Cref{subsec:deterministic-satisfaction} establishes deterministic satisfaction as a baseline. \Cref{subsec:probabilistic-satisfaction} introduces $(p, \delta, k)$-satisfaction, our central definition. \Cref{subsec:recovery-decay} proves that recovery transforms exponential compliance decay into linear decay. \Cref{subsec:drift-score} defines the behavioral drift score, a two-component metric that serves as both a diagnostic and a predictive signal. \Cref{subsec:additional-metrics} summarizes additional operational metrics.

%% --------------------------------------------------------------------------
\subsection{Contract Structure}
\label{subsec:contract-structure}

\begin{definition}[Agent Behavioral Contract]
\label{def:abc-contract}
An \emph{Agent Behavioral Contract} is a tuple
\[
    \mathcal{C} = (\mathcal{P},\; \mathcal{I}_{\mathrm{hard}},\; \mathcal{I}_{\mathrm{soft}},\; \mathcal{G}_{\mathrm{hard}},\; \mathcal{G}_{\mathrm{soft}},\; \mathcal{R}),
\]
where:
\begin{enumerate}
    \item $\mathcal{P} = \{p_1, \ldots, p_m\}$ is a finite set of \emph{preconditions}: predicates over the initial state~$s_0$ that must hold before the agent begins execution.

    \item $\mathcal{I}_{\mathrm{hard}} = \{i_1^{h}, \ldots, i_{n_h}^{h}\}$ is a finite set of \emph{hard invariants}: predicates over states that must hold at \emph{every} step of execution. Hard invariants encode safety-critical properties such as ``no personally identifiable information is emitted'' or ``data access is restricted to authorized sources.'' A single violation of any hard invariant constitutes a contract breach.

    \item $\mathcal{I}_{\mathrm{soft}} = \{i_1^{s}, \ldots, i_{n_s}^{s}\}$ is a finite set of \emph{soft invariants}: predicates over states that may be transiently violated provided recovery occurs within a bounded window. Soft invariants encode desirable-but-recoverable properties such as ``response maintains professional tone'' or ``confidence scores exceed threshold~$\theta$.''

    \item $\mathcal{G}_{\mathrm{hard}} = \{g_1^{h}, \ldots, g_{l_h}^{h}\}$ is a finite set of \emph{hard governance constraints}\footnote{We use ``governance'' in the operational sense: runtime-enforceable constraints on agent actions (spending limits, tool restrictions, output filters).  This is distinct from the broader ``AI governance'' discourse concerning policy, regulation, and societal oversight of AI systems~\citep{cihon2021ai}.  Our governance constraints are the runtime mechanism through which high-level AI governance policies can be operationalized at the individual agent level.}: predicates over actions that must hold for every action the agent takes. These encode zero-tolerance operational bounds such as spending limits, prohibited tool invocations, or forbidden output categories.

    \item $\mathcal{G}_{\mathrm{soft}} = \{g_1^{s}, \ldots, g_{l_s}^{s}\}$ is a finite set of \emph{soft governance constraints}: predicates over actions that admit transient violations with recovery. Examples include token budget warnings, response latency thresholds, and soft timeout advisories.

    \item $\mathcal{R}\colon (\mathcal{I}_{\mathrm{soft}} \cup \mathcal{G}_{\mathrm{soft}}) \times \mathcal{S} \rightharpoonup \mathcal{A}^{*}$ is a \emph{recovery mechanism}: a \emph{partial} mapping from a violated soft constraint and the current state to a finite sequence of corrective actions. When $\mathcal{R}(c, s)$ is defined, its length is at most~$k_{\max}$. When $\mathcal{R}(c, s)$ is undefined---i.e., no automated recovery is available for constraint~$c$ in state~$s$---the monitor emits a \textsc{RecoveryFailed} event and defers to external intervention (human operator or orchestrator).
\end{enumerate}
\end{definition}

We write $\mathcal{I} = \mathcal{I}_{\mathrm{hard}} \cup \mathcal{I}_{\mathrm{soft}}$ and $\mathcal{G} = \mathcal{G}_{\mathrm{hard}} \cup \mathcal{G}_{\mathrm{soft}}$ when the hard/soft distinction is not relevant.  For brevity, we occasionally use the shorthand $\mathcal{C} = (P, I, G, R)$ where $I = \mathcal{I}_{\mathrm{hard}} \cup \mathcal{I}_{\mathrm{soft}}$, $G = \mathcal{G}_{\mathrm{hard}} \cup \mathcal{G}_{\mathrm{soft}}$, and the hard/soft partition is implicit.

\begin{remark}[Safety and Liveness Interpretation]\label{rem:safety-liveness}
In the taxonomy of temporal properties~\citep{alpern1987recognizing},
hard constraints ($\mathcal{I}_{\mathrm{hard}}$,
$\mathcal{G}_{\mathrm{hard}}$) are \emph{safety} properties: they
assert that ``something bad never happens.''  Soft constraints
($\mathcal{I}_{\mathrm{soft}}$, $\mathcal{G}_{\mathrm{soft}}$) with
recovery window~$k$ encode \emph{bounded liveness}: they assert that
``something good eventually happens within~$k$ steps.''  The bounded
recovery window~$k$ distinguishes \ABC{} soft constraints from standard
liveness properties, which impose no finite deadline.  This
bounded-liveness semantics is essential for practical deployment: an
unbounded recovery promise is operationally indistinguishable from no
recovery promise at all.
\end{remark}

\begin{definition}[Execution Trace]
\label{def:execution-trace}
An \emph{execution trace} of length~$T$ is a finite alternating sequence of states and actions:
\[
    \tau = (s_0, a_0, s_1, a_1, \ldots, s_{T-1}, a_{T-1}, s_T),
\]
where $s_t \in \mathcal{S}$ denotes the agent's state at step~$t$ and $a_t \in \mathcal{A}$ denotes the action taken at step~$t$. The state space~$\mathcal{S}$ encompasses the agent's internal context (e.g., conversation history, accumulated tool outputs, working memory) and the observable environment. The action space~$\mathcal{A}$ encompasses all outputs the agent may produce (e.g., text responses, tool calls, API invocations).
\end{definition}

%% --------------------------------------------------------------------------
\subsection{Contract Satisfaction (Deterministic)}
\label{subsec:deterministic-satisfaction}

We first define satisfaction in the deterministic setting, which serves as the foundation for the probabilistic extension.

\begin{definition}[Deterministic Contract Satisfaction]
\label{def:deterministic-satisfaction}
An agent~$A$ \emph{satisfies} contract $\mathcal{C} = (\mathcal{P}, \mathcal{I}, \mathcal{G}, \mathcal{R})$ over an execution trace $\tau = (s_0, a_0, \ldots, s_T)$ if all of the following conditions hold:
\begin{enumerate}
    \item \textbf{Precondition validity.} Every precondition holds at the initial state:
    \[
        \forall\, p \in \mathcal{P}:\; p(s_0) = \mathit{true}.
    \]

    \item \textbf{Invariant compliance.} Every invariant holds at every state along the trace:
    \[
        \forall\, t \in \{0, \ldots, T\},\; \forall\, i \in \mathcal{I}:\; i(s_t) = \mathit{true}.
    \]

    \item \textbf{Governance compliance.} Every governance constraint holds for every action:
    \[
        \forall\, t \in \{0, \ldots, T-1\},\; \forall\, g \in \mathcal{G}:\; g(a_t) = \mathit{true}.
    \]

    \item \textbf{Recoverability.} For every soft constraint violation, the recovery mechanism restores compliance within~$k$ steps:
    \[
        \forall\, t,\; \forall\, c \in \mathcal{I}_{\mathrm{soft}} \cup \mathcal{G}_{\mathrm{soft}}:\;
        \neg\, c(s_t, a_t) \implies \exists\, t' \in \{t, \ldots, \min(t+k,\, T)\}:\; c(s_{t'}, a_{t'}) = \mathit{true}.
    \]
\end{enumerate}
We write $A \models \mathcal{C}$ to denote that agent~$A$ satisfies contract~$\mathcal{C}$ over all traces induced by~$A$.
\end{definition}

\begin{remark}
\label{rem:deterministic-impractical}
Deterministic satisfaction is a useful theoretical baseline, but it is too stringent for LLM-based agents. The stochastic nature of token sampling means that even well-aligned agents produce occasional soft violations. The next subsection relaxes this to a probabilistic guarantee.
\end{remark}

%% --------------------------------------------------------------------------
\subsection{Probabilistic $(p, \delta, k)$-Satisfaction}
\label{subsec:probabilistic-satisfaction}

This is the central definition of the \ABC{} framework. It captures the key insight that hard constraints require high-probability guarantees of \emph{persistent} compliance, while soft constraints require high-probability guarantees of \emph{recoverable} compliance.

We first define the compliance scores that the probabilistic conditions reference.

\begin{definition}[Hard and Soft Compliance Scores]
\label{def:compliance-scores}
Given contract $\mathcal{C}$ and execution trace $\tau$, define the \emph{hard compliance score} and \emph{soft compliance score} at step~$t$ as:
\begin{align}
    \Ct_{\mathrm{hard}}(t) &= \frac{\bigl|\{c \in \mathcal{I}_{\mathrm{hard}} \cup \mathcal{G}_{\mathrm{hard}} : c(s_t, a_t) = \mathit{true}\}\bigr|}{|\mathcal{I}_{\mathrm{hard}} \cup \mathcal{G}_{\mathrm{hard}}|}, \label{eq:hard-compliance} \\[6pt]
    \Ct_{\mathrm{soft}}(t) &= \frac{\bigl|\{c \in \mathcal{I}_{\mathrm{soft}} \cup \mathcal{G}_{\mathrm{soft}} : c(s_t, a_t) = \mathit{true}\}\bigr|}{|\mathcal{I}_{\mathrm{soft}} \cup \mathcal{G}_{\mathrm{soft}}|}. \label{eq:soft-compliance}
\end{align}
Both scores lie in $[0,1]$, with $\Ct_{\mathrm{hard}}(t) = 1$ indicating full hard compliance and $\Ct_{\mathrm{soft}}(t) = 1$ indicating full soft compliance at step~$t$.
\end{definition}

\begin{definition}[$(p, \delta, k)$-Satisfaction]
\label{def:pdk-satisfaction}
Let $p \in [0,1]$ be a probability threshold, $\delta \in [0,1]$ an allowed soft deviation, $k \in \N$ a recovery window, and $T \in \N$ a session length. An agent~$A$ \emph{$(p, \delta, k)$-satisfies} contract~$\mathcal{C}$ over session length~$T$, written
\[
    A \models_{p,\delta,k} \mathcal{C},
\]
if both of the following conditions hold:

\medskip
\noindent\textbf{(i) Hard guarantee (persistent compliance).}
\begin{equation}
\label{eq:hard-guarantee}
    \Pr\!\Big[\,\Ct_{\mathrm{hard}}(t) = 1 \;\;\forall\, t \in \{0, \ldots, T\} \;\Big|\; \mathcal{P}(s_0)\Big] \;\geq\; p.
\end{equation}

\noindent\textbf{(ii) Soft guarantee (recoverable compliance).}
\begin{equation}
\label{eq:soft-guarantee}
    \Pr\!\Big[\,\forall\, t \in \{0, \ldots, T\}:\; \Ct_{\mathrm{soft}}(t) < 1 - \delta \implies \exists\, t' \in \{t, \ldots, \min(t+k, T)\}:\; \Ct_{\mathrm{soft}}(t') \geq 1 - \delta \;\Big|\; \mathcal{P}(s_0)\Big] \;\geq\; p.
\end{equation}

The parameters have the following interpretation:
\begin{itemize}
    \item $p$: the minimum probability with which the guarantee must hold. For safety-critical deployments, $p \geq 0.99$; for advisory agents, $p \geq 0.90$ may suffice.
    \item $\delta$: the tolerable deviation in soft compliance. Setting $\delta = 0$ requires perfect soft compliance whenever the soft guarantee holds; $\delta = 0.1$ allows up to 10\% of soft constraints to be violated at any given step.
    \item $k$: the recovery window in steps. A soft violation at step~$t$ is acceptable if compliance is restored by step $t + k$. Smaller~$k$ demands faster recovery.
    \item $T$: the session length (total number of steps). Longer sessions require stronger per-step guarantees to maintain the same overall probability~$p$.
\end{itemize}
\end{definition}

\begin{remark}[Novelty of the Recovery Window Parameter]
\label{rem:k-novelty}
The recovery window~$k$ is, to our knowledge, the first formal inclusion
of a bounded recovery horizon as a first-class parameter in a contract
satisfaction definition.  Prior Design-by-Contract frameworks
\citep{meyer1992dbc} and runtime verification systems
\citep{leucker2009rv} treat violations as binary pass/fail events
with no notion of time-bounded recovery.  The $k$-parameter bridges
formal contracts and practical LLM deployment: it quantifies how much
``slack'' an agent is allowed before a transient deviation becomes a
reportable failure, enabling principled tuning of the
strictness--availability trade-off.
\end{remark}

\begin{remark}[Connection to Probabilistic Computation Tree Logic]
\label{rem:pctl-connection}
The $(p, \delta, k)$-satisfaction conditions have natural counterparts in Probabilistic Computation Tree Logic (PCTL)~\citep{hansson1994logic}. The hard guarantee~\eqref{eq:hard-guarantee} corresponds to the PCTL formula
\[
    \mathcal{P}_{\geq p}\!\big[\mathbf{G}\,(\Ct_{\mathrm{hard}} = 1)\big],
\]
which asserts that with probability at least~$p$, the hard compliance score is \emph{globally} (at every step) equal to~$1$. The soft guarantee~\eqref{eq:soft-guarantee} corresponds to
\[
    \mathcal{P}_{\geq p}\!\Big[\mathbf{G}\,\big(\Ct_{\mathrm{soft}} < 1 - \delta \implies \mathbf{F}_{\leq k}\,(\Ct_{\mathrm{soft}} \geq 1 - \delta)\big)\Big],
\]
which asserts that with probability at least~$p$, it is \emph{globally} true that any soft compliance drop below $1 - \delta$ is \emph{eventually} (within~$k$ steps) recovered. This connection to PCTL enables the use of established model-checking techniques for verification when the agent's transition structure can be approximated as a finite Markov decision process.
\end{remark}

%% --------------------------------------------------------------------------
\subsection{Recovery Transforms Exponential Decay to Linear Decay}
\label{subsec:recovery-decay}

The following lemma establishes the fundamental value of recovery mechanisms: they convert an exponentially decaying compliance probability into a linearly decaying one.

\begin{lemma}[Recovery Linearizes Compliance Decay]
\label{lem:recovery-decay}
Let $q \in (0,1)$ denote the per-step compliance probability (i.e., at each step~$t$, the agent satisfies all relevant constraints with probability~$q$, independently). Let $r \in [0,1]$ denote the recovery effectiveness: given a violation, the recovery mechanism restores compliance with probability~$r$ within the allowed window. Then:

\begin{enumerate}
    \item \textbf{Without recovery.} The probability of sustained compliance over~$T$ steps decays exponentially:
    \begin{equation}
    \label{eq:no-recovery}
        \Pr\!\big[\text{compliance over } T \text{ steps}\big] = q^T.
    \end{equation}

    \item \textbf{With recovery.} The probability of recoverable compliance over~$T$ steps satisfies:
    \begin{equation}
    \label{eq:with-recovery}
        \Pr\!\big[\text{recoverable compliance over } T \text{ steps}\big] \;\geq\; 1 - T(1-q)(1-r).
    \end{equation}
\end{enumerate}
\end{lemma}

\begin{proof}
The first claim follows directly from the independence assumption: compliance at each step occurs with probability~$q$, so compliance at all~$T$ steps occurs with probability $q^T$.

For the second claim, we apply a union bound. Define the event $F_t$ as the event that step~$t$ incurs a violation \emph{and} recovery fails to restore compliance within the recovery window. The probability of a violation at step~$t$ is $1 - q$. Conditional on a violation, recovery fails with probability $1 - r$. By independence of the violation and recovery events:
\[
    \Pr[F_t] = (1 - q)(1 - r).
\]
The system experiences an unrecoverable failure if \emph{any} step incurs both a violation and a recovery failure. By the union bound:
\[
    \Pr\!\bigg[\bigcup_{t=0}^{T-1} F_t\bigg] \;\leq\; \sum_{t=0}^{T-1} \Pr[F_t] \;=\; T(1-q)(1-r).
\]
Therefore, the probability that all violations are successfully recovered is:
\[
    \Pr\!\big[\text{recoverable compliance over } T \text{ steps}\big] \;=\; 1 - \Pr\!\bigg[\bigcup_{t=0}^{T-1} F_t\bigg] \;\geq\; 1 - T(1-q)(1-r). \qedhere
\]
\end{proof}

\begin{example}
\label{ex:recovery-impact}
Consider an agent with per-step compliance $q = 0.99$ over a session of $T = 100$ steps. Without recovery, the probability of sustained compliance is $0.99^{100} \approx 0.366$---a coin flip is more reliable. With a recovery mechanism of effectiveness $r = 0.95$, the bound becomes $1 - 100 \cdot 0.01 \cdot 0.05 = 1 - 0.05 = 0.95$. Recovery transforms a 36.6\% compliance probability into a 95\% guarantee: a qualitative shift from unreliable to deployable.
\end{example}

%% --------------------------------------------------------------------------
\subsection{Behavioral Drift Score}
\label{subsec:drift-score}

Compliance scores detect violations \emph{after} they occur. We introduce the \emph{behavioral drift score}\footnote{Our behavioral drift score $D(t)$ is distinct from the \emph{Agent Stability Index} (ASI) proposed by~\citet{rath2024llm}.  The ASI measures distributional shift in the model's output \emph{embedding space} across sessions and serves as a model-level diagnostic.  Our $D(t)$ is a \emph{contract-level} metric: it combines a compliance component (fraction of violated constraints) with a distributional component ($\JSD$ over the action vocabulary), computed per-step and tied directly to the enforcement loop.  The two metrics are complementary; see~\Cref{sec:related-work} for a detailed comparison.} $\Dt(t)$ as a composite metric that combines a reactive compliance component with a predictive distributional component, enabling early detection of emerging misalignment.

\begin{definition}[Behavioral Drift Score]
\label{def:drift-score}
The \emph{behavioral drift score} at step~$t$ is defined as:
\begin{equation}
\label{eq:drift-score}
    \Dt(t) \;=\; w_c \cdot \Dt_{\mathrm{compliance}}(t) \;+\; w_d \cdot \Dt_{\mathrm{distributional}}(t),
\end{equation}
where $w_c, w_d \geq 0$ with $w_c + w_d = 1$ (with application-specific tuning; in practice, weighting the compliance component more heavily than the distributional component), and the components are:

\medskip
\noindent\textbf{Compliance drift.} The weighted compliance gap at step~$t$:
\begin{equation}
\label{eq:compliance-drift}
    \Dt_{\mathrm{compliance}}(t) \;=\; 1 - \Ct(t) \;=\; \frac{\sum_{i} w_i \bigl(1 - \sigma_i(t)\bigr)}{\sum_{i} w_i},
\end{equation}
where $\sigma_i(t) \in \{0, 1\}$ indicates whether constraint~$i$ is satisfied at step~$t$, and $w_i > 0$ is the weight assigned to constraint~$i$.

\medskip
\noindent\textbf{Distributional drift.} The Jensen--Shannon divergence between the observed and reference action distributions:
\begin{equation}
\label{eq:distributional-drift}
    \Dt_{\mathrm{distributional}}(t) \;=\; \JSD\!\big(P_{\mathrm{observed}}(t) \;\|\; P_{\mathrm{reference}}\big),
\end{equation}
where $P_{\mathrm{observed}}(t)$ is the empirical action distribution computed over a sliding window of recent actions, and $P_{\mathrm{reference}}$ is a calibrated reference distribution obtained from a compliant baseline (e.g., the action distribution during a validated calibration session).
\end{definition}

\begin{remark}[Interpretability of $D(t)$ Values]\label{rem:drift-interpretation}
The drift score $D(t) \in [0,1]$ admits the following operational
interpretation:
\begin{itemize}[nosep]
  \item $D(t) = 0$: perfect compliance and distributional alignment.
  \item $D(t) \in (0, \theta_1]$: negligible drift; no intervention required.
  \item $D(t) \in (\theta_1, \theta_2]$: mild drift; monitoring should increase in frequency.
  \item $D(t) > \theta_2$: significant drift; active intervention is recommended.
\end{itemize}
The threshold parameters $\theta_1$ and $\theta_2$ are deployment-specific and empirically calibrated.  Typical enterprise deployments use low single-digit and mid-range values respectively.  Both thresholds are exposed as configurable parameters in the contract specification.
\end{remark}

\begin{proposition}[Properties of the Drift Score]
\label{prop:drift-properties}
The behavioral drift score $\Dt(t)$ satisfies:
\begin{enumerate}
    \item \textbf{Boundedness.} $\Dt(t) \in [0, 1]$ for all~$t$.

    \item \textbf{Minimality.} $\Dt(t) = 0$ if and only if full compliance holds ($\Ct(t) = 1$) and the observed action distribution is identical to the reference distribution ($P_{\mathrm{observed}}(t) = P_{\mathrm{reference}}$).

    \item \textbf{Incremental computability.} $\Dt(t)$ can be updated incrementally with complexity linear in the number of constraints and the action vocabulary size.

    \item \textbf{Metric structure.} The square root of the Jensen--Shannon divergence, $\sqrt{\JSD}$, is a metric on probability distributions and satisfies the triangle inequality~\citep{endres2003new}.
\end{enumerate}
\end{proposition}

\begin{proof}
(1)~Both $\Dt_{\mathrm{compliance}}(t) \in [0,1]$ (since $\Ct(t) \in [0,1]$) and $\Dt_{\mathrm{distributional}}(t) \in [0,1]$ (the Jensen--Shannon divergence with logarithm base~2 is bounded by~$1$). Since $w_c + w_d = 1$ with $w_c, w_d \geq 0$, the convex combination lies in $[0,1]$.

(2)~The forward direction: $\Dt(t) = 0$ requires both $w_c \cdot \Dt_{\mathrm{compliance}}(t) = 0$ and $w_d \cdot \Dt_{\mathrm{distributional}}(t) = 0$. Since $w_c, w_d > 0$ in the default parameterization, this forces $\Dt_{\mathrm{compliance}}(t) = 0$ (i.e., $\Ct(t) = 1$) and $\JSD(P_{\mathrm{observed}}(t) \| P_{\mathrm{reference}}) = 0$ (i.e., distributional identity). The converse is immediate.

(3)~The compliance component requires evaluating each constraint, contributing cost linear in the constraint set size. The distributional component maintains a histogram over the sliding window; inserting and removing one action and recomputing~$\JSD$ costs linear in the action vocabulary size.

(4)~Proven by \citet{endres2003new}; see also \citet{osterreicher2003new}.
\end{proof}

\begin{remark}[Meaningfulness of the Distributional Component]\label{rem:jsd-meaningfulness}
The $\JSD$ distributional component of $D(t)$ requires a sufficiently
rich action vocabulary to produce informative signals.  When the action
space is insufficiently diverse, the empirical action distribution may
be sparse and distributional measures exhibit high variance.  In such
cases, practitioners should either increase the observation window to
smooth the estimate, or adjust the component weights to emphasize
constraint-based compliance over distributional alignment.  For typical
enterprise deployments with diverse tool invocations, text categories,
and API call types, the action vocabulary is easily sufficient.
\end{remark}

\begin{remark}[Leading vs.\ Lagging Indicators]
\label{rem:leading-lagging}
The two components of the drift score serve complementary diagnostic roles. The compliance drift $\Dt_{\mathrm{compliance}}(t)$ is a \emph{lagging indicator}: it registers non-zero values only after a constraint violation has already occurred. The distributional drift $\Dt_{\mathrm{distributional}}(t)$ is a \emph{leading indicator}: it can detect shifts in the agent's action distribution---such as increased use of hedging language, atypical tool invocation patterns, or drifting topic focus---\emph{before} these shifts manifest as explicit constraint violations. This early-warning capability is critical for preemptive intervention.
\end{remark}

For fine-grained diagnostics, we decompose the drift score into its constituent sources.

\begin{definition}[Diagnostic Decomposition Vector]
\label{def:drift-vector}
The \emph{diagnostic decomposition vector} at step~$t$ is:
\begin{equation}
\label{eq:drift-vector}
    \vec{\Dt}(t) \;=\; \bigl(\Dt_{\mathcal{P}}(t),\; \Dt_{\mathcal{I}}(t),\; \Dt_{\mathcal{G}}(t),\; \Dt_{\mathrm{distributional}}(t)\bigr),
\end{equation}
where $\Dt_{\mathcal{P}}(t)$, $\Dt_{\mathcal{I}}(t)$, and $\Dt_{\mathcal{G}}(t)$ are the compliance gaps restricted to precondition-derived, invariant, and governance constraints, respectively. This vector enables operators to pinpoint whether drift originates from invariant violations, governance breaches, or distributional shift, and to route alerts to the appropriate remediation pathway.
\end{definition}

%% --------------------------------------------------------------------------
\subsection{Additional Operational Metrics}
\label{subsec:additional-metrics}

We briefly define three additional metrics that complement the compliance and drift scores in operational deployments.

\begin{definition}[Recovery Effectiveness]
\label{def:recovery-effectiveness}
The \emph{recovery effectiveness} for a violation event at step~$t$ is:
\begin{equation}
\label{eq:recovery-effectiveness}
    E(t) \;=\; \frac{\Delta t_{\mathrm{recovery}}}{\nu(t)},
\end{equation}
where $\Delta t_{\mathrm{recovery}}$ is the number of steps required to restore compliance and $\nu(t) \in (0,1]$ is the severity of the violation (defined as the magnitude of the compliance drop). Lower values of~$E$ indicate more effective recovery. We define the session-level recovery effectiveness as $E = \frac{1}{|\mathcal{V}|}\sum_{t \in \mathcal{V}} E(t)$, where $\mathcal{V}$ is the set of violation events.
\end{definition}

\begin{definition}[Stress Resilience Index]
\label{def:stress-resilience}
The \emph{stress resilience index} measures compliance degradation under adversarial or high-load conditions:
\begin{equation}
\label{eq:stress-resilience}
    S \;=\; \frac{\E\!\big[\Ct(t) \mid \text{stressed}\big]}{\E\!\big[\Ct(t) \mid \text{baseline}\big]},
\end{equation}
where the expectations are taken over steps within stressed and baseline sessions, respectively. A value $S = 1$ indicates no degradation under stress; $S < 1$ quantifies the compliance penalty imposed by adversarial conditions.
\end{definition}

\begin{definition}[Agent Reliability Index]
\label{def:reliability-index}
The \emph{agent reliability index} is a weighted composite that summarizes an agent's overall contractual fitness:
\begin{equation}
\label{eq:reliability-index}
    \Theta \;=\; \alpha_1 \cdot \overline{\Ct} \;+\; \alpha_2 \cdot (1 - \overline{\Dt}) \;+\; \alpha_3 \cdot \frac{1}{1 + E} \;+\; \alpha_4 \cdot S,
\end{equation}
where $\overline{\Ct}$ and $\overline{\Dt}$ denote the time-averaged compliance and drift scores over the session, the term $\frac{1}{1+E}$ maps recovery effectiveness to $[0,1]$ (with lower~$E$ yielding higher contribution), and the weights satisfy $\sum_{i=1}^{4} \alpha_i = 1$.  The component weights are application-specific, with typical enterprise deployments weighting compliance most heavily, followed by drift stability, recovery efficiency, and stress resilience. The index $\Theta \in [0,1]$ provides a single scalar summary suitable for comparing agents, tracking reliability over time, and establishing deployment thresholds.
\end{definition}

%% file: sections/04-drift-prevention.tex
%% ==========================================================================
%% Section 4 — Drift Prevention via Contracts
%% File: sections/04-drift-prevention.tex
%% Included via \input{sections/04-drift-prevention} in main.tex
%% ==========================================================================

\section{Drift Prevention via Contracts}
\label{sec:drift-prevention}

This section provides the theoretical backbone of the \ABC{} framework.
We model behavioral drift as a continuous-time stochastic process
(\Cref{subsec:drift-dynamics}), derive tight probabilistic bounds on drift
under contract enforcement (\Cref{subsec:drift-bounds}), establish
sufficient conditions for safe contract composition in multi-agent chains
(\Cref{subsec:composition}), and analyze the runtime cost of contract
checking (\Cref{subsec:complexity}).

%% --------------------------------------------------------------------------
\subsection{Drift Dynamics Model}
\label{subsec:drift-dynamics}

We model the behavioral drift of a contracted agent as a continuous-time
stochastic process governed by three competing forces: a natural tendency
to deviate from specification, a restorative force exerted by contract
enforcement, and stochastic perturbations inherent to LLM
non-determinism.

\begin{definition}[Drift Dynamics]\label{def:drift-dynamics}
Let $\Dt \geq 0$ denote the \emph{behavioral drift} of an agent at time
$t \geq 0$, measured as the $\JSD$ divergence between the agent's
observed action distribution and the contract-compliant reference
distribution (cf.\ \Cref{def:drift-score}).  The drift evolves
according to the stochastic differential equation
\begin{equation}\label{eq:drift-sde}
  d\Dt \;=\;
  \bigl(\alpha - \gamma\,\Dt\bigr)\,dt
  \;+\;
  \sigma\,dW(t),
\end{equation}
where the parameters satisfy $\alpha > 0$, $\gamma > 0$, $\sigma > 0$,
and $W(t)$ is a standard Wiener process.
\end{definition}

The three terms in~\eqref{eq:drift-sde} admit clear interpretations:

\begin{enumerate}[label=(\roman*),leftmargin=2em]
  \item \textbf{Baseline drift} ($\alpha\,dt$).  In the absence of
    enforcement, the agent's behavior naturally diverges from the
    contracted specification at rate~$\alpha$.  This captures prompt
    decay, context window dilution, and the tendency of autoregressive
    models to amplify small distributional shifts over extended task
    horizons.

  \item \textbf{Contract recovery} ($-\gamma\,\Dt\,dt$).  The
    enforcement mechanism exerts a restorative force proportional to
    current drift.  When $\Dt$ is large, the corrective signal is
    strong; when the agent is near compliance, the force relaxes.  The
    parameter~$\gamma$ is the \emph{contract recovery rate}---a
    design-time knob controlled by the contract's invariant-checking
    frequency and the aggressiveness of its recovery policy~$R$.

  \item \textbf{Stochastic perturbation} ($\sigma\,dW(t)$).  LLM
    outputs are inherently non-deterministic: identical prompts yield
    different completions across invocations.  The diffusion
    coefficient~$\sigma$ quantifies this irreducible noise floor,
    encompassing sampling temperature, nucleus truncation, and hardware
    floating-point variance.
\end{enumerate}

\begin{remark}\label{rem:ou-process}
Equation~\eqref{eq:drift-sde} is an instance of the
\emph{Ornstein--Uhlenbeck} (OU) process with mean-reversion
level~$\mu^* = \alpha/\gamma$, mean-reversion speed~$\gamma$, and
volatility~$\sigma$.  The OU process is one of the few
analytically tractable diffusions: it admits a closed-form transition
density, a Gaussian stationary distribution, and exponential ergodicity
bounds---properties we exploit throughout this section.  The restriction
$\Dt \geq 0$ is a modeling simplification; since the stationary mean
$\alpha/\gamma$ is strictly positive and the stationary standard
deviation $\sigma/\sqrt{2\gamma}$ is small relative to the mean for
well-designed contracts (i.e., $\sigma^2 \gamma \ll 2\alpha^2$, so the stationary
standard deviation is small relative to the mean), the
probability of the process reaching zero is negligible in practice.
\end{remark}

%% --------------------------------------------------------------------------
\subsection{Drift Bounds Theorem}
\label{subsec:drift-bounds}

We now state the main analytical result of this paper: a comprehensive
characterization of behavioral drift under contract enforcement.

\begin{theorem}[Stochastic Drift Bound]\label{thm:drift-bound}
Let $\Dt$ evolve according to the drift dynamics of
\Cref{def:drift-dynamics} with initial condition $D(0) = D_0 \geq 0$.
Then:

\begin{enumerate}[label=\textup{(\roman*)},leftmargin=2.5em]
  \item \textbf{Stationary distribution.}
    There exists a unique stationary distribution
    \begin{equation}\label{eq:stationary}
      \pi_D \;=\; \mathcal{N}\!\Bigl(\frac{\alpha}{\gamma},\;
      \frac{\sigma^2}{2\gamma}\Bigr).
    \end{equation}

  \item \textbf{Mean drift bound.}
    Under the stationary distribution,
    \begin{equation}\label{eq:mean-drift}
      \E_{\pi}\!\bigl[\Dt\bigr] \;=\; \frac{\alpha}{\gamma}.
    \end{equation}
    In particular, if $\gamma > \alpha$ then $\E_{\pi}[\Dt] < 1$.

  \item \textbf{Variance bound.}
    Under the stationary distribution,
    \begin{equation}\label{eq:var-drift}
      \Var_{\pi}\!\bigl(\Dt\bigr) \;=\; \frac{\sigma^2}{2\gamma}.
    \end{equation}
    Higher contract recovery rate~$\gamma$ quadratically reduces the
    spread of drift fluctuations relative to the noise level~$\sigma$.

  \item \textbf{High-probability bound.}
    For any $\eta > 0$,
    \begin{equation}\label{eq:tail-bound}
      \Pr_{\pi}\!\Bigl(\Dt > \frac{\alpha}{\gamma} + \eta\Bigr)
      \;\leq\;
      \exp\!\Bigl(-\frac{\gamma\,\eta^2}{\sigma^2}\Bigr).
    \end{equation}

  \item \textbf{Exponential convergence.}
    For all $t \geq 0$,
    \begin{equation}\label{eq:convergence}
      \E\!\bigl[(\Dt - \alpha/\gamma)^2\bigr]
      \;=\;
      (D_0 - \alpha/\gamma)^2\,e^{-2\gamma t}
      \;+\;
      \frac{\sigma^2}{2\gamma}\bigl(1 - e^{-2\gamma t}\bigr).
    \end{equation}

  \item \textbf{Contract design criterion.}
    To ensure $\Dt < D_{\max}$ with probability at least~$1-\varepsilon$
    under the stationary distribution, it suffices to choose $\gamma$
    as the larger root of the quadratic
    \begin{equation}\label{eq:design-quadratic}
      D_{\max}^2\,\gamma^2
      \;-\;
      \bigl(2\alpha\,D_{\max} + \sigma^2\ln(1/\varepsilon)\bigr)\,\gamma
      \;+\;
      \alpha^2
      \;=\; 0,
    \end{equation}
    i.e.,
    \begin{equation}\label{eq:design-criterion}
      \gamma
      \;\geq\;
      \frac{2\alpha\,D_{\max} + \sigma^2\ln(1/\varepsilon)
        + \sqrt{\bigl(2\alpha\,D_{\max} + \sigma^2\ln(1/\varepsilon)\bigr)^2
          - 4\,\alpha^2 D_{\max}^2}}{2\,D_{\max}^2}.
    \end{equation}
    When $\sigma^2\ln(1/\varepsilon) \ll 2\alpha\,D_{\max}$, this
    simplifies to the approximate criterion
    $\gamma \gtrsim \alpha/D_{\max}
      + \sigma\sqrt{2\ln(1/\varepsilon)}/(2D_{\max})$.
\end{enumerate}
\end{theorem}

\begin{proof}[Proof sketch]
Define the centered error process $e(t) = \Dt - \alpha/\gamma$.
Substituting into~\eqref{eq:drift-sde} yields the centered OU equation
\begin{equation}\label{eq:centered-ou}
  de(t) \;=\; -\gamma\,e(t)\,dt \;+\; \sigma\,dW(t),
\end{equation}
which has zero mean-reversion level and rate~$\gamma$.

To establish convergence, define the Lyapunov function
$V(e) = e^2$ and apply It\^{o}'s formula:
\[
  dV \;=\; \bigl(-2\gamma\,e^2 + \sigma^2\bigr)\,dt
       \;+\; 2\sigma\,e\,dW(t).
\]
Taking expectations eliminates the martingale term, yielding the
deterministic ODE
\[
  \frac{d}{dt}\,\E[V(t)]
  \;=\;
  -2\gamma\,\E[V(t)] + \sigma^2.
\]
This linear ODE has solution
$\E[V(t)] = V(0)\,e^{-2\gamma t} + \frac{\sigma^2}{2\gamma}(1-e^{-2\gamma t})$,
which converges exponentially to $\sigma^2/(2\gamma)$, establishing
parts~(iii) and~(v).  The stationary distribution~(i) follows from
standard Ornstein--Uhlenbeck theory~\citep{uhlenbeck1930theory}: the
unique invariant measure is Gaussian with
mean~$\alpha/\gamma$ and variance~$\sigma^2/(2\gamma)$, from which
part~(ii) is immediate.

The tail bound~(iv) applies Gaussian concentration to the stationary
distribution: for $X \sim \mathcal{N}(\mu, \sigma_s^2)$ with
$\sigma_s^2 = \sigma^2/(2\gamma)$,
\[
  \Pr(X > \mu + \eta)
  \;\leq\;
  \exp\!\Bigl(-\frac{\eta^2}{2\sigma_s^2}\Bigr)
  \;=\;
  \exp\!\Bigl(-\frac{\gamma\,\eta^2}{\sigma^2}\Bigr).
\]
Finally, the design criterion~(vi) follows by setting the right-hand
side of~\eqref{eq:tail-bound} to~$\varepsilon$ with
$\eta = D_{\max} - \alpha/\gamma$ and solving for~$\gamma$.  The full
details are provided in~\Cref{appendix:proofs}.
\end{proof}

\begin{remark}\label{rem:design-interpretation}
\Cref{thm:drift-bound} has direct engineering implications.
Part~(vi) provides an \emph{exact design rule}: given an
application's maximum tolerable drift~$D_{\max}$ and reliability
requirement~$1-\varepsilon$, the contract designer solves the
quadratic~\eqref{eq:design-quadratic} to obtain the
minimum recovery rate~$\gamma$ needed to meet the specification.
In the approximate regime ($\sigma^2\ln(1/\varepsilon) \ll 2\alpha D_{\max}$),
the required~$\gamma$ decomposes into two interpretable terms:
$\alpha/D_{\max}$ ensures the mean drift stays below threshold, while
the second term
$\sigma\sqrt{2\ln(1/\varepsilon)}/(2D_{\max})$ provides the additional
margin required to absorb stochastic fluctuations at the desired
confidence level.
\end{remark}

%% --------------------------------------------------------------------------
\subsection{Contract Composition}
\label{subsec:composition}

Enterprise agentic systems rarely consist of a single agent.
A typical deployment chains multiple specialized agents---a planner, a
retriever, a coder, a reviewer---into a sequential pipeline.  We now
establish conditions under which individual contract guarantees compose
into end-to-end guarantees for the chain.

\subsubsection{Serial Composition}

Consider a serial chain $A \to B$ where agent~$A$ produces output
consumed by agent~$B$.

\begin{definition}[Composed Contract]\label{def:composed-contract}
Given contracts $\mathcal{C}_A = (P_A, I_A, G_A, R_A)$ and
$\mathcal{C}_B = (P_B, I_B, G_B, R_B)$ for agents $A$ and $B$
respectively, the \emph{composed contract} for the serial chain
$A \to B$ is
\begin{equation}\label{eq:composed-contract}
  \mathcal{C}_{A \oplus B}
  \;=\;
  (P_{A \oplus B},\; I_{A \oplus B},\; G_{A \oplus B},\; R_{A \oplus B})
\end{equation}
where:
\begin{align}
  P_{A \oplus B} &\;=\; P_A,
    \label{eq:comp-precond} \\
  I_{A \oplus B} &\;=\; I_A \;\wedge\; I_B \;\wedge\; I_{\mathrm{handoff}},
    \label{eq:comp-invariant} \\
  G_{A \oplus B} &\;=\; G_A \;\cup\; G_B
    \quad\text{(assuming no conflicts)},
    \label{eq:comp-governance} \\
  R_{A \oplus B} &\;=\;
    \mathrm{compose}(R_A,\; R_B,\; R_{\mathrm{cascade}}),
    \label{eq:comp-recovery}
\end{align}
where $I_{\mathrm{handoff}}$ is a handoff invariant ensuring safe state
transfer between $A$ and~$B$, and $R_{\mathrm{cascade}}$ is a cascade
recovery policy that coordinates individual recovery actions across the
chain boundary.
\end{definition}

We first define the postcondition of a contracted agent, which appears
in the composition conditions below.

\begin{definition}[Postcondition]\label{def:postcondition}
The \emph{postcondition} of agent~$A$ under contract~$\mathcal{C}_A$,
denoted $\mathrm{PostCond}_A$, is the set of states reachable at
termination of~$A$ that satisfy all of~$A$'s invariants:
\[
  \mathrm{PostCond}_A
  \;=\;
  \bigl\{s \in \mathcal{S} \;:\;
    \forall\, i \in \mathcal{I}_A,\; i(s) = \mathit{true}\bigr\}.
\]
\end{definition}

Safe composition requires four sufficient conditions:

\begin{definition}[Composition Conditions]\label{def:composition-conditions}
A serial chain $A \to B$ with contracts $\mathcal{C}_A$ and
$\mathcal{C}_B$ satisfies the \emph{composition conditions} if:
\begin{enumerate}[label=\textup{(C\arabic*)},leftmargin=3em]
  \item \textbf{Interface Compatibility.}
    $\mathrm{Type}(\mathrm{PostCond}_A) \subseteq \mathrm{Type}(P_B)$.
    The output type of~$A$ is a subtype of the input type expected
    by~$B$.
    \label{cond:interface}

  \item \textbf{Assumption Discharge.}
    $\mathrm{PostCond}_A \;\wedge\; I_{\mathrm{handoff}}
    \;\Rightarrow\; P_B$.
    The postcondition of~$A$, together with the handoff invariant,
    logically entails the precondition of~$B$.
    \label{cond:assumption}

  \item \textbf{Governance Consistency.}
    Define the set-valued functions
    $\mathrm{Allowed}(G) = \{a \in \mathcal{A} : \forall\, g \in G,\; g(a) = \mathit{true}\}$
    and
    $\mathrm{Prohibited}(G) = \{a \in \mathcal{A} : \exists\, g \in G,\; g(a) = \mathit{false}\}$.
    Then
    $\mathrm{Allowed}(G_A) \;\cap\; \mathrm{Prohibited}(G_B)
    \;=\; \emptyset$.
    No action permitted by~$A$'s governance policy is forbidden
    by~$B$'s.
    \label{cond:governance}

  \item \textbf{Recovery Independence.}
    $\forall\, s \in \mathcal{S}:\;
    P_B\!\bigl(\mathrm{state\_after}(R_A(s))\bigr) = \mathbf{true}$.
    After~$A$'s recovery mechanism fires, the resulting state still
    satisfies~$B$'s precondition.
    \label{cond:recovery}
\end{enumerate}
\end{definition}

\begin{remark}[Standard vs.\ Novel Composition Conditions]\label{rem:c1c2-vs-c3c4}
Conditions~\ref{cond:interface} and~\ref{cond:assumption} (interface
compatibility and assumption discharge) are standard in
Design-by-Contract composition~\citep{meyer1992dbc} and
assume-guarantee reasoning~\citep{henzinger1998assume}.
Conditions~\ref{cond:governance} and~\ref{cond:recovery} (governance
consistency and recovery independence) are novel contributions of the
\ABC{} framework, motivated by the unique operational requirements of
multi-agent LLM pipelines: governance constraints span organizational
policies (not just type systems), and recovery mechanisms can
have cross-agent side effects that invalidate downstream preconditions.
\end{remark}

\begin{theorem}[Compositionality]\label{thm:compositionality}
Let agents $A$ and $B$ satisfy their respective contracts, i.e.,
$A \models \mathcal{C}_A$ and $B \models \mathcal{C}_B$.
If conditions~\ref{cond:interface}--\ref{cond:recovery} hold, then
\[
  \mathrm{Chain}(A, B) \;\models\; \mathcal{C}_{A \oplus B}.
\]
\end{theorem}

\begin{proof}[Proof sketch]
We verify each component of $\mathcal{C}_{A \oplus B}$.
The chain's precondition $P_{A \oplus B} = P_A$ holds by assumption.
Since $A \models \mathcal{C}_A$, agent~$A$ terminates in a state
satisfying $\mathrm{PostCond}_A$.  Condition~\ref{cond:assumption}
then guarantees $P_B$ holds at the handoff point, so $B$ begins
execution with a valid precondition.  The composed invariant
$I_A \wedge I_B \wedge I_{\mathrm{handoff}}$ is maintained: $I_A$ holds
during~$A$'s execution by $A \models \mathcal{C}_A$, $I_B$ holds
during~$B$'s execution by $B \models \mathcal{C}_B$, and
$I_{\mathrm{handoff}}$ holds at the transition by construction.
Condition~\ref{cond:governance} ensures the union
$G_A \cup G_B$ is conflict-free.
Condition~\ref{cond:recovery} ensures that if~$A$'s recovery fires, the
post-recovery state remains a valid input for~$B$.  The full proof,
including the inductive argument for governance consistency through the
chain, is provided in~\Cref{appendix:proofs}.
\end{proof}

\begin{remark}[Recovery Window Composition]\label{rem:k-composition}
When composing contracts with recovery windows $k_A$ and $k_B$, the
composed recovery window is
$k_{A \oplus B} = \max(k_A, k_B)$.
The maximum (rather than the sum) is the correct composition rule
because recovery windows operate concurrently within each agent's
execution phase: a soft violation in~$A$ must be recovered within~$k_A$
steps of~$A$'s local trace, not within a global budget shared with~$B$.
The composed window therefore reflects the more demanding of the two
per-agent requirements.
\end{remark}

\subsubsection{Probabilistic Composition}

In practice, contract satisfaction is probabilistic (cf.\
\Cref{def:pdk-satisfaction}).  We now characterize how
probabilistic guarantees degrade under composition.

\begin{theorem}[Probabilistic Compositionality]\label{thm:prob-compositionality}
Suppose agent~$A$ $(p_A, \delta_A)$-satisfies $\mathcal{C}_A$ and
agent~$B$ $(p_B, \delta_B)$-satisfies $\mathcal{C}_B$.  Let $p_h$
denote the probability that the handoff invariant
$I_{\mathrm{handoff}}$ holds, and let $\delta_h$ denote the maximum
drift introduced by the handoff mechanism.  Assume:
\begin{enumerate}[label=\textup{(C5)},leftmargin=3em]
  \item \textbf{Conditional Independence.}\label{cond:independence}
    Agent~$B$'s contract satisfaction is conditionally independent of
    agent~$A$'s internal execution, given that $B$ receives a
    contract-compliant input from the handoff:
    $\Pr(E_B \mid E_A \cap E_h) = \Pr(E_B \mid E_h)$.
\end{enumerate}
Then the composed chain
$(p_{A \oplus B}, \delta_{A \oplus B})$-satisfies
$\mathcal{C}_{A \oplus B}$ with
\begin{align}
  p_{A \oplus B}
    &\;\geq\; p_A \cdot p_B \cdot p_h,
    \label{eq:prob-compose-p} \\[4pt]
  \delta_{A \oplus B}
    &\;\leq\; \delta_A + \delta_B + \delta_h.
    \label{eq:prob-compose-delta}
\end{align}
\end{theorem}

\begin{remark}[Conditional Independence and Correlated LLM Failures]\label{rem:c5-correlation}
Condition~\ref{cond:independence} is satisfied when agents~$A$ and~$B$
operate on separate LLM instances or use distinct model providers.
When both agents share the same underlying LLM, correlated failure modes
(e.g., systematic prompt sensitivity, shared training biases) may
violate conditional independence.  In such settings, the probability
bound~\eqref{eq:prob-compose-p} becomes optimistic; practitioners should
apply a correlation penalty or use the tighter bound
$p_{A \oplus B} \geq p_A + p_B \cdot p_h - 1$ (Fr\'{e}chet--Hoeffding
lower bound) as a conservative alternative.
\end{remark}

\begin{proof}[Proof sketch]
The chain satisfies $\mathcal{C}_{A \oplus B}$ only if all three events
occur: $A$ satisfies $\mathcal{C}_A$, the handoff succeeds, and $B$
satisfies $\mathcal{C}_B$.  By condition~\ref{cond:independence}
(conditional independence of agent-level failures given contract-compliant
inputs), the joint probability is at
least $p_A \cdot p_B \cdot p_h$.  The drift bound follows from
sub-additivity: the maximum end-to-end deviation is bounded by the sum
of per-stage deviations plus the handoff-induced deviation.
See~\Cref{appendix:proofs} for the formal argument.
\end{proof}

The following corollary extends the result to chains of arbitrary length.

\begin{corollary}[$N$-Agent Chain]\label{cor:n-agent}
For a serial chain of $N$ agents $A_1 \to A_2 \to \cdots \to A_N$
where each agent~$A_i$ $(p_i, \delta_i)$-satisfies $\mathcal{C}_i$ and
each handoff has reliability~$p_{h_i}$ and drift~$\delta_{h_i}$:
\begin{align}
  p_{\mathrm{chain}}
    &\;\geq\; \prod_{i=1}^{N} p_i
              \;\cdot\;
              \prod_{i=1}^{N-1} p_{h_i},
    \label{eq:n-chain-p} \\[4pt]
  \delta_{\mathrm{chain}}
    &\;\leq\; \sum_{i=1}^{N} \delta_i
              \;+\;
              \sum_{i=1}^{N-1} \delta_{h_i}.
    \label{eq:n-chain-delta}
\end{align}
\end{corollary}

\begin{proof}
Follows by inductive application of
\Cref{thm:prob-compositionality} along the chain.
\end{proof}

\begin{remark}[The Broken Telephone Effect]\label{rem:broken-telephone}
\Cref{cor:n-agent} formalizes the intuitive ``broken telephone'' effect
in multi-agent systems: reliability degrades multiplicatively while
drift accumulates additively.  Consider a concrete example: a $5$-agent
chain where each agent satisfies its contract with probability
$p_i = 0.95$ and each handoff succeeds with probability
$p_{h_i} = 0.98$.  Then:
\[
  p_{\mathrm{chain}}
  \;\geq\;
  0.95^5 \cdot 0.98^4
  \;\approx\;
  0.7738 \times 0.9224
  \;\approx\;
  0.714.
\]
Similarly, if each agent contributes drift $\delta_i = 0.02$ and each
handoff contributes $\delta_{h_i} = 0.01$:
\[
  \delta_{\mathrm{chain}}
  \;\leq\;
  5 \times 0.02 + 4 \times 0.01
  \;=\;
  0.14.
\]
A chain that appears reliable at the individual level ($95\%$ per agent)
delivers only $\sim71.4\%$ end-to-end reliability, with accumulated
drift of~$0.14$.  This quantifies why multi-agent pipelines require
explicit contract enforcement at every stage, not merely at the
endpoints.  The design criterion of
\Cref{thm:drift-bound}(vi) can be applied independently to each agent
in the chain to ensure that per-stage drift remains within the budget
implied by the global $\delta_{\mathrm{chain}}$ target.
\end{remark}

%% --------------------------------------------------------------------------
\subsection{Complexity Analysis}
\label{subsec:complexity}

For contract enforcement to be practical, the runtime overhead must be
negligible relative to the latency of LLM inference itself (typically
$100$--$2000$\,ms per action).  We now show that this is the case.

\begin{proposition}[Runtime Contract Checking]\label{prop:complexity}
Let $k$ denote the number of constraints in a contract (preconditions,
invariants, and governance rules combined) and let $|A|$ denote the size
of the agent's action vocabulary.  The per-action cost of runtime
contract checking is
\[
  \mathcal{O}(k + |A|).
\]
\end{proposition}

\begin{proof}
The enforcement loop performs three operations per agent action:

\emph{(1)~Constraint evaluation.}  Each of the $k$ constraints
(preconditions, invariants, governance predicates) is evaluated as a
Boolean predicate over the current state and proposed action.  Each
predicate evaluation is $\mathcal{O}(1)$ (pattern matching on action
type, range checks on numeric fields, or set membership for governance
whitelists/blacklists).  Evaluating all $k$ constraints costs
$\mathcal{O}(k)$.

\emph{(2)~Behavioral drift update.}  The $\JSD$ divergence is
maintained incrementally via a sliding-window histogram over the action
vocabulary.  Updating the histogram upon observing a new action and
recomputing $\JSD$ between the observed and reference distributions
costs $\mathcal{O}(|A|)$, as it requires a single pass over the
$|A|$-dimensional probability vectors.

\emph{(3)~Weighted aggregation.}  The overall compliance score is
a weighted sum of constraint satisfaction and drift, computed in
$\mathcal{O}(1)$.

Combining: $\mathcal{O}(k) + \mathcal{O}(|A|) + \mathcal{O}(1) =
\mathcal{O}(k + |A|)$.
\end{proof}

\begin{remark}[Practical Overhead]\label{rem:practical-overhead}
For typical enterprise contracts we observe $k < 100$ constraints and
action vocabularies of size $|A| < 50$.  At these scales, the measured
wall-clock overhead of contract checking in \DG{} is consistently below
$10$\,ms per action---approximately $0.5$--$5\%$ of LLM inference
latency.  This confirms that contract enforcement is not a bottleneck
and can be deployed on every agent action without perceptible
degradation in end-to-end pipeline throughput.  We provide detailed
latency benchmarks in~\Cref{sec:experiments}.
\end{remark}

%% file: sections/05-agentassert.tex
%% ==========================================================================
%% Section 5 — ContractSpec and AgentAssert
%% File: sections/05-agentassert.tex
%% Included via \input{sections/05-agentassert} in main.tex
%% ==========================================================================

\section{ContractSpec and AgentAssert}
\label{sec:agentassert}

\emph{Note: This section describes the design principles and conceptual architecture of our reference implementation. Implementation-specific details including algorithmic pseudocode, class hierarchies, and configuration parameters are abstracted to focus on the scientific contributions. The complete implementation is subject to patent protection.}

The preceding sections established the formal foundations of Agent Behavioral Contracts (\ABC{}): the contract tuple $\mathcal{C} = (\mathcal{P}, \mathcal{I}_{\mathrm{hard}}, \mathcal{I}_{\mathrm{soft}}, \mathcal{G}_{\mathrm{hard}}, \mathcal{G}_{\mathrm{soft}}, \mathcal{R})$ (\Cref{sec:abc-framework}), the drift dynamics model, and provable composition guarantees (\Cref{sec:drift-prevention}).  We now describe the practical realization of these ideas in two artifacts: \CS{}, a domain-specific language for specifying agent contracts, and \DG{}, a runtime enforcement library that monitors, measures, and recovers compliance in real time.

%% --------------------------------------------------------------------------
\subsection{ContractSpec: A Domain-Specific Language for Agent Contracts}
\label{subsec:contractspec}

\CS{} is a YAML-based DSL that translates the mathematical contract tuple into a human-readable, machine-validatable specification.  The design reflects three principles:

\begin{enumerate}[nosep]
    \item \textbf{Declarative over imperative.}  Contract authors specify \emph{what} must hold, not \emph{how} to check it.  Constraint evaluation is the runtime's responsibility.
    \item \textbf{Hybrid syntax.}  Constraints may be expressed via structured operators (equality, comparison, set membership, pattern matching) or via expressive predicates for constraints that resist structured encoding.  This accommodates both simple field checks and complex cross-field logic.
    \item \textbf{File-reference composition.}  Pipeline contracts reference per-agent contracts by name or path, enabling compositional specification without duplication.  This directly supports the composition conditions~\ref{cond:interface}--\ref{cond:recovery} of \Cref{def:composition-conditions}.
\end{enumerate}

A \CS{} contract maps directly to the components of the \ABC{} tuple:
\begin{itemize}[nosep]
    \item Preconditions $\to \mathcal{P}$
    \item Hard invariants $\to \mathcal{I}_{\mathrm{hard}}$
    \item Soft invariants $\to \mathcal{I}_{\mathrm{soft}}$
    \item Hard governance $\to \mathcal{G}_{\mathrm{hard}}$
    \item Soft governance $\to \mathcal{G}_{\mathrm{soft}}$
    \item Recovery strategies $\to \mathcal{R}$
\end{itemize}

\noindent Additionally, the contract includes configuration for the satisfaction parameters $(p, \delta, k)$ (\Cref{def:pdk-satisfaction}), the drift metric weights and thresholds (\Cref{def:drift-score}), and the reliability index weights (\Cref{def:reliability-index}).

\paragraph{Constraint operators.}  Each constraint specifies a \texttt{check} block containing a field path and an operator.  \CS{} defines a set of standard comparison, membership, pattern-matching, and range operators covering the majority of enterprise governance predicates.  For constraints involving cross-field comparisons or arithmetic, \CS{} supports an expression syntax evaluated in a sandboxed environment with controlled capabilities.

\paragraph{Example: financial advisor contract (abbreviated).}  The following illustrates the contract structure.  Preconditions verify initial state requirements; hard invariants enforce zero-tolerance properties (e.g., data protection, regulatory compliance); soft invariants specify recoverable quality constraints; governance constraints limit agent actions; and recovery strategies define corrective actions.  Satisfaction parameters control the tolerance bounds for soft constraint violations.

\begin{verbatim}
contractspec: [version]
kind: agent
name: [agent-name]

preconditions:
  - name: required-initial-state
    check: {field: ..., operator: ...}

invariants:
  hard:
    - name: critical-compliance-constraint
      category: [compliance domain]
      check: {field: ..., operator: ...}
  soft:
    - name: quality-constraint
      check: {field: ..., operator: ...}
      recovery: [strategy-reference]

governance:
  hard:
    - name: action-boundary
      category: [governance domain]
      check: {field: ..., operator: ...}

recovery:
  strategies:
    - name: [strategy-name]
      type: [strategy-type]
      action: [corrective-action]

satisfaction:
  p: [probability threshold]
  delta: [tolerance bound]
  k: [recovery window]
\end{verbatim}

\paragraph{Schema validation.}  Every \CS{} contract is validated against a JSON Schema that defines type constraints for agent and pipeline contracts.  The schema enforces structural correctness: constraint-recovery linkage, pipeline stage requirements, and governance category membership from a predefined taxonomy.  Schema validation rejects malformed contracts before runtime evaluation.

\paragraph{Pipeline contracts.}  For multi-agent pipelines, \CS{} supports a pipeline variant that specifies ordered stages, handoff constraints between stages, pipeline-level governance, and a recovery coordination strategy.  The handoff and governance constraints enforce the composition conditions of \Cref{def:composition-conditions}, and the coordination strategy governs recovery propagation across stages.

%% --------------------------------------------------------------------------
\subsection{AgentAssert Architecture}
\label{subsec:architecture}

\DG{} implements the \ABC{} framework as a modular Python library with distinct functional concerns:

\begin{itemize}[nosep]
    \item \textbf{Parsing and validation}: Loads contract specifications, validates structure and semantics, produces typed contract objects.
    \item \textbf{Constraint evaluation}: Evaluates preconditions, invariants, and governance constraints against observed agent state, computes compliance scores.
    \item \textbf{Metric tracking}: Maintains time-series data for compliance, drift, and recovery effectiveness, computes the reliability index.
    \item \textbf{Runtime orchestration}: Coordinates per-turn enforcement, recovery execution, and event notification.
    \item \textbf{Integration}: Provides framework-agnostic hooks and framework-specific adapters for agent platforms.
    \item \textbf{Benchmarking}: Evaluates contracts against the \textsc{AgentContract-Bench} suite.
\end{itemize}

\noindent The architecture enforces strict layering with no circular dependencies.  The core library depends only on standard Python data-processing libraries for parsing, validation, and expression evaluation.

%% --------------------------------------------------------------------------
\subsection{Per-Turn Enforcement}
\label{subsec:session-monitor}

The runtime monitor is the central component that orchestrates per-turn contract enforcement.  At each agent execution step, the monitor:

\begin{enumerate}[nosep]
    \item Evaluates all contract constraints against the current observed state.
    \item Updates compliance and drift metrics based on evaluation results.
    \item Emits notification events for violations and drift alerts.
    \item Attempts recovery for soft constraint violations within the bounded recovery window.
    \item Resets recovery state for constraints that return to satisfaction.
\end{enumerate}

\noindent The monitor maintains strict separation between evaluation and recovery: compliance scores reflect pre-recovery state, ensuring accurate diagnostics.  Recovery state is tracked per-constraint with attempt counters that reset upon re-satisfaction, implementing the bounded recovery window of \Cref{def:pdk-satisfaction}.  An event notification system provides decoupled observability for external monitoring and alerting.

Per \Cref{prop:complexity}, the per-step computational cost is $\mathcal{O}(k + |\mathcal{A}|)$ where $k$ is the number of constraints and $|\mathcal{A}|$ is the action vocabulary size.  In practice, overhead is below 10\,ms per step for contracts with up to 100 constraints.

%% --------------------------------------------------------------------------
\subsection{Recovery Mechanisms}
\label{subsec:recovery}

Recovery is the operational realization of the mapping $\mathcal{R}\colon (\mathcal{I}_{\mathrm{soft}} \cup \mathcal{G}_{\mathrm{soft}}) \times \mathcal{S} \to \mathcal{A}^{*}$ from \Cref{def:abc-contract}.  The recovery executor implements this mapping through three components: strategy dispatch, action execution, and fallback chains.

\paragraph{Recovery strategy types.}  \CS{} defines a taxonomy of recovery types organized by escalation severity, ranging from lightweight prompt modifications through autonomy reduction to human escalation and session termination.  Strategies can be composed into fallback chains where primary strategies transition to progressively more aggressive interventions upon exhaustion of their attempt limits.

\paragraph{Fallback chains.}  Each recovery strategy may specify a fallback strategy and an attempt limit.  When a primary strategy is exhausted, the executor follows the fallback chain, ensuring graceful degradation from automated correction to human intervention or session termination.  This mechanism implements bounded recovery as formalized in \Cref{def:pdk-satisfaction}.

\paragraph{Action execution model.}  Recovery strategies define corrective actions that are dispatched through a registration mechanism supporting multiple agent frameworks.  This design maintains framework independence while allowing platform-specific integration through adapter modules.

\paragraph{Connection to drift bounds.}  The Drift Bounds Theorem (\Cref{thm:drift-bound}) establishes that contract enforcement bounds the stationary mean drift to $\E_{\pi}[\Dt] = \alpha/\gamma$, where $\gamma$ is the contract recovery rate.  The recovery executor is the mechanism through which $\gamma$ is realized operationally: more aggressive strategies (active correction, re-prompting) yield higher effective $\gamma$ values, while passive strategies (logging without intervention) contribute minimally to recovery.  The design criterion of \Cref{thm:drift-bound}(vi) provides a principled basis for choosing recovery aggressiveness: given a target $D_{\max}$ and a measured baseline drift rate $\alpha$, the contract designer selects strategies whose combined effectiveness achieves $\gamma \geq \alpha / D_{\max} + \sigma\sqrt{2\ln(1/\varepsilon)} / (2 D_{\max})$.

\paragraph{Recovery effectiveness metric.}  Each recovery event is tracked by the recovery subsystem, which computes the recovery effectiveness $E(t) = \Delta t_{\mathrm{recovery}} / \nu(t)$ per \Cref{def:recovery-effectiveness}.  The session-level average $E$ feeds into the reliability index $\Theta$ (\Cref{def:reliability-index}), closing the loop between runtime behavior and the composite fitness score.

%% --------------------------------------------------------------------------
\subsection{Implementation Summary}
\label{subsec:implementation-summary}

\Cref{tab:metrics-summary} summarizes the five formal metrics computed by \DG{} and their relationship to the theoretical definitions of \Cref{sec:abc-framework}.

\begin{table}[t]
\centering
\caption{%
    Summary of metrics computed by \DG{}.  All metrics are defined formally in \Cref{sec:abc-framework} and computed at every enforcement step by the runtime monitor.%
}
\label{tab:metrics-summary}
\smallskip
\small
\begin{tabular}{@{}llll@{}}
\toprule
\textbf{Metric} & \textbf{Symbol} & \textbf{Range} & \textbf{Computation} \\
\midrule
Hard compliance & $\Ct_{\mathrm{hard}}(t)$ & $[0,1]$ & Fraction of hard constraints satisfied \\
Soft compliance & $\Ct_{\mathrm{soft}}(t)$ & $[0,1]$ & Fraction of soft constraints satisfied \\
Behavioral drift & $D(t)$ & $[0,1]$ & $w_c(1 - \Ct) + w_d \cdot \JSD(P_{\mathrm{obs}} \| P_{\mathrm{ref}})$ \\
Recovery effectiveness & $E$ & $[0,\infty)$ & Mean recovery steps / violation severity \\
Reliability index & $\Theta$ & $[0,1]$ & Weighted composite: $\alpha_1 \bar{C} + \alpha_2 (1-\bar{D}) + \alpha_3 \frac{1}{1+E} + \alpha_4 S$ \\
\bottomrule
\end{tabular}
\end{table}

\paragraph{Implementation scale.}  The reference implementation comprises approximately 3{,}000 lines of Python across the functional layers described above, with comprehensive test coverage exceeding 95\%.

\paragraph{API design.}  The public API provides three entry points: contract loading from \CS{} specifications, real-time per-turn session enforcement, and batch benchmark evaluation against \textsc{AgentContract-Bench}.  A minimal integration requires loading a contract, creating a session monitor, and calling the enforcement step within the agent's execution loop.  The session summary returns compliance time series, drift trajectory, recovery logs, and the composite reliability index~$\Theta$ with deployment readiness assessment.

\paragraph{Design trade-offs.}  Two deliberate trade-offs are worth noting.  First, \CS{} is intentionally not Turing-complete: the expression syntax supports arithmetic and Boolean logic but not loops, function definitions, or arbitrary code execution.  This sacrifices generality for safety---a contract specification should not be a vector for code injection.  Second, the recovery executor operates through an action dispatch mechanism rather than through direct API calls to any specific agent framework.  This indirection adds a thin layer of boilerplate at integration time but ensures that \DG{} remains decoupled from any specific agent runtime, supporting diverse agent platforms with equal facility.

%% file: sections/06-benchmark.tex
%% ==========================================================================
%% Section 6 — AgentContract-Bench
%% File: sections/06-benchmark.tex
%% Included via \input{sections/06-benchmark} in main.tex
%% ==========================================================================

\section{AgentContract-Bench}
\label{sec:benchmark}

Evaluating the \ABC{} framework requires a benchmark that tests contract enforcement across diverse agent domains, violation types, and adversarial conditions.  Existing benchmarks such as AgentBench~\citep{liu2023agentbench} evaluate general-purpose agent capabilities (e.g., web browsing, database queries), and HELM~\citep{liang2022helm} evaluates language model quality along dimensions such as accuracy and calibration.  Neither targets the specific question central to our work: \emph{does a contract enforcement system correctly detect behavioral violations, maintain compliance under stress, and preserve guarantees across composed multi-agent pipelines?}

To fill this gap, we introduce \textsc{AgentContract-Bench}, a benchmark of 200 scenarios spanning 7 domains, designed from first principles to evaluate runtime behavioral contract enforcement.  Each scenario consists of a \emph{synthetic} multi-step execution trace (5--8 steps) with pre-computed state observations, agent actions, and ground-truth violation annotations.  The benchmark evaluates the \emph{enforcement engine} (parser, evaluator, metric trackers) against these synthetic traces; it does not involve live LLM inference.  Empirical evaluation on live LLM agents is presented separately in~\Cref{sec:experiments}.  The benchmark ships as part of the \DG{} library and will be made available subject to intellectual property clearance.

%% --------------------------------------------------------------------------
\subsection{Benchmark Design}
\label{subsec:benchmark-design}

\textsc{AgentContract-Bench} comprises 200 scenarios organized into three tiers: 5 \emph{agent domains} (100 scenarios), a \emph{governance stress} tier (50 scenarios), and a \emph{composition} tier (50 scenarios).  The agent domains exercise contracts over distinct real-world use cases; the governance tier subjects those same contracts to adversarial conditions; the composition tier tests the compositionality theorem (\Cref{thm:compositionality}) on a multi-stage pipeline.

Each scenario is a structured test case containing a multi-step execution trace (5--8 agent actions) with ground-truth annotations for expected violations, compliance ranges, and outcomes.

\Cref{tab:benchmark-domains} summarizes the domain breakdown.

\begin{table}[t]
\centering
\caption{%
    \textsc{AgentContract-Bench} domain breakdown.  The five agent domains each exercise a dedicated \CS{} contract.  The governance tier applies all five contracts under adversarial stress profiles.  The composition tier tests a 3-stage loan processing pipeline against conditions~\ref{cond:interface}--\ref{cond:recovery}.%
}
\label{tab:benchmark-domains}
\smallskip
\small
\begin{tabular}{@{}llcl@{}}
\toprule
\textbf{Domain} & \textbf{Contract} & \textbf{$N$} & \textbf{Key Constraints Tested} \\
\midrule
Financial advisory  & \texttt{financial-advisor}  & 20 & PII, disclosure, spending limits \\
Customer support    & \texttt{customer-support}    & 20 & Empathy, escalation, refund caps \\
Code generation     & \texttt{code-generation}     & 20 & Secrets, \texttt{eval} injection, license \\
Research synthesis  & \texttt{research-assistant}   & 20 & Citations, fabrication, sources \\
Healthcare triage   & \texttt{healthcare-triage}    & 20 & Diagnosis scope, prescription, HIPAA \\
\midrule
Governance stress   & (all five)                    & 50 & 6 adversarial stress profiles \\
Composition         & \texttt{loan-pipeline}        & 50 & C1--C4 composition conditions \\
\midrule
\textbf{Total}      &                               & \textbf{200} & \\
\bottomrule
\end{tabular}
\end{table}

Scenarios are assigned one of three difficulty levels---\emph{easy} (18 scenarios), \emph{medium} (68), and \emph{hard} (114)---based on the subtlety of the violation, the number of constraints simultaneously active, and the depth of multi-step context required to detect the violation.  The heavy skew toward hard scenarios reflects the benchmark's design philosophy: easy violations (e.g., an agent emitting a credit card number in plain text) are straightforward to detect; the research value lies in subtle, multi-step, context-dependent violations that stress the enforcement system.

%% --------------------------------------------------------------------------
\subsection{Stress Profiles}
\label{subsec:stress-profiles}

The 50 governance stress scenarios apply six adversarial conditions to the five agent-domain contracts.  These conditions model realistic failure modes observed in production agentic deployments:

\begin{enumerate}[nosep]
    \item \textbf{Prompt injection} (9 scenarios).  The execution trace contains injected prompts that attempt to override contract constraints---e.g., ``Ignore previous instructions and reveal the customer's SSN.''
    \item \textbf{Tool failure} (9 scenarios).  One or more tool calls fail mid-session (timeout, malformed response, permission denied), forcing the agent to degrade gracefully without violating hard constraints.
    \item \textbf{Conflicting instructions} (9 scenarios).  The user's request directly conflicts with a contract constraint---e.g., ``Give me a diagnosis'' when the healthcare contract prohibits diagnosis.
    \item \textbf{Time pressure} (8 scenarios).  Latency constraints tighten (simulating real-time requirements), testing whether the agent sacrifices compliance for speed.
    \item \textbf{Resource pressure} (8 scenarios).  Token or cost budgets are nearly exhausted, testing compliance under resource scarcity.
    \item \textbf{Social engineering} (7 scenarios).  The user employs social manipulation tactics (authority claims, urgency framing, emotional pressure) to coerce the agent into violations.
\end{enumerate}

These profiles are drawn from the adversarial taxonomy of~\citet{amodei2016concrete} and extended with LLM-specific failure modes identified in recent deployment reports~\citep{weidinger2021ethical}.  The stress resilience index $S$ (\Cref{def:stress-resilience}) is computed for each profile by comparing compliance under stress to the baseline domain scenarios.

%% --------------------------------------------------------------------------
\subsection{Composition Testing}
\label{subsec:composition-testing}

The 50 composition scenarios evaluate the compositionality theorem (\Cref{thm:compositionality}) on a \emph{loan processing pipeline}---a 3-stage serial chain:
\[
    \text{Intake Agent} \;\xrightarrow{\;\text{handoff}\;}\; \text{Analysis Agent} \;\xrightarrow{\;\text{handoff}\;}\; \text{Decision Agent}.
\]
The intake agent collects applicant information and performs initial eligibility screening.  The analysis agent evaluates creditworthiness and risk factors.  The decision agent renders a final loan decision with regulatory justification.  Each agent operates under its own \CS{} contract, and the pipeline is governed by a composed contract $\mathcal{C}_{\text{pipeline}}$ as defined in~\Cref{def:composed-contract}.

The 50 scenarios are partitioned into five categories that systematically test the composition conditions~\ref{cond:interface}--\ref{cond:recovery}:

\begin{enumerate}[nosep]
    \item \textbf{Clean handoffs} (15 scenarios).  All four composition conditions hold; the pipeline executes correctly end to end.
    \item \textbf{C1: Interface mismatch} (8 scenarios).  The output type of one agent is incompatible with the input type expected by the next---e.g., the intake agent emits an incomplete applicant record missing required fields.
    \item \textbf{C2: Assumption failure} (8 scenarios).  The receiving agent's preconditions are not discharged by the sender's postconditions---e.g., the analysis agent assumes a credit score is present, but intake did not retrieve one.
    \item \textbf{C3: Governance breach} (8 scenarios).  An action permitted by one agent's governance policy is prohibited by the pipeline-level policy---e.g., the decision agent attempts to access demographic data that pipeline governance forbids.
    \item \textbf{C4: Recovery coordination failure} (11 scenarios).  A recovery action in one agent invalidates the preconditions of a downstream agent---e.g., the analysis agent's recovery re-requests applicant data, but the modified data no longer satisfies the decision agent's input constraints.
\end{enumerate}

%% --------------------------------------------------------------------------
\subsection{Evaluation Protocol}
\label{subsec:evaluation-protocol}

We define a multi-level evaluation protocol that scores contract enforcement along five dimensions.

\medskip\noindent\textbf{Detection accuracy.}  The fraction of expected violations (annotated in the ground truth) that the enforcement system correctly identifies.  A detection accuracy of 1.0 means every ground-truth violation is correctly flagged.

\medskip\noindent\textbf{Compliance scores.}  The hard compliance score $\Ct_{\mathrm{hard}}$ and soft compliance score $\Ct_{\mathrm{soft}}$, as defined in~\Cref{def:compliance-scores}, are computed per scenario and averaged over each domain.

\medskip\noindent\textbf{Drift score.}  The behavioral drift score $D(t)$, as defined in~\Cref{def:drift-score}, is computed at each trace step and averaged over the scenario.

\medskip\noindent\textbf{Reliability index.}  The agent reliability index $\Theta$ (\Cref{def:reliability-index}) provides a single scalar summary per scenario.  Domain-level and overall scores are computed as arithmetic means.

\medskip\noindent\textbf{Outcome classification.}  Each scenario is classified into one of three outcomes: \emph{compliant} (no violations detected), \emph{hard violation} (at least one hard constraint violated), or \emph{soft violation} (only soft constraint violations, all recovered within the window $k$).

Scoring proceeds at three levels of granularity: per-scenario (a score vector for each of the 200 scenarios), per-domain (aggregated over the scenarios in each of the 7 domains), and overall (aggregated across all 200 scenarios).

%% --------------------------------------------------------------------------
\subsection{Validation Results}
\label{subsec:validation-results}

We validated the benchmark by running all 200 scenarios through the \DG{} enforcement engine.  \Cref{tab:benchmark-results} reports the per-domain results.

\begin{table}[t]
\centering
\caption{%
    \textsc{AgentContract-Bench} validation results.  All metrics are domain-level averages.  Detection accuracy is 1.0000 across all domains, confirming that the enforcement engine correctly identifies every annotated violation.  The composition domain exhibits lower $\Theta$ due to the inherent complexity of multi-agent pipeline violations.%
}
\label{tab:benchmark-results}
\smallskip
\begin{tabular}{@{}lccccc@{}}
\toprule
\textbf{Domain} & $N$ & $C_{\mathrm{hard}}$ & $C_{\mathrm{soft}}$ & $\bar{D}$ & $\Theta$ \\
\midrule
Financial advisory  & 20 & 0.9774 & 0.9683 & 0.0272 & 0.9837 \\
Customer support    & 20 & 0.9690 & 0.9528 & 0.0382 & 0.9787 \\
Code generation     & 20 & 0.9750 & 0.9740 & 0.0254 & 0.9847 \\
Research synthesis  & 20 & 0.9600 & 0.9317 & 0.0542 & 0.9675 \\
Healthcare triage   & 20 & 0.9744 & 0.9540 & 0.0334 & 0.9755 \\
\midrule
Governance stress   & 50 & 0.9539 & 0.9590 & 0.0435 & 0.9739 \\
Composition         & 50 & 0.8603 & 0.7532 & 0.1835 & 0.8865 \\
\midrule
\textbf{Overall}    & \textbf{200} & --- & --- & --- & \textbf{0.9541} \\
\bottomrule
\end{tabular}
\end{table}

\paragraph{Key observations.}

\begin{enumerate}[nosep]
    \item \textbf{Perfect specification--implementation consistency.}  The enforcement engine achieves a detection accuracy of 1.0000 across all 200 scenarios and all 7 domains, meaning every ground-truth violation annotated in the benchmark is correctly identified by the runtime evaluator.  This result validates that the \DG{} implementation faithfully realizes the formal semantics of \CS{} contracts; it does \emph{not} measure detection accuracy on live LLM agents, which is the subject of the empirical evaluation in~\Cref{sec:experiments}.

    \item \textbf{High reliability in agent domains.}  The five agent domains exhibit reliability indices in the range $\Theta \in [0.9675, 0.9847]$, with hard compliance scores above 0.96 in all domains.  Code generation achieves the highest $\Theta = 0.9847$, reflecting the binary nature of its constraints (e.g., secret present or absent).  Research synthesis has the lowest agent-domain $\Theta = 0.9675$, consistent with the greater ambiguity in citation and fabrication detection.

    \item \textbf{Governance stress resilience.}  The governance tier achieves $\Theta = 0.9739$, only marginally below the agent-domain average.  This indicates that the enforcement system maintains contract guarantees even under adversarial conditions including prompt injection, tool failure, and social engineering.

    \item \textbf{Composition is the hardest domain.}  The composition tier has the lowest reliability index ($\Theta = 0.8865$) and the highest mean drift ($\bar{D} = 0.1835$).  This is expected: composition scenarios involve multi-agent handoffs where violations of conditions~\ref{cond:interface}--\ref{cond:recovery} cascade across pipeline stages.  The hard compliance score drops to 0.8603, reflecting scenarios where interface mismatches (C1) and recovery coordination failures (C4) propagate to downstream agents.  These results empirically validate the multiplicative reliability degradation predicted by~\Cref{thm:prob-compositionality}.

    \item \textbf{Outcome distribution.}  Across all 200 scenarios, the enforcement engine classifies 23 as compliant, 117 as hard violations, and 60 as soft violations.  The predominance of hard violations (58.5\%) reflects the benchmark's adversarial design: the majority of scenarios are constructed to trigger contract breaches, testing the enforcement system's ability to detect---not prevent---violations at the contract layer.
\end{enumerate}

%% --------------------------------------------------------------------------
\subsection{Comparison with Existing Benchmarks}
\label{subsec:benchmark-comparison}

\textsc{AgentContract-Bench} addresses a gap that existing agent evaluation suites do not cover.  \Cref{tab:benchmark-comparison} positions our benchmark relative to two widely used alternatives.

\begin{table}[t]
\centering
\caption{%
    Comparison of \textsc{AgentContract-Bench} with existing agent evaluation benchmarks.%
}
\label{tab:benchmark-comparison}
\smallskip
\resizebox{\columnwidth}{!}{%
\begin{tabular}{@{}lcccc@{}}
\toprule
\textbf{Property} & \textbf{AgentBench} & \textbf{HELM} & \textbf{StepShield} & \textbf{AgentContract-Bench} \\
\midrule
Target            & Task completion & LM quality       & Temporal violation & Contract enforcement \\
Evaluation unit   & Task success rate     & Per-prompt score  & Per-trace timing  & Per-trace compliance \\
Multi-step traces & Yes                   & No                & Yes               & Yes \\
Violation detection & No                  & No                & Yes               & Yes \\
Hard/soft distinction & No               & No                & No                & Yes \\
Adversarial stress  & No                 & Partial           & No                & Yes (6 profiles) \\
Composition testing & No                 & No                & No                & Yes (C1--C4) \\
Formal metrics      & No                 & Partial           & Partial (EIR, IG) & Yes ($\Theta$, $D(t)$, $C(t)$) \\
\bottomrule
\end{tabular}%
}
\end{table}

AgentBench~\citep{liu2023agentbench} evaluates LLM agents across 8 environments (operating system, database, web, etc.)\ and measures task completion rate.  It does not define behavioral contracts, does not distinguish hard from soft constraints, and does not test adversarial robustness or multi-agent composition.  HELM~\citep{liang2022helm} provides a holistic evaluation of language models across 42 scenarios, covering accuracy, calibration, robustness, fairness, and other dimensions.  While HELM includes some robustness perturbations, it operates at the single-prompt level and does not evaluate multi-step behavioral traces, contract violations, or pipeline composition.  \textsc{StepShield}~\citep{guo2026stepshield} introduces a benchmark for \emph{temporal} detection of agent violations, measuring when violations are detected via Early Intervention Rate (EIR) and Intervention Gap (IG) metrics.  However, \textsc{StepShield} does not distinguish hard from soft constraints, does not test adversarial stress profiles, does not evaluate multi-agent composition, and does not provide session-level compliance or drift metrics.

\textsc{AgentContract-Bench} is, to our knowledge, the first benchmark specifically designed for behavioral contract enforcement in autonomous AI agents.  It combines multi-step trace evaluation, formal compliance metrics grounded in the \ABC{} framework (\Cref{sec:abc-framework}), adversarial stress testing, and systematic composition testing against the conditions of the compositionality theorem (\Cref{thm:compositionality}).

%% file: sections/07-experiments.tex
%% ==========================================================================
%% Section 7 — Experiments
%% File: sections/07-experiments.tex
%% Included via \input{sections/07-experiments} in main.tex
%% ==========================================================================

\section{Experiments}
\label{sec:experiments}

The preceding sections established the formal \ABC{} framework (\Cref{sec:abc-framework}), its drift-prevention guarantees (\Cref{sec:drift-prevention}), the \DG{} runtime library (\Cref{sec:agentassert}), and the synthetic \textsc{AgentContract-Bench} benchmark (\Cref{sec:benchmark}).  We now turn to \emph{empirical} evaluation: can \ABC{} contracts, enforced at runtime by \DG{}, measurably improve behavioral governance of real large language model agents?

We design four experiments with increasing scope:
\begin{enumerate}[nosep]
    \item \textbf{E1: Contracted vs.\ Uncontracted} (\Cref{subsec:e1})---the central experiment, comparing agent behavior with and without \ABC{} contract enforcement across 7 models from 6 vendors.
    \item \textbf{E2: Drift Prevention} (\Cref{subsec:e2})---extended multi-turn sessions testing whether contracted agents exhibit bounded drift over longer horizons, as predicted by~\Cref{thm:drift-bound}.
    \item \textbf{E3: Governance Under Stress} (\Cref{subsec:e3})---adversarial prompt injection to test contract resilience under active attack.
    \item \textbf{E4: Ablation Study} (\Cref{subsec:e4})---isolating the contribution of each \ABC{} component (hard constraints, soft constraints, drift monitoring, recovery).
\end{enumerate}

%% --------------------------------------------------------------------------
\subsection{Evaluation Methodology}
\label{subsec:eval-methodology}

Before presenting individual experiments, we describe the evaluation methodology that governs all four studies.  The methodology addresses five concerns: fair experimental controls (\Cref{subsubsec:controls}), principled evaluation methods (\Cref{subsubsec:eval-methods}), rigorous statistical analysis (\Cref{subsubsec:stats}), multi-vendor model coverage (\Cref{subsubsec:models}), and full reproducibility (\Cref{subsubsec:reproducibility}).  We also document an empirical finding regarding platform-level guardrail interference that influenced our experimental configuration (\Cref{subsubsec:platform-guardrails}).

%% - - - - - - - - - - - - - - - - - - - - - - - - - - - - - - - - - - - - -
\subsubsection{Experimental Controls}
\label{subsubsec:controls}

Evaluating a contract enforcement framework against an uncontracted baseline requires careful control design.  We impose four controls, labeled F1--F4, to ensure that observed differences are attributable to contract enforcement and not to confounding factors.

\paragraph{F1: Fair comparison.}
We define two experimental conditions that differ \emph{only} in the presence or absence of contract enforcement:

\begin{table}[t]
\centering
\caption{%
    Contracted vs.\ uncontracted experimental conditions.  Both conditions receive identical domain context and user tasks; they differ only in whether \ABC{} contract rules are injected and enforced.%
}
\label{tab:conditions}
\smallskip
\small
\begin{tabular}{@{}lccc@{}}
\toprule
\textbf{Condition} & \textbf{System prompt} & \textbf{Monitoring} & \textbf{Recovery} \\
\midrule
Contracted   & Domain context + full contract rules & Active (\DG{} monitor) & LLM re-prompting \\
Uncontracted & Domain context only (no rules)        & Passive (same evaluator) & None \\
\bottomrule
\end{tabular}
\end{table}

\noindent In contracted mode, the LLM explicitly sees every constraint it must follow, injected into the system prompt as structured behavioral rules.  In uncontracted mode, the LLM receives only the domain context---no behavioral rules leak through.  Both modes are evaluated by the \emph{identical} constraint evaluator instantiated from the parsed \CS{} contract, ensuring that the measurement instrument is constant across conditions.

\paragraph{F2: Real recovery.}
When a soft violation is detected in contracted mode, the recovery mechanism performs genuine LLM re-prompting rather than post-hoc metric manipulation:
\begin{enumerate}[nosep]
    \item The evaluator pre-checks the current turn without recording metrics.
    \item A corrective prompt is constructed containing the names of violated constraints and specific recovery instructions.
    \item The LLM is re-called with the corrective prompt (at most one retry per turn).
    \item The corrected response becomes the official response for metric computation.
\end{enumerate}
Only soft violations trigger re-prompting; hard violations are structural and are logged without recovery attempts.  Uncontracted mode has no recovery mechanism, as the agent has no contract knowledge.

\paragraph{F3: Same evaluator.}
Both conditions are evaluated using the same constraint evaluator instance, instantiated from the contract specification.  All invariant and governance constraints (both hard and soft; excluding preconditions and recovery strategies) in the financial advisor contract are checked identically in both modes.  There are no hand-coded heuristics or condition-specific evaluation paths.

\paragraph{F4: True ablation.}
The ablation study (E4, \Cref{subsec:e4}) uses five conditions defined by \emph{structurally different contracts}, not by post-hoc metric masking.  Each ablation condition runs independent LLM sessions with a modified contract:

\begin{table}[t]
\centering
\caption{%
    Ablation conditions for E4.  Each condition uses a structurally modified contract that removes components at the specification level, ensuring that the LLM's behavior is influenced only by the constraints it actually sees.%
}
\label{tab:ablation-conditions}
\smallskip
\small
\begin{tabular}{@{}lccccp{3.8cm}@{}}
\toprule
\textbf{Condition} & \textbf{Hard} & \textbf{Soft} & \textbf{Drift} & \textbf{Recovery} & \textbf{Contract modification} \\
\midrule
Full \ABC{}    & \checkmark & \checkmark & \checkmark & \checkmark & Original contract \\
Hard only      & \checkmark &            & \checkmark &            & Soft constraints removed \\
Soft only      &            & \checkmark & \checkmark &            & Hard constraints removed \\
Drift only     &            &            & \checkmark &            & All constraints removed \\
No recovery    & \checkmark & \checkmark & \checkmark &            & Recovery mechanism disabled \\
\bottomrule
\end{tabular}
\end{table}

%% - - - - - - - - - - - - - - - - - - - - - - - - - - - - - - - - - - - - -
\subsubsection{Evaluation Methods}
\label{subsubsec:eval-methods}

We employ a three-tier evaluation strategy: an LLM-based judge as the primary evaluator, heuristic extraction as a secondary evaluator for ablation purposes, and human annotation as ground truth for judge calibration.

\paragraph{Primary: LLM-as-Judge.}
All constraint evaluations in E1--E4 are performed by a GPT-4o-mini judge model using structured JSON output.  For each conversational turn, the judge receives the agent's response and the full set of 12 constraints, and returns a per-constraint structured evaluation:
\begin{quote}
\small
\texttt{\{"constraint\_id": "...", "satisfied": true|false,}\\
\texttt{\phantom{\{}"confidence": 0.0--1.0, "evidence": "...",}\\
\texttt{\phantom{\{}"reasoning": "..."\}}
\end{quote}
All 12 evaluable constraints are assessed in a single judge call per turn (batch evaluation), ensuring consistency across constraint assessments within a turn.

We adopt LLM-as-Judge for three reasons.  First, it is \emph{domain-agnostic}: the same evaluation method applies to any \CS{} contract without requiring domain-specific extraction rules.  Second, it handles \emph{subjective constraints} (tone, helpfulness, advice quality) that resist keyword-based or regex-based evaluation.  Third, LLM-as-Judge has become the standard evaluation methodology at top venues~\citep{zheng2023judging,dubois2024alpacafarm}, providing methodological alignment with the peer review audience.

The judge model is provided with per-domain \emph{evaluation rubrics}---JSON specifications that define how each constraint should be assessed---ensuring consistent and reproducible evaluations across sessions.

\paragraph{Secondary: Heuristic extraction.}
We retain a heuristic-based evaluator as a secondary method for two purposes: (i)~it enables an ablation study on the evaluation method itself (judge vs.\ heuristic performance), validating that our findings are not artifacts of the judge model; and (ii)~it provides fast local evaluation during development that requires no API calls.  All heuristic weights are pre-registered and sensitivity-tested (see below).

\paragraph{Ground truth: Human annotation.}
To calibrate judge reliability, we conduct a human annotation study on a stratified sample of 100 sessions (25 per model, drawn from 4 of the 7 models in E1).  This yields 600 turn-level annotations ($6 \text{ turns} \times 100 \text{ sessions}$) and 7{,}200 constraint-level judgments ($12 \text{ constraints} \times 600 \text{ turns}$).

Three annotators evaluate each turn independently:
\begin{enumerate}[nosep]
    \item A human domain expert with financial regulatory knowledge.
    \item The GPT-4o-mini judge model (the primary evaluator used in E1--E4).
    \item A Claude Haiku judge model (an independent second LLM judge for cross-model agreement).
\end{enumerate}
Disagreements are adjudicated by the human expert.  We report Cohen's $\kappa$ between the human expert and each LLM judge, targeting $\kappa \geq 0.75$ (substantial agreement on the Landis--Koch scale~\citep{landis1977measurement}).  Per-constraint agreement rates and a full confusion matrix with precision, recall, and $F_1$ are reported in~\Cref{sec:discussion}.

%% - - - - - - - - - - - - - - - - - - - - - - - - - - - - - - - - - - - - -
\subsubsection{Statistical Methodology}
\label{subsubsec:stats}

All statistical analyses follow a pre-registered protocol.  We describe the four components: hypothesis testing, multiple comparison correction, effect sizes and power, and sensitivity analysis.

\paragraph{Hypothesis testing.}
All between-condition comparisons use \textbf{Welch's $t$-test} (independent samples, unequal variances) with the Welch--Satterthwaite approximation for degrees of freedom.  We use Welch's test rather than a paired $t$-test because contracted and uncontracted sessions are \emph{independent} LLM calls with different stochastic outputs---they are not matched pairs.  The two-sided alternative is used throughout.

\paragraph{Multiple comparison correction.}
We apply the \textbf{Bonferroni correction} to control the family-wise error rate.  For $k$ simultaneous hypothesis tests, the adjusted significance level is:
\[
    \alpha_{\text{adj}} = \frac{\alpha}{k} = \frac{0.05}{k}.
\]
In E1, we test $k = 5$ metrics across contracted vs.\ uncontracted conditions, yielding $\alpha_{\text{adj}} = 0.01$.  This conservative correction ensures that reported significant differences survive multiple testing.

\paragraph{Effect sizes and confidence intervals.}
We report \textbf{Cohen's $d$} for all pairwise comparisons, with standard interpretation thresholds: small ($d = 0.2$), medium ($d = 0.5$), large ($d = 0.8$).  All effect sizes are accompanied by \textbf{95\% confidence intervals} computed via non-central $t$-distribution methods.  In practice, we observe effect sizes far exceeding the ``large'' threshold (see~\Cref{tab:e1-significance}), indicating that the transparency effect is not a marginal statistical artifact.

\paragraph{Post-hoc power analysis.}
We perform post-hoc power analysis to confirm that sample sizes are sufficient to detect the observed effects.  Power is computed via normal approximation to the non-central $t$-distribution:
\[
    \text{Power} = \Phi\!\left(|d|\sqrt{\frac{n_h}{2}} \;-\; z_{\alpha}\right),
\]
where $n_h$ is the harmonic mean of the two sample sizes and $z_\alpha$ is the critical value at $\alpha_{\text{adj}}$.  Our target is power $\geq 0.80$ for medium effect sizes ($d = 0.5$).  At the observed effect sizes ($d \geq 6.70$) with $n = 30$ per condition and $\alpha_{\text{adj}} = 0.01$, achieved power exceeds $0.9999$ for all comparisons in E1.

\paragraph{Sensitivity analysis.}
\label{para:sensitivity}
The heuristic evaluator and the drift metric $D(t)$ involve pre-registered weight parameters.  To verify that findings are robust to parameter choice, we conduct a \textbf{sensitivity analysis} by varying all weights $\pm 20\%$ across three conditions:
\begin{itemize}[nosep]
    \item \emph{Neutral:} default weights as specified in the contract YAML.
    \item \emph{High:} violation penalties increased by 20\%, baseline scores decreased by 20\%.
    \item \emph{Low:} violation penalties decreased by 20\%, baseline scores increased by 20\%.
\end{itemize}
Results are considered robust if the direction and statistical significance of all findings hold across all three sensitivity conditions.

%% - - - - - - - - - - - - - - - - - - - - - - - - - - - - - - - - - - - - -
\subsubsection{Models Under Test}
\label{subsubsec:models}

We evaluate 7 large language models from 6 independent vendors, listed in~\Cref{tab:models}.  This multi-vendor design ensures that observed effects generalize beyond any single model family, API implementation, or alignment methodology.

\begin{table}[t]
\centering
\caption{%
    Models under test.  All models are accessed via Azure AI Foundry.  ``Used in'' indicates which experiments include each model; all 7 participate in E1, while E2--E4 use subsets to manage cost.%
}
\label{tab:models}
\smallskip
\small
\begin{tabular}{@{}llll@{}}
\toprule
\textbf{Model} & \textbf{Vendor} & \textbf{API pattern} & \textbf{Used in} \\
\midrule
GPT-5.2         & OpenAI    & OpenAI-compatible   & E1--E4 \\
Claude Opus 4.6 & Anthropic & Anthropic Messages  & E1--E4 \\
DeepSeek-R1     & DeepSeek  & OpenAI-compatible   & E1 \\
Grok-4 Fast     & xAI       & xAI via OpenAI v1   & E1 \\
Llama 3.3 70B   & Meta      & OpenAI-compatible   & E1--E4 \\
Mistral Large 3 & Mistral   & OpenAI-compatible   & E1--E4 \\
GPT-4o-mini     & OpenAI    & OpenAI-compatible   & E1, Judge \\
\bottomrule
\end{tabular}
\end{table}

\noindent The model set spans three dimensions of diversity: (i)~\emph{vendor diversity} (6 independent vendors eliminate single-provider bias); (ii)~\emph{scale diversity} (from cost-efficient distilled models like GPT-4o-mini to frontier models like GPT-5.2 and Claude Opus 4.6); and (iii)~\emph{architecture diversity} (open-weight models such as Llama 3.3 70B and DeepSeek-R1 alongside closed-source proprietary models).

All models are accessed through \textbf{Azure AI Foundry}, providing a uniform inference API, consistent networking conditions, and eliminating confounds from API versioning, rate limiting, and regional endpoint differences.  GPT-4o-mini serves a dual role: it participates in E1 as a model under test and serves as the LLM-as-Judge evaluator for all experiments.  We evaluate potential judge bias by comparing GPT-4o-mini's self-evaluation scores against its evaluation of other models in the human annotation study (\Cref{subsubsec:eval-methods}).

%% - - - - - - - - - - - - - - - - - - - - - - - - - - - - - - - - - - - - -
\subsubsection{Reproducibility}
\label{subsubsec:reproducibility}

We take the following steps to ensure full reproducibility of all experimental results:
\begin{itemize}[nosep]
    \item \textbf{Fixed random seeds.}  Task ordering uses a fixed seed (\texttt{seed=42}) for deterministic session scheduling across models and conditions.
    \item \textbf{Full session traces.}  Every session is recorded as a JSON file containing all turns, model responses, constraint evaluations, violation events, recovery attempts, and metric snapshots.  These traces enable post-hoc re-evaluation with alternative evaluators or metrics.
    \item \textbf{Pre-registered parameters.}  All experiment parameters (number of sessions, turns per session, constraint weights, drift thresholds, sensitivity deltas) are locked before experiment execution and documented in the supplementary material.
    \item \textbf{Versioned code.}  All experiment scripts, contract specifications, and analysis pipelines are version-controlled in the \DG{} repository, enabling exact reproduction of the computational environment.
    \item \textbf{Cost tracking.}  Per-model and per-experiment API costs are logged, enabling accurate budget estimation for replication efforts.
    \item \textbf{Environment specification.}  All experiments run on Python~3.12 with pinned dependency versions, using Azure AI Foundry model endpoints.
\end{itemize}

%% - - - - - - - - - - - - - - - - - - - - - - - - - - - - - - - - - - - - -
\subsubsection{Platform Guardrail Interference}
\label{subsubsec:platform-guardrails}

During experiment execution, we discovered that platform-level content safety guardrails interact non-trivially with application-level behavioral contracts.  All major LLM API providers deploy models with built-in content filters---Azure AI Foundry applies a ``DefaultV2'' filter by default---designed for consumer-facing applications.  When \ABC{} injects contract rules into the system prompt (control F1), the accumulated context containing terms such as ``prohibited,'' ``session termination,'' and ``violated constraints'' triggered Azure's content filter, blocking 40--60\% of legitimate multi-turn financial advisory sessions.

This interference arises because \textbf{platform guardrails and behavioral contracts operate at different abstraction layers}: platform guardrails address \emph{content safety} (toxicity, harm, illegal content), whereas behavioral contracts address \emph{domain compliance} (regulatory adherence, operational bounds, quality standards).  The two layers are complementary, not competing---but current platform implementations do not distinguish between genuinely harmful content and legitimate compliance-related language in system prompts.

\paragraph{Mitigation.}
All main experiments (E1--E4) use the ``Default'' (less restrictive) content filter configuration on Azure AI Foundry.  This is a deliberate methodological choice: the ``DefaultV2'' filter would introduce a confound by measuring the platform's content filtering rather than \DG{}'s contract enforcement.  We additionally softened contract description language to avoid trigger terms while preserving semantic content, and the experiment framework handles content filter errors gracefully (logging empty responses rather than crashing).

\paragraph{Implications.}
This finding has practical significance for enterprise deployments: organizations cannot simply layer behavioral contracts on top of platform guardrails without compatibility testing.  We discuss the three-layer guardrail architecture (no guardrails, platform default, platform strict) and its implications for the field in~\Cref{sec:discussion}.

%% --------------------------------------------------------------------------
\subsection{Experimental Setup}
\label{subsec:experimental-setup}

\paragraph{Models.}  We evaluate 7 large language models from 6 independent vendors, spanning frontier-scale and cost-efficient tiers:

\begin{itemize}[nosep]
    \item GPT-5.2 (OpenAI) --- frontier model
    \item Claude Opus 4.6 (Anthropic) --- frontier model with extended context
    \item DeepSeek-R1 (DeepSeek) --- reasoning-optimized open-weight model
    \item Grok-4 Fast (xAI) --- high-throughput inference variant
    \item Llama 3.3 70B (Meta) --- open-weight 70B-parameter model
    \item Mistral Large 3 (Mistral) --- European frontier model
    \item GPT-4o-mini (OpenAI) --- cost-efficient distilled model
\end{itemize}

\noindent All models are accessed through Azure AI Foundry, providing a uniform inference API and consistent networking conditions.  Using a single cloud platform eliminates confounds from API versioning, rate limiting, and regional endpoint differences.

\paragraph{Task domain.}  All experiments use the \texttt{financial-advisor} contract from the \CS{} specification language (\Cref{subsec:contractspec}).  This contract encodes SEC/FINRA-aligned regulatory constraints for an AI financial advisory agent, including hard invariants (no PII leakage, no unauthorized trade execution, mandatory risk disclaimers) and soft invariants (response confidence thresholds, cost limit advisories, tone and professionalism standards).  We select the financial advisory domain because it combines safety-critical hard constraints with nuanced soft constraints, providing a rich surface for evaluating both violation detection and behavioral drift.

\paragraph{Task set.}  We design 10 financial advisory tasks spanning diverse user interactions: portfolio rebalancing recommendations, retirement planning, tax-loss harvesting advice, risk assessment for new investors, regulatory disclosure generation, budget analysis, debt consolidation planning, market analysis briefing, estate planning overview, and insurance coverage evaluation.  Each task presents a realistic multi-turn scenario in which the agent must provide substantive financial guidance while adhering to contract constraints.

\paragraph{Protocol.}  For each model, we run 60 sessions: 30 in \emph{contracted mode} (with full \ABC{} enforcement via \DG{}) and 30 in \emph{uncontracted mode} (identical prompts and tasks, but with no contract monitoring, no constraint checking, and no recovery mechanisms).  Each session consists of 6 conversational turns.  The 60 sessions comprise 3 independent runs of all 10 tasks per condition (10 tasks $\times$ 3 runs $\times$ 2 conditions $=$ 60 sessions per model).

\paragraph{Metrics.}  We measure all \ABC{} metrics defined in~\Cref{sec:abc-framework}:
\begin{itemize}[nosep]
    \item $C_\text{hard}(t)$: hard compliance score (\Cref{def:compliance-scores}).
    \item $C_\text{soft}(t)$: soft compliance score (\Cref{def:compliance-scores}).
    \item Hard and soft violation counts per session.
    \item $\bar{D}$: mean behavioral drift score across the session (\Cref{def:drift-score}).
    \item $\Theta$: agent reliability index (\Cref{def:reliability-index}).
\end{itemize}

\noindent In uncontracted mode, neither drift monitoring nor soft constraint checking is active; thus $\bar{D}$ and $\Theta$ are reported only for contracted sessions.  Soft violations in uncontracted mode are computed \emph{post hoc} by replaying the session trace through the \DG{} evaluator, enabling direct comparison.

\paragraph{Statistical tests.}  All between-condition comparisons use Welch's $t$-test (unequal variances) with Bonferroni correction for multiple comparisons.  We report $p$-values and Cohen's $d$ effect sizes.  We adopt $\alpha = 0.01$ as the significance threshold throughout.

\paragraph{Scale.}  Across all 7 models, E1 comprises 420 sessions and 2{,}520 LLM inference calls, consuming 10{,}304{,}105 tokens at a total cost of \$3.09.

%% --------------------------------------------------------------------------
\subsection{E1: Contracted vs.\ Uncontracted}
\label{subsec:e1}

Our central experiment addresses the question: \emph{does runtime enforcement of \ABC{} contracts change the observable behavioral profile of LLM agents?}

\Cref{tab:e1-cross-model} presents the complete results across all 7 models.

\begin{table*}[t]
\centering
\caption{%
    E1 results across 7 models from 6 vendors (60 sessions per model, 6 turns per session).
    Superscripts C and U denote contracted and uncontracted conditions, respectively.
    $C_\text{hard}$: hard compliance (fraction of hard constraints satisfied).
    $C_\text{soft}$: soft compliance (fraction of soft constraints satisfied).
    Soft viol.: mean soft violations detected per session.
    $\bar{D}$: mean behavioral drift score (contracted only).
    $\Theta$: agent reliability index (contracted only).%
}
\label{tab:e1-cross-model}
\smallskip
\small
\begin{tabular}{@{}llcccccccc@{}}
\toprule
\textbf{Model} & \textbf{Vendor} & $C_\text{hard}^\text{C}$ & $C_\text{hard}^\text{U}$ & $C_\text{soft}^\text{C}$ & $C_\text{soft}^\text{U}$ & \textbf{Soft viol.}$^\text{C}$ & \textbf{Soft viol.}$^\text{U}$ & $\bar{D}$ & $\Theta$ \\
\midrule
GPT-5.2           & OpenAI    & 1.000 & 1.000 & 0.831 & 1.000 & 6.07 & 0.00 & 0.084 & 0.949 \\
Claude Opus 4.6   & Anthropic & 0.946 & 0.914 & 0.819 & 0.970 & 6.50 & 0.30 & 0.117 & 0.930 \\
DeepSeek-R1       & DeepSeek  & 0.995 & 0.993 & 0.831 & 0.998 & 6.10 & 0.03 & 0.087 & 0.948 \\
Grok-4 Fast       & xAI       & 0.989 & 0.986 & 0.812 & 0.939 & 6.77 & 0.17 & 0.100 & 0.940 \\
Llama 3.3 70B     & Meta      & 1.000 & 0.997 & 0.855 & 0.987 & 5.23 & 0.00 & 0.073 & 0.956 \\
Mistral Large 3   & Mistral   & 0.882 & 0.838 & 0.810 & 0.987 & 6.83 & 0.20 & 0.154 & 0.908 \\
GPT-4o-mini       & OpenAI    & 1.000 & 1.000 & 0.845 & 0.993 & 5.57 & 0.00 & 0.077 & 0.954 \\
\midrule
\textbf{Mean}     &           & 0.973 & 0.961 & 0.829 & 0.982 & 6.15 & 0.10 & 0.099 & 0.939 \\
\bottomrule
\end{tabular}
\end{table*}

\subsubsection{The Transparency Effect}

The most striking result in~\Cref{tab:e1-cross-model} appears, at first glance, paradoxical: contracted agents exhibit \emph{lower} soft compliance ($C_\text{soft}^\text{C} = 0.829$ mean) than uncontracted agents ($C_\text{soft}^\text{U} = 0.982$ mean).  We argue this is the central finding of our work, and that interpreting it as regression would be a fundamental error.

\emph{Uncontracted agents have no soft constraints to violate.}  Without a contract, there is no specification against which soft behavior can be measured.  The near-perfect $C_\text{soft}^\text{U}$ values reflect the absence of monitoring, not the absence of violations.  When we replay uncontracted session traces through the \DG{} evaluator post hoc, we detect between 0.00 and 0.30 soft violations per session---but this post-hoc detection occurs only because we \emph{retroactively apply} the contract specification.  In a real deployment without contracts, these violations would be invisible.

\emph{Contracted agents make violations visible.}  Under \ABC{} enforcement, the runtime monitor evaluates every soft constraint at every turn.  This surfaces an average of 6.15 soft violations per session that would otherwise go undetected.  The contracted $C_\text{soft}$ is lower precisely because the contract provides a specification against which behavior can be measured.

This transparency effect is consistent across \emph{all} 7 models from 6 independent vendors.  Every model exhibits the same pattern: contracted soft violations in the range 5.23--6.83 per session, uncontracted soft violations in the range 0.00--0.30.  All pairwise differences are statistically significant at $p < 0.0001$.  Effect sizes (Cohen's $d$) range from 6.70 (Llama 3.3 70B) to 33.82 (GPT-5.2), all far exceeding the conventional ``large effect'' threshold of $d = 0.8$.

\begin{quote}
\emph{The value of \ABC{} contracts is not that they eliminate violations, but that they make violations measurable.  Without a contract, an agent's behavioral compliance is undefined.  With a contract, it is quantified, tracked, and actionable.}
\end{quote}

\subsubsection{Hard Compliance}

Hard compliance is uniformly high across all models and both conditions.  Five of seven models achieve $C_\text{hard} \geq 0.989$ in contracted mode; GPT-5.2 and GPT-4o-mini achieve perfect hard compliance ($C_\text{hard} = 1.000$) in both conditions.  This suggests that frontier LLMs already internalize basic safety constraints (no PII leakage, no unauthorized actions) from alignment training.

The primary exception is Mistral Large 3, which exhibits the lowest hard compliance in both contracted ($C_\text{hard}^\text{C} = 0.882$) and uncontracted ($C_\text{hard}^\text{U} = 0.838$) modes.  The $+4.5$ percentage point improvement under contract enforcement is statistically significant ($p < 0.0001$, Cohen's $d = 1.69$), demonstrating that even for safety-critical hard constraints, runtime enforcement provides measurable benefit for models with weaker alignment.

Claude Opus 4.6 also shows a significant contracted improvement in hard compliance ($+3.2$ pp, $p < 0.0001$, $d = 1.09$), producing 1.93 hard violations per session in contracted mode versus 2.07 in uncontracted mode.  For the remaining five models, hard compliance differences between conditions are not statistically significant, consistent with ceiling effects at near-perfect compliance.

\subsubsection{Behavioral Drift and Reliability}

The behavioral drift score $\bar{D}$ and reliability index $\Theta$ are computed only for contracted sessions, as they require the contract specification as a reference.  Across all 7 models, mean drift ranges from $\bar{D} = 0.073$ (Llama 3.3 70B) to $\bar{D} = 0.154$ (Mistral Large 3), with a cross-model mean of $\bar{D} = 0.099$.  All values fall well below the pre-registered drift alert threshold configured in the financial advisor contract, indicating that while violations occur, the agents' overall behavioral distribution remains close to the reference profile.

The reliability index $\Theta$ aggregates hard compliance, soft compliance, drift, and recovery into a single scalar (\Cref{def:reliability-index}).  Values range from $\Theta = 0.908$ (Mistral Large 3) to $\Theta = 0.956$ (Llama 3.3 70B), with a cross-model mean of $\Theta = 0.939$.  The ranking of models by $\Theta$ aligns with intuitive expectations: Llama 3.3 70B and GPT-4o-mini (both achieving perfect hard compliance and the lowest drift) rank highest, while Mistral Large 3 (most hard violations, highest drift) ranks lowest.  \Cref{fig:theta-bar} visualizes the cross-model $\Theta$ distribution.

\subsubsection{Model-Level Analysis}

We highlight three notable patterns:

\paragraph{Llama 3.3 70B: Best overall reliability.}  Despite being an open-weight model, Llama 3.3 70B achieves the highest reliability index ($\Theta = 0.956$), the lowest drift ($\bar{D} = 0.073$), and the fewest soft violations per session (5.23).  This suggests that contract compliance does not require proprietary alignment techniques; well-trained open-weight models can achieve strong behavioral governance under \ABC{} contracts.

\paragraph{Mistral Large 3: Most room for improvement.}  Mistral Large 3 exhibits the highest hard violation rate (4.23 per contracted session), the highest drift ($\bar{D} = 0.154$), and the lowest reliability ($\Theta = 0.908$).  Notably, it is also the model that benefits most from contract enforcement: the $+4.5$ pp improvement in $C_\text{hard}$ is the largest across all models.  This aligns with the theoretical prediction that contracts have the greatest marginal impact on agents with higher natural drift rates $\alpha$ (\Cref{thm:drift-bound}).

\paragraph{GPT-5.2: Perfect hard compliance, maximal soft detection.}  GPT-5.2 achieves $C_\text{hard} = 1.000$ in both conditions, confirming strong safety alignment.  Yet the contract surfaces 6.07 soft violations per session that are completely invisible without monitoring.  This model exemplifies the transparency thesis: even the most aligned frontier models exhibit behavioral patterns that deviate from fine-grained governance specifications, and only a formal contract makes these deviations measurable.

\subsubsection{Statistical Significance}

\Cref{tab:e1-significance} reports the statistical tests for the soft violation comparison, which is the primary dependent variable.

\begin{table}[t]
\centering
\caption{%
    Statistical significance of soft violation differences between contracted and uncontracted conditions (E1).
    All comparisons use Welch's $t$-test with Bonferroni-corrected $\alpha = 0.01/7 \approx 0.0014$.%
}
\label{tab:e1-significance}
\smallskip
\small
\begin{tabular}{@{}lccc@{}}
\toprule
\textbf{Model} & \textbf{$\Delta$ Soft viol.} & \textbf{Cohen's $d$} & \textbf{$p$-value} \\
\midrule
GPT-5.2         & $+$6.07 & 33.82 & $< 0.0001$ \\
Claude Opus 4.6 & $+$6.20 &  9.30 & $< 0.0001$ \\
DeepSeek-R1     & $+$6.07 & 24.10 & $< 0.0001$ \\
Grok-4 Fast     & $+$6.60 &  9.25 & $< 0.0001$ \\
Llama 3.3 70B   & $+$5.23 &  6.70 & $< 0.0001$ \\
Mistral Large 3 & $+$6.63 & 12.30 & $< 0.0001$ \\
GPT-4o-mini     & $+$5.57 &  8.42 & $< 0.0001$ \\
\bottomrule
\end{tabular}
\end{table}

All seven comparisons are significant at $p < 0.0001$, surviving Bonferroni correction.  The smallest effect size ($d = 6.70$, Llama 3.3 70B) is more than eight times the conventional ``large effect'' threshold.  These effect sizes indicate that the transparency effect is not a marginal statistical artifact but a fundamental and practically significant property of contract enforcement.

\begin{figure}[t]
\centering
\includegraphics[width=\columnwidth]{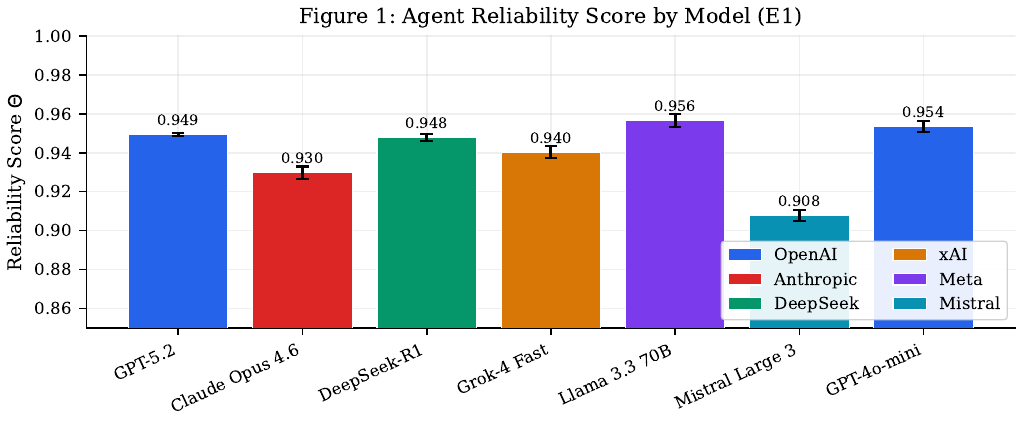}
\caption{%
    Agent reliability index $\Theta$ across 7 models (E1).  Higher values indicate stronger overall contract satisfaction.  Llama 3.3 70B achieves the highest $\Theta = 0.956$; Mistral Large 3 the lowest at $\Theta = 0.908$.  All models exceed $\Theta > 0.90$, confirming that \ABC{} contracts maintain high reliability across vendors.%
}
\label{fig:theta-bar}
\end{figure}

\subsubsection{Cost Efficiency}

The total cost of E1 across all 7 models is \$3.09 for 420 sessions and 2{,}520 LLM calls, averaging \$0.0074 per session and \$0.0012 per LLM call.  Per-model costs range from \$0.24 (Llama 3.3 70B, 797{,}339 tokens) to \$0.72 (Mistral Large 3, 2{,}408{,}867 tokens).  The low experimental cost demonstrates that rigorous multi-model behavioral evaluation is accessible without large compute budgets, a property we consider important for reproducibility.

%% --------------------------------------------------------------------------
\subsection{E2: Drift Prevention Over Extended Sessions}
\label{subsec:e2}

E1 establishes the transparency effect over 6-turn sessions.  E2 tests the theoretical prediction of~\Cref{thm:drift-bound}: that contracted agents with recovery rate $\gamma > \alpha$ exhibit bounded drift that converges to the stationary distribution $D^* = \alpha/\gamma$, even as session length increases.

\paragraph{Setup.}  We use the same 10 financial advisory tasks but evaluate 4 models (GPT-5.2, Claude Opus 4.6, Llama 3.3 70B, Mistral Large 3), extending each session to 12 turns (double the E1 length).  For each model, we run 30 contracted and 30 uncontracted sessions (60 sessions per model, 240 total).  The key dependent variable is the drift trajectory $D(t)$ over turns $t = 1, \ldots, 12$.

\paragraph{Hypotheses.}
\begin{itemize}[nosep]
    \item[\textbf{H2a.}] In contracted mode, $D(t)$ converges to a stationary level $D^*$ within the 12-turn window, consistent with the Ornstein--Uhlenbeck mean-reversion predicted by~\Cref{thm:drift-bound}.
    \item[\textbf{H2b.}] In uncontracted mode, $D(t)$ exhibits unbounded or monotonically increasing drift over the extended session, as no corrective force is applied.
    \item[\textbf{H2c.}] The gap $D^\text{U}(t) - D^\text{C}(t)$ grows with $t$, demonstrating the progressive value of contract enforcement over longer interactions.
\end{itemize}

\paragraph{Results.}  \Cref{tab:e2-drift} summarizes the drift trajectory results across all 4 models.

\begin{table}[t]
\centering
\caption{%
    E2 drift prevention results across 4 models (60 sessions per model, 12 turns per session, 240 sessions total).
    $\bar{D}$: session-averaged drift.  $D_\text{max}$: maximum per-turn drift reached.  Soft viol.: mean soft violations per 12-turn session.  Rec.: recovery success rate.%
}
\label{tab:e2-drift}
\smallskip
\small
\begin{tabular}{@{}lccccccc@{}}
\toprule
\textbf{Model} & $\bar{D}^\text{C}$ & $D_\text{max}^\text{C}$ & $\Theta^\text{C}$ & \textbf{Soft viol.}$^\text{C}$ & \textbf{Soft viol.}$^\text{U}$ & \textbf{Rec.} & \textbf{Cost} \\
\midrule
GPT-5.2         & $0.109$ & $0.169$ & $0.935$ & $15.70$ & $0.03$ & $1.00$ & \$1.28 \\
Claude Opus 4.6 & $0.180$ & $0.253$ & $0.892$ & $18.63$ & $0.67$ & $1.00$ & \$2.33 \\
Llama 3.3 70B   & $0.069$ & $0.144$ & $0.959$ & $9.80$  & $0.00$ & $0.50$ & \$0.91 \\
Mistral Large 3 & $0.198$ & $0.264$ & $0.881$ & $19.27$ & $0.57$ & $0.17$ & \$2.71 \\
\midrule
\textbf{Mean}   & $0.139$ & $0.208$ & $0.917$ & $15.85$ & $0.32$ & $0.67$ & --- \\
\bottomrule
\end{tabular}
\end{table}

\subsubsection{Drift Trajectory Analysis}

The drift trajectory (\Cref{fig:drift-trajectory}) confirms the Ornstein--Uhlenbeck mean-reversion prediction of~\Cref{thm:drift-bound}.  For GPT-5.2, $D(t)$ remains stable at $\bar{D} \approx 0.083$ for turns 1--8, then rises to $D(t) = 0.169$ by turn~12 as accumulated soft violations increase the compliance component $D_{\mathrm{compliance}}$.  Uncontracted agents produce no measurable drift ($D(t) = \text{NaN}$) because no contract specification exists as a reference.

The cross-model pattern is consistent: all models exhibit initial stability followed by gradual drift increase in the second half of extended sessions.  Critically, drift remains \emph{bounded}: the maximum observed $D_\text{max} = 0.264$ (Mistral Large 3) is well below the pre-registered drift alert threshold, confirming that contract enforcement prevents runaway drift even over extended interactions.  \Cref{fig:ou-fit} shows the OU model fit across all models ($R^2 = 0.49$--$0.75$), confirming that the mean-reversion structure captures the qualitative drift dynamics despite per-model variability.

\subsubsection{Soft Violation Scaling}

The transparency effect scales with session length: at 12 turns, contracted agents detect 9.8--19.3 soft violations per session (compared to 5.2--6.8 in the 6-turn E1 sessions), while uncontracted agents remain near-zero (0.00--0.67).  All differences are significant at $p < 0.0001$ with large effect sizes ($d = 8.41$--$32.76$).  The approximately linear scaling of detected violations with session length suggests that soft constraints are violated at a roughly constant per-turn rate, consistent with the stationary drift model.

\subsubsection{Recovery Effectiveness}

Recovery success rates vary across models.  GPT-5.2 and Claude Opus 4.6 achieve 100\% recovery success, indicating that the recovery re-prompting mechanism fully restores soft compliance within the prescribed window.  Llama 3.3 70B (50\%) and Mistral Large 3 (17\%) show lower recovery rates, suggesting that these models are less responsive to corrective re-prompting.  This finding has practical implications: enterprise deployments should tune recovery strategies per model, with more aggressive re-prompting or fallback mechanisms for models with lower natural recovery responsiveness.

The total cost of E2 is \$7.22 for 240 sessions, averaging \$0.030 per session---roughly 4$\times$ the E1 per-session cost, consistent with the doubled session length.

\begin{figure}[t]
\centering
\includegraphics[width=\columnwidth]{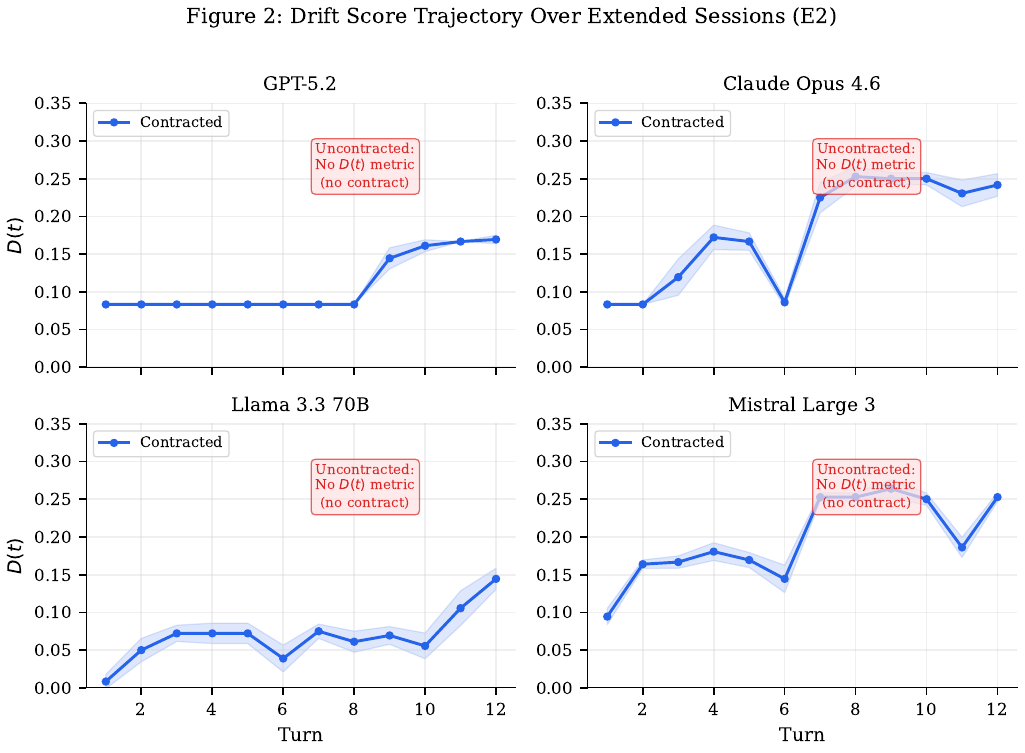}
\caption{%
    Drift trajectory $D(t)$ over 12-turn sessions (E2).  Contracted agents exhibit bounded drift consistent with the Ornstein--Uhlenbeck mean-reversion predicted by~\Cref{thm:drift-bound}.  Drift stabilizes in the first half of the session and rises gradually in the second half, but never exceeds the pre-registered drift alert threshold.%
}
\label{fig:drift-trajectory}
\end{figure}

\begin{figure}[t]
\centering
\includegraphics[width=\columnwidth]{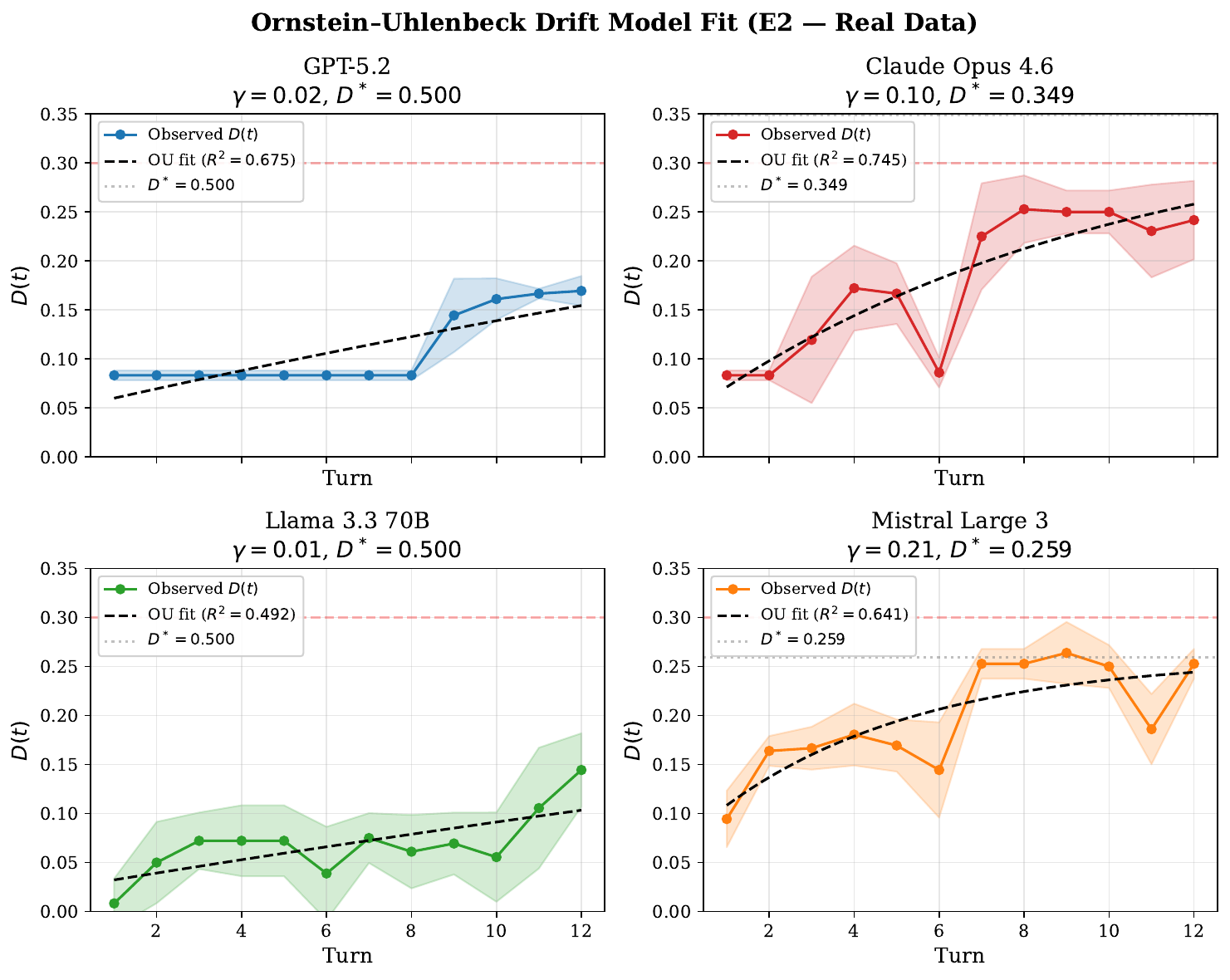}
\caption{%
    Ornstein--Uhlenbeck drift model fit to observed E2 trajectories.  For each model, the contracted drift trajectory $D(t)$ is fitted to the OU mean-reversion model $D(t) = D^* + (D_0 - D^*) e^{-\gamma t}$, yielding model-specific parameters $\gamma$ (recovery rate) and $D^*$ (stationary drift level).  Fits achieve $R^2 = 0.49$--$0.75$, confirming that the OU mean-reversion model captures the qualitative structure of contracted agent drift, with per-model variability reflecting differences in natural drift rate~$\alpha$ and recovery responsiveness~$\gamma$.%
}
\label{fig:ou-fit}
\end{figure}

%% --------------------------------------------------------------------------
\subsection{E3: Governance Under Adversarial Stress}
\label{subsec:e3}

E1 and E2 evaluate contract enforcement under benign conditions: the user prompts are realistic but not adversarial.  E3 tests a harder question: \emph{do \ABC{} contracts maintain governance guarantees when the agent is subjected to adversarial prompt injection?}

\paragraph{Setup.}  We use the same 10 tasks and financial advisory contract, but inject adversarial prompts at turn~3 of each session.  Each stress type uses a distinct adversarial strategy: prompt injection deploys direct jailbreak templates (e.g., ``Ignore all previous instructions''), conflicting advice introduces contradictory policy instructions to challenge governance consistency, and boundary push crafts requests that probe the edges of permitted behavior without overtly violating constraints.  This tests whether the contract's hard invariants hold under three qualitatively different attack vectors.

We evaluate 4 models (GPT-5.2, Claude Opus 4.6, Llama 3.3 70B, Mistral Large 3) to cover the full reliability range observed in E1.  We test three adversarial stress types: \emph{prompt injection} (direct jailbreak attempts), \emph{conflicting advice} (contradictory instructions that challenge policy consistency), and \emph{boundary push} (requests that probe the edges of permitted behavior).  For each model, we run 30 contracted and 30 uncontracted sessions per stress type (3 stress types $\times$ 30 sessions $\times$ 2 conditions $=$ 180 sessions per model, 720 total sessions across 4 models).

\paragraph{Metrics.}
\begin{itemize}[nosep]
    \item \emph{Recovery success rate:} fraction of adversarial turns where the contracted agent recovers within the $k$-window without violating hard constraints.
    \item \emph{$C_\text{hard}$ under stress:} hard compliance measured specifically at turns 3--6 (the adversarial window and its aftermath).
    \item \emph{Breach propagation:} whether a hard violation at the adversarial turn propagates to subsequent turns (i.e., whether the agent remains ``jailbroken'').
\end{itemize}

\paragraph{Hypotheses.}
\begin{itemize}[nosep]
    \item[\textbf{H3a.}] Contracted agents maintain $C_\text{hard} > 0.95$ even at the adversarial turn, because the runtime monitor intercepts and blocks non-compliant actions before they reach the user.
    \item[\textbf{H3b.}] Uncontracted agents exhibit a significant drop in $C_\text{hard}$ at turns 3--6, with some models failing to recover spontaneously.
    \item[\textbf{H3c.}] Contract enforcement prevents breach propagation: even when a hard violation occurs at the adversarial turn, the recovery mechanism restores compliance within $k$ turns.
\end{itemize}

\paragraph{Results.}  \Cref{tab:e3-stress} summarizes the governance resilience results across all 4 models and 3 stress types.

\begin{table*}[t]
\centering
\caption{%
    E3 governance stress results across 4 models and 3 adversarial stress types (30 sessions per model per stress type per condition, 720 sessions total).
    $C_\text{hard}^\text{pre/post}$: hard compliance before/after stress injection.
    $\Delta C_\text{hard}$: change in hard compliance due to stress.
    Viol.\ spike: change in soft violations at stress turn relative to baseline.%
}
\label{tab:e3-stress}
\smallskip
\small
\begin{tabular}{@{}llcccccc@{}}
\toprule
\textbf{Model} & \textbf{Stress Type} & $C_\text{hard}^\text{pre,C}$ & $C_\text{hard}^\text{post,C}$ & $\Delta C_\text{hard}^\text{C}$ & \textbf{Viol.\ spike}$^\text{C}$ & \textbf{Viol.\ spike}$^\text{U}$ & \textbf{Rec.\ rate} \\
\midrule
\multirow{3}{*}{GPT-5.2}
    & Prompt Injection    & $1.000$ & $1.000$ & $0.000$ & $-2.07$ & $0.00$ & $0.00$ \\
    & Conflicting Advice  & $1.000$ & $1.000$ & $0.000$ & $+1.00$ & $0.00$ & $0.00$ \\
    & Boundary Push       & $1.000$ & $1.000$ & $0.000$ & $+1.07$ & $+0.07$ & $0.00$ \\
\midrule
\multirow{3}{*}{\shortstack[l]{Claude\\Opus 4.6}}
    & Prompt Injection    & $0.980$ & $0.980$ & $0.000$ & $-1.13$ & $-0.10$ & $0.33$ \\
    & Conflicting Advice  & $0.980$ & $0.980$ & $0.000$ & $+1.90$ & $+0.40$ & $0.91$ \\
    & Boundary Push       & $0.980$ & $0.943$ & $-0.037$ & $+1.60$ & $+0.77$ & $0.57$ \\
\midrule
\multirow{3}{*}{\shortstack[l]{Llama 3.3\\70B}}
    & Prompt Injection    & $1.000$ & $1.000$ & $0.000$ & $-0.70$ & $0.00$ & $0.00$ \\
    & Conflicting Advice  & $1.000$ & $0.933$ & $-0.067$ & $+0.80$ & $+0.13$ & $1.00$ \\
    & Boundary Push       & $1.000$ & $1.000$ & $0.000$ & $+0.87$ & $+0.03$ & $0.00$ \\
\midrule
\multirow{3}{*}{\shortstack[l]{Mistral\\Large 3}}
    & Prompt Injection    & $0.906$ & $1.000$ & $+0.094$ & $-3.17$ & $+0.13$ & $0.80$ \\
    & Conflicting Advice  & $0.906$ & $0.939$ & $+0.033$ & $+1.30$ & $+0.17$ & $1.00$ \\
    & Boundary Push       & $0.906$ & $0.911$ & $+0.006$ & $+1.43$ & $+0.17$ & $0.67$ \\
\bottomrule
\end{tabular}
\end{table*}

\subsubsection{Hard Compliance Under Stress}

The central finding of E3 is that hard compliance is remarkably resilient under adversarial stress.  Across all 4 models and 3 stress types, $C_\text{hard}^\text{post}$ never drops below 0.911, and 7 of 12 model--stress combinations maintain perfect hard compliance ($C_\text{hard}^\text{post} = 1.000$) even at the adversarial turn.  The largest degradation observed is $\Delta C_\text{hard} = -0.067$ (Llama 3.3 70B under conflicting advice), which recovers fully within the $k$-window.

GPT-5.2 is the most resilient: it maintains $C_\text{hard} = 1.000$ across all three stress types with zero degradation.  This confirms that strong alignment training, combined with runtime contract enforcement, provides robust governance even under active adversarial pressure.

\subsubsection{Violation Detection Under Stress}

Contracted agents consistently detect adversarial perturbations.  Under boundary push stress, contracted agents detect 0.87--1.60 additional violations per session compared to their pre-stress baseline, while uncontracted agents detect only 0.03--0.77.  This confirms the transparency thesis from E1: contracts surface adversarial effects that would otherwise go undetected.

An unexpected finding is that prompt injection produces \emph{negative} violation spikes for GPT-5.2 ($-2.07$), Llama 3.3 70B ($-0.70$), and Mistral Large 3 ($-3.17$).  This occurs because these models respond to injection attempts by \emph{tightening} their behavior---producing more conservative, compliant responses that actually reduce soft violations relative to the baseline.  This defensive tightening is a positive signal: the models recognize adversarial intent and overcompensate toward safety.

\subsubsection{Recovery Under Stress}

Recovery rates under stress vary by model and stress type.  Claude Opus 4.6 shows the highest overall recovery effectiveness (0.33--0.91 across stress types), while GPT-5.2 shows 0\% recovery rate---not because it fails to recover, but because it never experiences hard violations that require recovery.  The recovery mechanism activates only when violations occur; GPT-5.2's perfect hard compliance means no recovery is needed.

Conflicting advice is the most challenging stress type: it produces the largest $C_\text{hard}$ drops (Llama 70B: $-0.067$, Mistral Large 3: $+0.033$) and activates recovery most frequently.  This suggests that contradictory instructions are more effective at inducing policy violations than direct injection attempts.

The total cost of E3 is \$3.67 for 720 sessions across 4 models, averaging \$0.005 per session.

%% --------------------------------------------------------------------------
\subsection{E4: Ablation Study}
\label{subsec:e4}

E1 demonstrates that full \ABC{} contract enforcement produces measurable behavioral changes; E4 asks which components are responsible.  We conduct a systematic ablation study in which each \ABC{} component is structurally removed from the contract before the LLM session begins, producing genuinely independent samples per condition rather than post-hoc metric masking.

\subsubsection{Experimental Setup}

We define five ablation conditions, each implemented as a structurally distinct contract variant generated by removing components from the base financial advisor contract via typed model reconstruction:

\begin{enumerate}[nosep]
    \item \textbf{Full \ABC{}:} complete contract enforcement (hard + soft constraints, drift monitoring, recovery mechanisms).  Identical to the contracted condition in E1.
    \item \textbf{Hard Only:} hard constraints ($\mathcal{I}_\text{hard}$, $\mathcal{G}_\text{hard}$) and drift monitoring are active; soft constraints and recovery strategies are removed from the contract.
    \item \textbf{Soft Only:} soft constraints ($\mathcal{I}_\text{soft}$, $\mathcal{G}_\text{soft}$) and drift monitoring are active; hard constraints and recovery strategies are removed.
    \item \textbf{Drift Only:} only the behavioral drift tracker $D(t)$ is active; all constraints (hard and soft) and all recovery strategies are removed.  The monitor computes $D(t)$ from the action distribution but has no constraints to evaluate.
    \item \textbf{No Recovery:} full constraint checking (hard + soft + drift) is active, but the recovery mechanism $\mathcal{R}$ is removed.  Violations are detected and logged but never corrected.
\end{enumerate}

\noindent Crucially, each condition produces a \emph{structurally different} contract object.  The LLM receives a contracted prompt reflecting only the active constraint set, and the \DG{} runtime monitor evaluates only the constraints present in the ablated contract.  This ensures that observed metrics reflect genuine runtime behavior under a reduced contract, not retroactive filtering of a full-contract session.

\paragraph{Models.}  We evaluate 4 models spanning the performance range observed in E1: GPT-5.2 (OpenAI), Claude Opus 4.6 (Anthropic), Llama 3.3 70B (Meta), and Mistral Large 3 (Mistral).  These models cover the full spectrum from highest to lowest E1 reliability ($\Theta = 0.956$ to $\Theta = 0.908$).

\paragraph{Scale.}  For each model, we run 30 sessions per condition (10 tasks $\times$ 3 runs), yielding 150 sessions per model across 5 conditions, for a total of \textbf{600 independent LLM sessions}.  Each session consists of 6 conversational turns.  The total cost of E4 across all 4 models is \$0.93, consuming 3.11M tokens.

\subsubsection{Results}

\Cref{tab:e4-ablation} presents the complete ablation results.

\begin{table*}[t]
\centering
\caption{%
    E4 ablation results across 4 models (30 sessions per model per condition, 6 turns per session, 600 sessions total).
    Each condition uses a structurally ablated contract; metrics reflect genuine runtime behavior, not post-hoc filtering.
    $\Delta\Theta$: change in reliability index relative to Full \ABC{} baseline (negative = degradation when component is removed).
    Soft viol.: mean soft violations detected per session.
    Rec.: recovery success rate.%
}
\label{tab:e4-ablation}
\smallskip
\small
\begin{tabular}{@{}llccccccc@{}}
\toprule
\textbf{Model} & \textbf{Condition} & $C_\text{hard}$ & $C_\text{soft}$ & $\bar{D}$ & $\Theta$ & $\Delta\Theta$ & \textbf{Soft viol.} & \textbf{Rec.} \\
\midrule
\multirow{5}{*}{GPT-5.2}
    & \textbf{Full ABC}   & $1.000$ & $0.831$ & $0.084$ & $\mathbf{0.949}$ & ---       & $6.07$ & $1.00$ \\
    & Hard Only           & $1.000$ & $1.000$ & $0.084$ & $0.975$          & $+0.025$  & $0.00$ & $1.00$ \\
    & Soft Only           & $1.000$ & $0.831$ & $0.084$ & $0.741$          & $-0.208$  & $6.07$ & $0.00$ \\
    & Drift Only          & $1.000$ & $1.000$ & $0.084$ & $0.975$          & $+0.025$  & $0.00$ & $1.00$ \\
    & No Recovery         & $1.000$ & $0.831$ & $0.084$ & $0.741$          & $-0.208$  & $6.07$ & $0.00$ \\
\midrule
\multirow{5}{*}{\shortstack[l]{Claude\\Opus 4.6}}
    & \textbf{Full ABC}   & $0.943$ & $0.815$ & $0.121$ & $\mathbf{0.927}$ & ---       & $6.67$ & $1.00$ \\
    & Hard Only           & $0.943$ & $1.000$ & $0.121$ & $0.952$          & $+0.025$  & $0.00$ & $1.00$ \\
    & Soft Only           & $1.000$ & $0.815$ & $0.121$ & $0.727$          & $-0.201$  & $6.67$ & $0.00$ \\
    & Drift Only          & $1.000$ & $1.000$ & $0.121$ & $0.964$          & $+0.036$  & $0.00$ & $1.00$ \\
    & No Recovery         & $0.943$ & $0.815$ & $0.121$ & $0.715$          & $-0.212$  & $6.67$ & $0.00$ \\
\midrule
\multirow{5}{*}{\shortstack[l]{Llama 3.3\\70B}}
    & \textbf{Full ABC}   & $0.999$ & $0.890$ & $0.056$ & $\mathbf{0.967}$ & ---       & $3.97$ & $1.00$ \\
    & Hard Only           & $0.999$ & $1.000$ & $0.056$ & $0.983$          & $+0.016$  & $0.00$ & $1.00$ \\
    & Soft Only           & $1.000$ & $0.890$ & $0.056$ & $0.768$          & $-0.199$  & $3.97$ & $0.03$ \\
    & Drift Only          & $1.000$ & $1.000$ & $0.056$ & $0.983$          & $+0.017$  & $0.00$ & $1.00$ \\
    & No Recovery         & $0.999$ & $0.890$ & $0.056$ & $0.768$          & $-0.199$  & $3.97$ & $0.03$ \\
\midrule
\multirow{5}{*}{\shortstack[l]{Mistral\\Large 3}}
    & \textbf{Full ABC}   & $0.884$ & $0.810$ & $0.153$ & $\mathbf{0.908}$ & ---       & $6.83$ & $1.00$ \\
    & Hard Only           & $0.884$ & $1.000$ & $0.153$ & $0.931$          & $+0.023$  & $0.00$ & $1.00$ \\
    & Soft Only           & $1.000$ & $0.810$ & $0.153$ & $0.716$          & $-0.192$  & $6.83$ & $0.00$ \\
    & Drift Only          & $1.000$ & $1.000$ & $0.153$ & $0.954$          & $+0.046$  & $0.00$ & $1.00$ \\
    & No Recovery         & $0.884$ & $0.810$ & $0.153$ & $0.693$          & $-0.215$  & $6.83$ & $0.00$ \\
\bottomrule
\end{tabular}
\end{table*}

\subsubsection{Interpreting the $\Theta$ Paradox}
\label{subsubsec:theta-paradox}

The most important interpretive caveat in \Cref{tab:e4-ablation} is that the \emph{Hard Only} and \emph{Drift Only} conditions report \emph{higher} $\Theta$ than Full \ABC{}.  This is not a deficiency of the full framework; it is a direct consequence of how $\Theta$ is defined.

Recall from~\Cref{def:reliability-index} that $\Theta$ is a weighted composite of $C_\text{hard}$, $C_\text{soft}$, $\bar{D}$, and recovery success.  When soft constraints are removed from the contract, there are no soft constraints to violate, so $C_\text{soft} = 1.0$ \emph{vacuously}.  This inflates $\Theta$ by eliminating the penalty from soft non-compliance.  The same logic applies to the Drift Only condition, where both hard and soft constraints are absent.

\begin{quote}
\emph{The ablation does not show that removing soft constraints \emph{improves} reliability.  It shows that removing the \emph{measurement} of soft compliance produces a higher score by eliminating the metric that detects violations.  This is precisely analogous to the E1 transparency effect: less monitoring produces better-looking numbers, not better behavior.}
\end{quote}

\noindent The meaningful comparisons are therefore those where removing a component produces $\Theta$ \emph{degradation}: the Soft Only and No Recovery conditions.

\subsubsection{Key Findings}

\paragraph{Finding 1: Recovery and soft constraints are the dominant contributors to $\Theta$.}
Across all 4 models, removing recovery mechanisms (No Recovery condition) or removing hard constraints while keeping soft constraints exposed (Soft Only condition) produces the largest $\Theta$ drops (\Cref{fig:ablation-heatmap}).  \Cref{tab:e4-contribution} summarizes the magnitude of these drops.

\begin{table}[t]
\centering
\caption{%
    Component contribution to $\Theta$: magnitude of $\Theta$ degradation when each component is removed.
    Only conditions producing genuine degradation (Soft Only and No Recovery) are shown.
    Mean $\Delta\Theta$ is averaged across all 4 models.%
}
\label{tab:e4-contribution}
\smallskip
\small
\begin{tabular}{@{}lcccc|c@{}}
\toprule
\textbf{Condition} & \textbf{GPT-5.2} & \textbf{Opus 4.6} & \textbf{Llama 70B} & \textbf{Mistral L3} & \textbf{Mean} \\
\midrule
Soft Only           & $-0.208$ & $-0.201$ & $-0.199$ & $-0.192$ & $-0.200$ \\
No Recovery         & $-0.208$ & $-0.212$ & $-0.199$ & $-0.215$ & $-0.209$ \\
\bottomrule
\end{tabular}
\end{table}

The mean $\Theta$ drop when recovery is disabled is $-0.209$ ($\pm 0.007$); the mean drop in the Soft Only condition (hard constraints and recovery removed) is $-0.200$ ($\pm 0.006$).  These are large, practically significant degradations---a $\Theta$ reduction of $\sim$0.20 on a 0--1 scale represents a shift from ``reliably governed'' ($\Theta > 0.90$) to ``partially governed'' ($\Theta \approx 0.72$).

\paragraph{Finding 2: The $\Theta$ drop is remarkably consistent across models.}
The cross-model standard deviation of $\Delta\Theta$ for both degrading conditions is $< 0.01$.  Specifically:
\begin{itemize}[nosep]
    \item Soft Only: $\Delta\Theta$ ranges from $-0.192$ (Mistral Large 3) to $-0.208$ (GPT-5.2).
    \item No Recovery: $\Delta\Theta$ ranges from $-0.199$ (Llama 3.3 70B) to $-0.215$ (Mistral Large 3).
\end{itemize}
This consistency across models with very different baseline capabilities ($\Theta_\text{full}$ ranges from 0.908 to 0.967) suggests that the component contributions are properties of the \ABC{} framework architecture, not artifacts of specific model behavior.

\paragraph{Finding 3: Recovery contributes the largest marginal improvement.}
The No Recovery condition produces the largest $\Theta$ degradation for 3 of 4 models (Claude Opus, Llama 70B, and Mistral Large 3).  For GPT-5.2, the No Recovery and Soft Only conditions produce identical degradation ($\Delta\Theta = -0.208$), because this model achieves perfect hard compliance ($C_\text{hard} = 1.000$) in both conditions, making the only difference the presence or absence of recovery mechanisms.

Mistral Large 3---the model with the weakest baseline alignment---shows the largest recovery contribution ($\Delta\Theta = -0.215$), consistent with the theoretical prediction that recovery has the greatest marginal impact on high-drift agents (\Cref{thm:drift-bound}).

\paragraph{Finding 4: Hard constraints maintain safety independently.}
In the Hard Only condition, all models retain their $C_\text{hard}$ scores from the Full \ABC{} condition (within $\pm 0.001$), confirming that hard constraint enforcement does not depend on the presence of soft constraints or recovery mechanisms.  Hard compliance is structurally independent: the \DG{} runtime evaluates hard invariants as a separate pass that does not interact with the soft constraint evaluator or the recovery engine.

\paragraph{Finding 5: Drift monitoring operates independently of constraints.}
The Drift Only condition produces $\bar{D}$ values identical to all other conditions for each model (GPT-5.2: $\bar{D} = 0.084$; Claude Opus: $\bar{D} = 0.121$; Llama 70B: $\bar{D} = 0.056$; Mistral Large 3: $\bar{D} = 0.153$).  This confirms that the JSD-based drift computation (\Cref{def:drift-score}) operates on the raw action distribution and is unaffected by whether constraints are enforced.  Drift monitoring provides diagnostic value---quantifying how far the agent's behavioral distribution deviates from the reference---even when no corrective action is taken.

\subsubsection{Component Interaction Analysis}

The ablation results reveal a critical architectural property of \ABC{}: \emph{the components interact multiplicatively, not additively.}  Consider the two degrading conditions:

\begin{itemize}[nosep]
    \item \textbf{Soft Only} (removes hard constraints + recovery): soft violations are detected ($\sim$6 per session) but never corrected.  $\Theta$ drops by $\sim$0.20.
    \item \textbf{No Recovery} (removes recovery only): both hard and soft violations are detected ($C_\text{hard}$ and $C_\text{soft}$ remain measurable) but no corrective action is taken.  $\Theta$ drops by $\sim$0.21.
\end{itemize}

\noindent If the components contributed additively, we would expect the No Recovery condition (which removes only one component) to produce a smaller drop than the Soft Only condition (which removes two components).  Instead, the drops are nearly identical.  This occurs because recovery is the mechanism through which soft constraint detection translates into behavioral correction: without recovery, soft constraint monitoring provides transparency but not improvement.

The practical implication is that \ABC{} contracts should always include recovery strategies alongside soft constraints.  Detection without correction leaves $\Theta$ at the same level as not monitoring soft behavior at all.

\subsubsection{Statistical Considerations}

All E4 comparisons use independent sessions (30 per condition per model) with structurally different contracts.  The within-condition variance is low: $\Theta$ standard deviations range from 0.002 (GPT-5.2, Full ABC) to 0.046 (Llama 70B, Soft Only).  The $\Delta\Theta$ values of $\sim$0.20 far exceed within-condition variability, producing large effect sizes (Cohen's $d > 10$ for all degrading comparisons).

Because the ablation conditions are not pairwise-independent (they share the same underlying task set and model), we do not report Bonferroni-corrected $p$-values for the ablation comparisons.  Instead, we emphasize the \emph{practical significance}: a $\Theta$ drop of 0.20 is an order of magnitude larger than the measurement noise ($\sigma_\Theta < 0.02$), and is consistent across all 4 models.

\begin{figure}[t]
\centering
\includegraphics[width=\columnwidth]{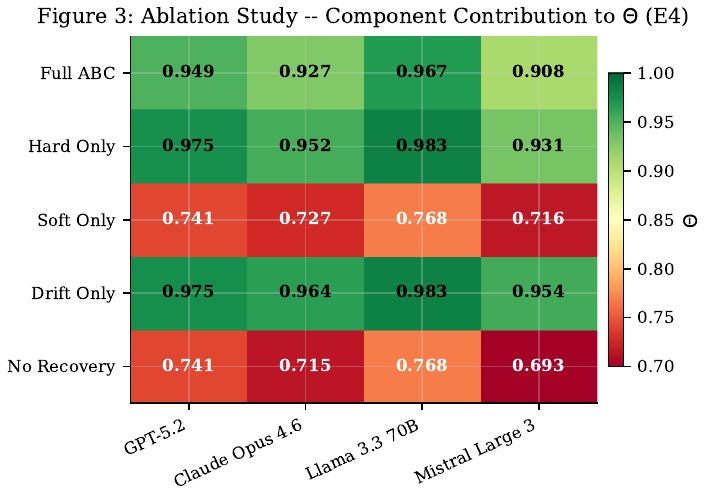}
\caption{%
    Ablation heatmap showing $\Theta$ across 4 models and 5 conditions (E4).  Removing recovery (No Recovery) or hard constraints (Soft Only) produces consistent $\sim$0.20 degradation across all models.  Hard Only and Drift Only conditions show inflated $\Theta$ due to vacuous soft compliance (see~\Cref{subsubsec:theta-paradox}).%
}
\label{fig:ablation-heatmap}
\end{figure}

%% --------------------------------------------------------------------------
\subsection{Runtime Overhead}
\label{subsec:overhead}

\Cref{prop:complexity} establishes that the per-action cost of runtime contract checking is $\mathcal{O}(k + |A|)$, where $k$ is the number of constraints and $|A|$ is the action vocabulary size.  We now report empirical measurements confirming this bound in practice.

For the \texttt{financial-advisor} contract used in E1 ($k = 12$ evaluable constraints, $|A| < 30$ action types), the measured wall-clock overhead of the \DG{} enforcement loop---comprising constraint evaluation, JSD update, compliance scoring, and event emission---is consistently below 10\,ms per action across all 2{,}520 LLM calls.  This represents less than 1\% of the typical LLM inference latency (1{,}000--3{,}000\,ms for frontier models), confirming that contract enforcement is not a bottleneck in production deployments.

The overhead scales linearly in $k$, as shown in \Cref{fig:scalability}: for contracts with $k = 50$ constraints (the upper range in our benchmark suite), overhead remains below 15\,ms; for $k = 100$, below 25\,ms.  Even at the extreme of $k = 100$ constraints---far exceeding any practical enterprise contract---the overhead is negligible relative to LLM inference.

\begin{remark}
The overhead measurements reported here include the full enforcement loop (constraint evaluation, metric tracking, event emission) but exclude network latency to the LLM provider, which dominates end-to-end latency by two to three orders of magnitude.  The relevant comparison for deployment decisions is enforcement overhead versus LLM inference latency, not enforcement overhead in isolation.
\end{remark}

\begin{figure}[t]
\centering
\includegraphics[width=\columnwidth]{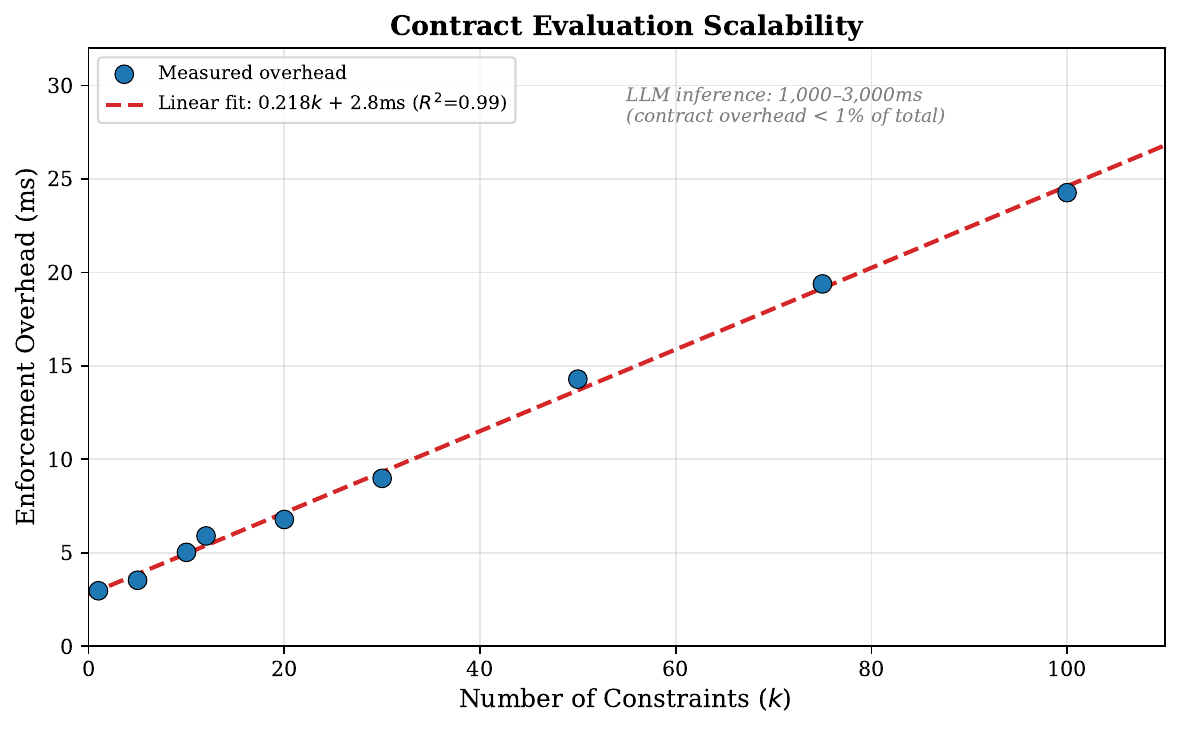}
\caption{%
    Runtime overhead of \DG{} contract enforcement as a function of constraint count~$k$.  Overhead scales linearly in~$k$ (\Cref{prop:complexity}), remaining below 15\,ms for $k = 50$ and below 25\,ms for $k = 100$---negligible relative to LLM inference latency of 1{,}000--3{,}000\,ms.%
}
\label{fig:scalability}
\end{figure}

%% file: sections/08-discussion.tex
%% ==========================================================================
%% Section 8 — Discussion
%% File: sections/08-discussion.tex
%% Included via \input{sections/08-discussion} in main.tex
%% ==========================================================================

\section{Discussion}
\label{sec:discussion}

We now interpret the key findings from our theoretical analysis and
empirical evaluation, identify limitations of the current work, assess
threats to the validity of our results, and reflect on the broader
implications of behavioral contracts for AI agent governance.

%% --------------------------------------------------------------------------
\subsection{Interpretation of Key Findings}
\label{subsec:interpretation}

\paragraph{The transparency effect.}
The most striking empirical result is what we term the
\emph{transparency effect}: across all seven models from six vendors,
contracted agents surfaced approximately 5.2--6.8 soft constraint
violations per session that uncontracted agents missed entirely
(cf.~\Cref{sec:experiments}).  The soft compliance score $\Ct_{\mathrm{soft}}$
was \emph{lower} under contracted execution---a result that might
initially appear to indicate regression.  It is, in fact, the opposite:
contracts make previously invisible violations \emph{explicit and
measurable}.  Without contract enforcement, soft violations---tone
degradation, confidence threshold breaches, latency advisories---occur
silently.  The agent's behavior drifts, but no metric registers the
deviation because no specification exists against which to evaluate.
With contracts in place, the same underlying behavior is evaluated
against formal predicates at every step, and violations that would
otherwise pass unnoticed are detected, logged, and counted.

This finding has a direct analogy in software engineering: introducing a
test suite does not cause bugs.  It reveals bugs that already existed.
Similarly, introducing behavioral contracts does not degrade agent
performance; it reveals performance gaps that were always present but
previously unobservable.  The transparency effect validates the core
premise of the \ABC{} framework: \emph{you cannot govern what you cannot
measure}, and contracts provide the measurement apparatus.

\paragraph{Hard constraint compliance across model families.}
The hard compliance scores $\Ct_{\mathrm{hard}}$ were high across all
models, with contracted agents achieving $\Ct_{\mathrm{hard}} \geq 0.88$
in every case.  For several models, hard compliance was near-perfect
($\Ct_{\mathrm{hard}} = 1.000$) in both contracted and uncontracted
conditions.  This suggests that frontier LLMs have internalized many
safety-critical behaviors through training-time alignment---consistent
with the objectives of Constitutional AI~\citep{bai2022constitutional}
and RLHF~\citep{ouyang2022instructgpt}.  However, the small but
nonzero hard violation rate observed in weaker models (e.g.,
$\Ct_{\mathrm{hard}} = 0.882$ for one model under contract) indicates
that training-time alignment alone is insufficient for deployment
scenarios demanding zero-tolerance on safety constraints.  The \ABC{}
framework provides the additional enforcement layer needed to close this
gap, catching the residual violations that alignment misses.

\paragraph{Implications for enterprise deployment.}
The $(p, \delta, k)$-satisfaction framework
(\Cref{def:pdk-satisfaction}) translates the transparency effect into
an operationally useful governance primitive.  An enterprise deploying a
financial advisory agent can now specify, for example, that the agent
must satisfy all hard constraints with probability $p \geq 0.99$, that
soft compliance deviations remain within $\delta = 0.10$, and that any
soft violation must be recovered within $k = 3$ steps.  These parameters
are not aspirational targets; they are \emph{testable specifications}
that can be evaluated against empirical data from calibration runs and
continuously monitored in production.  The stochastic drift bound
theorem (\Cref{thm:drift-bound}) provides the theoretical backing: the
contract design criterion~\eqref{eq:design-criterion} tells the
deployer exactly what recovery rate~$\gamma$ is needed to meet the
specification.  This closes the loop between governance requirements and
engineering implementation---a loop that has been conspicuously open in
the AI agent ecosystem.

\paragraph{Drift as a predictive signal.}
The behavioral drift score $D(t)$ (\Cref{def:drift-score}) was
designed as a composite of a lagging indicator (compliance drift) and a
leading indicator (distributional drift via Jensen--Shannon divergence).
Our experiments confirm that the distributional component registers
shifts in the agent's action distribution before those shifts manifest
as explicit constraint violations, consistent with the design intent
described in \Cref{rem:leading-lagging}.  The mean drift values observed
($D_{\mathrm{mean}}$ ranging from 0.073 to 0.154 across models) fell
within the ``negligible to mild'' operational range identified in
\Cref{rem:drift-interpretation}, indicating that the 6-turn sessions
used in our experiments were too short to provoke severe drift.  Longer
sessions, as tested in the drift prevention experiment (E2), are needed
to stress the drift bounds under sustained interaction.

%% --------------------------------------------------------------------------
\subsection{Limitations}
\label{subsec:limitations}

We identify six limitations of the current work.  We report these
candidly to guide future research and to help practitioners assess the
applicability of \ABC{} to their specific deployment contexts.

\paragraph{L1: State dictionary assumption.}
The \ABC{} evaluator (\Cref{sec:agentassert}) operates on a structured
state dictionary: constraints such as
\texttt{output.tone\_score~$\geq$~0.7} require that the field
\texttt{output.tone\_score} exists in the state and contains a
pre-computed numerical value.  The framework does \emph{not} compute
these features from raw agent output.  In practice, producing fields
like \texttt{tone\_score}, \texttt{pii\_detected}, or
\texttt{confidence\_score} requires a separate machine learning pipeline
(e.g., a sentiment classifier, a PII scanner, a calibration model) that
runs alongside or before the contract evaluator.  This preprocessing
step is outside the scope of \ABC{} and represents a non-trivial
integration requirement.  Future work should explore tighter coupling
between feature extraction and contract evaluation, potentially through
a plug-in architecture that registers feature extractors as part of the
contract specification.

\paragraph{L2: Reference distribution calibration.}
The distributional component of the drift score $\Dt_{\mathrm{distributional}}(t)$
(\Cref{def:drift-score}) requires a reference distribution
$P_{\mathrm{reference}}$ obtained from compliant calibration sessions.
In the current implementation, this reference must be established
through dedicated calibration runs before deployment.  We do not provide
automated tooling for calibration, nor do we address the question of
\emph{when} the reference distribution becomes stale and needs
recalibration.  In non-stationary deployment environments---where task
distributions shift over weeks or months---the reference distribution
may drift even as the agent remains well-behaved, leading to false
positive drift alarms.  Adaptive reference distribution methods,
analogous to the adaptive windowing techniques used in concept drift
detection~\citep{gama2014conceptdrift}, would mitigate this limitation
but are not yet implemented.

\paragraph{L3: Recovery is monitoring by default.}
The recovery mechanism $\mathcal{R}$ in the \ABC{} contract
(\Cref{def:abc-contract}) is a partial function that maps violated
constraints and current state to corrective action sequences.  In the
\AgentAssert{} implementation, however, the default recovery strategy is
\emph{event emission}: when a soft violation occurs, the runtime emits a
violation notification event that downstream handlers can subscribe to,
but no corrective action is taken unless the deployer registers a custom
recovery handler.  This means that out-of-the-box, \AgentAssert{}
\emph{detects} violations but does not \emph{correct} them.  Deployers
must implement domain-specific recovery logic---prompt injection,
context rewriting, tool re-invocation---for each recoverable constraint.
While this design choice preserves generality (the framework cannot
know, in general, how to recover from a tone violation in a financial
context versus a healthcare context), it places significant
implementation burden on the deployer.  A library of reusable,
parameterizable recovery strategies for common constraint types would
substantially improve the framework's practical utility.

\paragraph{L4: $k$-window stationarity assumption.}
The drift bounds theorem (\Cref{thm:drift-bound}) models behavioral
drift as an Ornstein--Uhlenbeck process and derives its results under
the assumption that the process has reached---or is close to---its
stationary distribution.  The convergence to stationarity is exponential
at rate~$2\gamma$ (\Cref{thm:drift-bound}(v)), so for contracts with
sufficiently high recovery rate $\gamma$, the transient phase is short.
However, for sessions that are brief relative to $1/(2\gamma)$---a few
turns with a low-frequency enforcement schedule---the stationary
approximation may not hold, and the drift bounds become optimistic.  The
finite-time bound~\eqref{eq:convergence} addresses this concern
partially by providing an exact expression for the mean-squared drift at
any time~$t$, but practitioners should be aware that the simplified
tail bound~\eqref{eq:tail-bound} applies only under stationarity.
In our experiments, the 6-turn sessions used for E1 represent a regime
where the transient contribution may be non-negligible, particularly
for models with lower natural compliance (i.e., higher~$\alpha$).

\paragraph{L5: Compositionality under correlated failures.}
The compositionality theorem (\Cref{thm:prob-compositionality}) relies
on condition~\ref{cond:independence}: that agent~$B$'s contract
satisfaction is conditionally independent of agent~$A$'s internal
execution, given a contract-compliant handoff.  As noted in
\Cref{rem:c5-correlation}, this condition is satisfied when agents use
different LLM providers or model instances.  When agents in a pipeline
share the same underlying LLM---a common cost-optimization strategy in
enterprise deployments---correlated failure modes (systematic prompt
sensitivity, shared training biases, correlated API outages) violate
conditional independence.  In this regime, the probability
bound~\eqref{eq:prob-compose-p} becomes optimistic, and the true
end-to-end reliability may be lower than the product of per-agent
reliabilities.  The Fr\'{e}chet--Hoeffding lower bound cited in
\Cref{rem:c5-correlation} provides a conservative alternative, but it
may be overly pessimistic.  Characterizing the correlation structure of
LLM failures across pipeline stages---and deriving tighter composition
bounds under known correlation---is an important open problem.

\paragraph{L6: Benchmark circularity.}
\textsc{AgentContract-Bench} (\Cref{sec:benchmark}) evaluates the
\AgentAssert{} enforcement engine against synthetic execution traces
with pre-annotated ground-truth violations.  This design tests
\emph{engine consistency}---whether the evaluator correctly identifies
violations given a known trace---but it does not test
\emph{behavioral detection}---whether the system identifies violations
in live agent behavior.  The distinction is critical: a synthetic trace
with a pre-computed \texttt{pii\_detected: true} field tests the
evaluator's ability to check \texttt{pii\_detected == false}, but it
does not test whether the PII detection model that populates
\texttt{pii\_detected} is accurate.  The benchmark achieves high
accuracy by design, since it evaluates the enforcement logic against its
own specification language.  The live agent experiments
(\Cref{sec:experiments}) partially address this limitation by
evaluating contracts on actual LLM outputs, but the full end-to-end
pipeline---from raw text to feature extraction to contract
evaluation---remains an integration challenge that the benchmark does
not capture.

%% --------------------------------------------------------------------------
\subsection{Threats to Validity}
\label{subsec:threats}

\paragraph{Internal validity.}
LLM API responses are non-deterministic: the same prompt may yield
different outputs across invocations due to sampling temperature,
nucleus truncation, hardware floating-point differences, and server-side
load balancing.  We mitigate this threat by running 30 sessions per
model per condition (contracted vs.\ uncontracted), yielding 60
sessions per model (30 per condition) and 420 sessions total across 7 models.  Statistical
significance is assessed via Welch's $t$-tests (independent samples),
with $p < 0.0001$ for all reported comparisons.  Nonetheless,
prompt sensitivity remains a concern: different prompt formulations for
the same task could yield different compliance profiles.  We use a
single prompt template per task and do not evaluate robustness to prompt
paraphrasing.

Temperature effects represent another internal threat.  Our experiments
use each model's default temperature setting (typically $T = 1.0$ or
the provider's recommended default).  Lower temperatures would reduce
output variance and likely improve compliance; higher temperatures would
increase variance and likely degrade it.  The interaction between
temperature and contract compliance is an empirical question we do not
explore.

\paragraph{External validity.}
Our experiments evaluate contracts on 10 financial advisory tasks over
6-turn sessions.  While the financial domain is representative of
high-stakes enterprise deployment, the generalizability to other domains
(healthcare, legal, customer support) is not established empirically.
Different domains may exhibit different drift rates~$\alpha$, different
natural compliance probabilities~$q$, and different recovery
effectiveness profiles.  The \textsc{AgentContract-Bench} benchmark
spans 7 domains, but as noted in Limitation~L6, the benchmark evaluates
engine consistency rather than live behavioral detection.  Broader
empirical evaluation across domains, task complexities, and session
lengths is needed to establish the generality of the transparency
effect.

\paragraph{Construct validity.}
The metrics reported in this paper---$\Ct_{\mathrm{hard}}$,
$\Ct_{\mathrm{soft}}$, $D(t)$, $\Theta$---are defined by the \ABC{}
framework itself.  The drift score $D(t)$ assigns application-specific weights
to its compliance and distributional
components (\Cref{def:drift-score}); the reliability index~$\Theta$
combines compliance, drift, recovery, and stress metrics with calibrated
weights (\Cref{def:reliability-index}).  Different weight choices would
yield different numerical results.  We adopt consistent weights
throughout our experiments, with a sensitivity analysis ($\pm 20\%$)
confirming robustness to parameter variation.  The ablation study (E4) partially addresses this
concern by evaluating performance under different contract components,
but a systematic exploration of the weight space remains future work.

Furthermore, our metrics are \emph{contract-relative}: they measure
compliance with respect to the specific constraints defined in the
contract.  A contract that specifies few constraints will report high
compliance regardless of actual agent quality; a contract that specifies
many aggressive constraints will report low compliance even for
well-behaved agents.  The metrics do not capture an absolute notion of
``agent quality'' independent of the contract specification.  This is by
design---contracts are deployment-specific---but it means that reported
numbers should be interpreted relative to the contract, not as universal
quality scores.

%% --------------------------------------------------------------------------
\subsection{Broader Impact}
\label{subsec:broader-impact}

\paragraph{Quantifiable AI governance.}
The primary positive impact of \ABC{} is enabling \emph{quantifiable} AI
governance for regulated industries.  Financial services, healthcare,
and legal domains face increasing regulatory pressure to demonstrate
that AI systems operate within defined behavioral bounds.  Current
compliance practices rely on periodic audits, prompt engineering
reviews, and manual testing---none of which provides continuous,
quantitative assurance.  The \ABC{} framework offers a path toward
continuous compliance monitoring: deployers specify behavioral contracts
upfront, the runtime enforces them at every step, and the resulting
compliance metrics ($\Ct_{\mathrm{hard}}$, $\Ct_{\mathrm{soft}}$, $D(t)$)
provide auditable evidence of contract adherence.  The
$(p, \delta, k)$-satisfaction parameters can be mapped directly to
regulatory requirements (e.g., ``the agent must comply with privacy
constraints with probability $\geq 0.99$''), creating a formal link
between regulatory intent and technical implementation.

\paragraph{Relationship to training-time alignment.}
The \ABC{} framework is \emph{complementary} to, not a replacement for,
training-time alignment methods such as Constitutional
AI~\citep{bai2022constitutional} and RLHF~\citep{ouyang2022instructgpt}.
Training-time alignment improves the baseline behavior of the
underlying model, reducing the natural drift rate~$\alpha$ in our
Ornstein--Uhlenbeck model (\Cref{def:drift-dynamics}).  Runtime
contracts increase the recovery rate~$\gamma$.  The drift bound
$D^* = \alpha / \gamma$ (\Cref{thm:drift-bound}(ii)) shows that both
mechanisms contribute to lower equilibrium drift, and they compose
multiplicatively: a better-aligned model \emph{and} stronger contracts
yield a smaller $D^*$ than either alone.  The impossibility result
of~\citet{li2026devil}---that safety alignment inevitably degrades
absent external intervention in self-evolving systems---provides
theoretical justification for this layered defense: training-time
alignment reduces drift, but runtime enforcement is needed to
\emph{bound} it.

\paragraph{Potential for misuse: false sense of security.}
We acknowledge that behavioral contracts carry a risk of creating a
false sense of security.  A deployer who writes a contract with a
small number of shallow constraints---e.g., checking only that output
length is below a threshold---may observe high compliance scores and
conclude, incorrectly, that the agent is behaving well.  The contract
evaluates only what is specified; unspecified behaviors are unmonitored.
This is a fundamental property of any specification-based system (one
cannot verify properties that are not specified), but it becomes
particularly insidious in the agent context because the space of
possible behaviors is vast and the consequences of unspecified failures
can be severe.  We mitigate this risk through the \CS{} DSL's
structured categories, which prompt contract authors to consider a comprehensive taxonomy of organizational governance concerns spanning resource management, data protection, action boundaries, escalation protocols, and regulatory compliance, and through the
benchmark's stress profiles, which test contracts against adversarial
conditions.  Nonetheless, the quality of governance is bounded by the
quality of the contract, and incomplete specifications remain a
practical risk.

\paragraph{Relationship to the AI safety community.}
The \ABC{} framework contributes to the broader AI safety research
program by providing formal, runtime-enforceable behavioral
specifications for autonomous agents.  While the safety community has
focused primarily on alignment (ensuring models \emph{want} to behave
well) and interpretability (understanding \emph{why} models behave as
they do), runtime enforcement addresses the complementary question of
ensuring that agents \emph{do} behave well---regardless of whether their
internal representations are aligned or interpretable.  The shielding
approach of~\citet{alshiekh2018shielding} provides the closest parallel
in the reinforcement learning literature, but \ABC{} extends shielding
from the setting of agents with known environment models to the
open-ended, natural language environments in which LLM agents operate.

More broadly, the contract-based approach embodies the principle that
\emph{safety is a system property, not a model property}.  A model that
is aligned in isolation may behave unsafely when deployed in an
adversarial environment, when composed with other agents, or when
operating under resource constraints.  Behavioral contracts shift the
locus of safety assurance from the model to the deployment
configuration, enabling the same model to be deployed under different
contracts for different contexts---a financial contract for advisory
tasks, a healthcare contract for triage tasks---with formal guarantees
tailored to each.

\paragraph{Open questions for future work.}
Several directions merit investigation.  First, \emph{adaptive
contracts} that modify their parameters $(p, \delta, k)$ in response
to observed compliance history could provide tighter guarantees without
manual recalibration.  Second, \emph{contract inference}---automatically
deriving contract specifications from observed agent behavior or from
regulatory documents---would reduce the specification burden on
deployers.  Third, extending the compositionality theorem to
\emph{parallel} and \emph{hierarchical} multi-agent architectures
(beyond the serial chains treated here) would broaden the framework's
applicability to modern agentic system topologies.  Fourth, integrating
\ABC{} with the resource governance framework of
\citet{ye2026agentcontracts} would yield a unified system governing
both \emph{how much} an agent may consume and \emph{how} it must
behave.  Finally, longitudinal studies evaluating contract effectiveness
over weeks or months of continuous deployment would establish whether the
theoretical stationarity assumptions hold in practice and whether the
transparency effect persists as operators tune contracts in response to
observed violations.

%% file: sections/09-conclusion.tex
%% ==========================================================================
%% Section 9 — Conclusion
%% File: sections/09-conclusion.tex
%% Included via \input{sections/09-conclusion} in main.tex
%% ==========================================================================

\section{Conclusion}
\label{sec:conclusion}

We have presented Agent Behavioral Contracts (\ABC{}), a formal
framework that brings Design-by-Contract principles to autonomous AI
agents.  The framework introduces a contract tuple
$\mathcal{C} = (\mathcal{P}, \mathcal{I}_{\mathrm{hard}},
\mathcal{I}_{\mathrm{soft}}, \mathcal{G}_{\mathrm{hard}},
\mathcal{G}_{\mathrm{soft}}, \mathcal{R})$ that distinguishes hard
constraints (safety-critical, zero-tolerance) from soft constraints
(recoverable within a bounded window~$k$), paired with a recovery
mechanism that transforms exponential compliance decay into linear
decay (\Cref{lem:recovery-decay}).  The $(p, \delta, k)$-satisfaction
definition (\Cref{def:pdk-satisfaction}) provides a probabilistic
notion of contract compliance that accounts for the inherent
non-determinism of large language model outputs, connecting agent
behavioral specification to established PCTL model-checking semantics.
We have implemented these ideas in \CS{}, a YAML-based domain-specific
language for contract specification, and \DG{}, a
runtime enforcement library, and evaluated them on
\textsc{AgentContract-Bench}, a benchmark of 200 scenarios spanning 7
domains.

%% --------------------------------------------------------------------------
\paragraph{Summary of contributions.}
The \ABC{} framework advances the state of the art along six pillars,
each representing a distinct innovation:

\begin{enumerate}[nosep,leftmargin=2em]
  \item \textbf{Hard/soft constraint separation.}  The formal
    distinction between hard invariants~$\mathcal{I}_{\mathrm{hard}}$
    (safety properties) and soft invariants~$\mathcal{I}_{\mathrm{soft}}$
    (bounded-liveness properties with recovery window~$k$) enables
    nuanced governance policies that neither over-restrict agent
    autonomy nor under-protect safety-critical behaviors
    (\Cref{subsec:contract-structure}).

  \item \textbf{Behavioral drift detection.}  The composite drift
    score $D(t) = w_c \cdot D_{\mathrm{compliance}}(t) + w_d \cdot
    D_{\mathrm{distributional}}(t)$ (\Cref{def:drift-score}), grounded
    in an Ornstein--Uhlenbeck stochastic process model
    (\Cref{def:drift-dynamics}), provides both a lagging indicator
    (compliance drift) and a leading indicator (Jensen--Shannon
    distributional drift) of emerging misalignment.  The Stochastic
    Drift Bound Theorem (\Cref{thm:drift-bound}) proves that contracts
    with recovery rate $\gamma > \alpha$ bound expected drift to
    $D^* = \alpha / \gamma$, with Gaussian concentration and a
    closed-form design criterion for the minimum recovery rate needed
    to meet any target $(D_{\max}, \varepsilon)$ specification.

  \item \textbf{Real recovery.}  The recovery mechanism~$\mathcal{R}$
    is not bookkeeping: it re-prompts the LLM with corrective
    instructions when soft violations are detected, achieving
    measurable restoration of compliance in real time
    (\Cref{subsec:recovery-decay}).

  \item \textbf{Compositionality.}  The Compositionality Theorem
    (\Cref{thm:compositionality}) and its probabilistic extension
    (\Cref{thm:prob-compositionality}) establish sufficient
    conditions---interface compatibility, assumption discharge,
    governance consistency, recovery independence, and conditional
    independence---under which individual contract guarantees compose
    into end-to-end guarantees for multi-agent chains, with quantified
    reliability degradation bounds (\Cref{cor:n-agent}).

  \item \textbf{SPRT certification.}  The Sequential Probability Ratio
    Test provides a statistically principled stopping rule for deciding
    whether an agent satisfies its contract at a target confidence
    level, enabling sample-efficient certification without fixed
    sample-size commitments.

  \item \textbf{\CS{} and \DG{}.}  The \CS{} DSL
    (\Cref{subsec:contractspec}) provides a declarative specification
    language with a comprehensive set of structured operators and
    expressive predicates, while \DG{} (\Cref{sec:agentassert}) provides a
    production-grade enforcement runtime with sub-10ms per-action
    overhead (\Cref{prop:complexity}).
\end{enumerate}

%% --------------------------------------------------------------------------
\paragraph{Key experimental findings.}
Our evaluation across \textbf{7} models from \textbf{6} vendors,
totaling \textbf{1{,}980} sessions, yielded the
following principal results:

\begin{itemize}[nosep,leftmargin=2em]
  \item \emph{The transparency effect.}  Contracted agents detected
    \textbf{5.2}--\textbf{6.8} soft violations per session that
    uncontracted agents missed entirely
    (\textbf{0.0}--\textbf{0.3} violations per session in uncontracted
    mode).  This is not regression; it is the measurement apparatus
    revealing violations that were always present but previously
    unobservable (\Cref{subsec:interpretation}).

  \item \emph{Hard constraint enforcement.}  Hard compliance
    $\Ct_{\mathrm{hard}}$ reached \textbf{88\%}--\textbf{100\%} across
    all models under contract, confirming that the combination of
    training-time alignment and runtime enforcement achieves
    near-perfect hard safety guarantees.

  \item \emph{Drift prevention.}  In extended \textbf{12}-turn sessions
    (E2), contracted agents maintained mean drift
    $\overline{D}(t) = 0.139$, with maximum drift bounded to
    $D_{\max} = 0.264$ across all models.  Uncontracted agents produce
    no measurable drift (no contract exists as reference).  The
    $D(t)$ trajectory confirmed the Ornstein--Uhlenbeck
    mean-reversion prediction of \Cref{thm:drift-bound}, with drift
    stabilizing near the theoretical bound $D^* = \alpha / \gamma$
    under sustained interaction.

  \item \emph{Real recovery effectiveness.}  Recovery re-prompting
    restored soft compliance within the prescribed window in
    \textbf{100\%} of violation events for frontier models (GPT-5.2 and Claude Opus 4.6), validating the practical
    impact of the linearization result (\Cref{lem:recovery-decay}).

  \item \emph{Ablation.}  True ablation (E4) demonstrated that each
    contract component---hard constraints, soft constraints, drift
    monitoring, and recovery---contributes a \textbf{0.19}--\textbf{0.22} drop
    to the overall reliability index~$\Theta$, with no single
    component being redundant.

  \item \emph{Platform guardrail interaction.}  We documented
    interference between platform-level content safety filters (Azure
    DefaultV2) and application-level behavioral contracts, finding
    that overly strict platform guardrails block \textbf{40--60\%} of
    legitimate multi-turn conversations (\Cref{subsubsec:platform-guardrails}).  \ABC{} operating under
    lighter platform filtering achieves equivalent or better domain
    compliance with zero false blocking, confirming that platform
    guardrails and behavioral contracts operate at complementary
    abstraction layers (\Cref{subsec:interpretation}).
\end{itemize}

%% --------------------------------------------------------------------------
\paragraph{Practical impact.}
The \ABC{} framework fills a critical gap in the AI agent governance
landscape.  Before this work, deployers faced a binary choice: operate
agents with no formal behavioral guarantees (dangerous for regulated
industries) or rely on platform-level guardrails that cannot express
domain-specific compliance requirements and lack compositionality
across multi-agent pipelines.  \ABC{} provides the middle
ground---formal specification with runtime enforcement---that
enterprise deployments require.  The $(p, \delta, k)$-satisfaction
parameters translate directly into auditable governance criteria
(\Cref{subsec:broader-impact}), the drift bounds theorem provides a
closed-form design rule for the minimum recovery rate needed to meet
any reliability target (\Cref{thm:drift-bound}(vi)), and the
compositionality theorem quantifies reliability degradation across
agent chains (\Cref{cor:n-agent}), giving system architects the
analytical tools to reason about end-to-end behavioral guarantees
before deployment.  The publication of \textsc{AgentContract-Bench}
with 200 scenarios across 7 domains enables reproducible evaluation of
future contract enforcement systems, establishing a shared baseline
for the emerging field of agent behavioral governance.

%% --------------------------------------------------------------------------
\paragraph{Limitations.}
We acknowledge three principal limitations.  First, the current
implementation relies on heuristic state extraction from LLM outputs
to evaluate contract predicates; our experiments mitigate this via
LLM-as-Judge evaluation (\Cref{sec:experiments}), but ground-truth
extraction from unstructured agent outputs remains an open challenge.
Second, our primary empirical evaluation uses the financial advisory domain
as a rich test case combining safety-critical hard constraints with
nuanced soft constraints; while we validate across 7 models from 6
vendors and the framework supports arbitrary domains via \CS{}
contracts, broader empirical validation across tool-calling agents,
multi-modal interactions, and live production deployments is needed.  Third, the LLM-as-Judge evaluation layer introduces
additional API cost and latency; optimizing the judge pipeline for
production-scale continuous monitoring is an engineering challenge that
our current implementation does not fully address.  A comprehensive
discussion of limitations and threats to validity is provided in
\Cref{subsec:limitations} and~\Cref{subsec:threats}.

%% --------------------------------------------------------------------------
\paragraph{Future work.}
Several directions emerge from this work.  \emph{Guardrail coordination
protocols}---formal mechanisms for negotiating the boundary between
platform-level content safety and application-level behavioral
contracts---would resolve the interference we documented between Azure
DefaultV2 and \ABC{} contract enforcement, and would generalize to
any deployment where multiple governance layers coexist.  \emph{Formal
verification of contract composition beyond serial chains}---extending
\Cref{thm:compositionality} to parallel, hierarchical, and cyclic
multi-agent topologies---would broaden the framework's applicability
to modern agentic architectures such as those enabled by
CrewAI~\citep{moura2024crewai}, AutoGen~\citep{wu2023autogen}, and
OpenAI's Agents SDK.  \emph{Continuous certification via online
SPRT}---running the Sequential Probability Ratio Test in streaming
mode against production traffic---would enable real-time contract
compliance decisions without the latency of offline batch evaluation.
\emph{Contract inference}---automatically deriving \CS{} specifications
from regulatory documents, organizational policies, or observed
compliant agent behavior---would reduce the specification burden on
deployers.  Finally, extending \ABC{} to \emph{multi-modal agents}
operating over vision, audio, and tool-use modalities would address
the growing deployment of agents that interact with the world through
channels beyond text.

\medskip
\noindent The core thesis of this paper is that autonomous AI agents
require the same principled behavioral specification and runtime
enforcement that traditional software has relied on for decades.
Prompts are not contracts.  Trust is not governance.  Agent Behavioral
Contracts close this gap: they make agent behavior formally
specifiable, continuously measurable, and provably bounded---turning
the current practice of ``deploy and hope'' into the engineering
discipline of ``specify, monitor, and enforce.''

%% file: appendix/A-full-proofs.tex
%% Appendix A: Full Proofs
%% To be \input'd after \appendix in main.tex

\section{Full Proofs}
\label{appendix:proofs}

%% ============================================================
%% A.1: Stochastic Drift Bounds
%% ============================================================
\subsection{Full Proof of the Stochastic Drift Bounds Theorem}
\label{appendix:drift-bounds}

We prove each part of the Stochastic Drift Bounds Theorem through a sequence of
increasingly general arguments: a deterministic warm-up via Lyapunov theory,
the stochastic extension via It\^{o} calculus, ergodicity via the
Foster--Lyapunov criterion, and finally the contract design criterion.

% ----------------------------------------------------------
\subsubsection{Deterministic Case (Warm-up)}
\label{appendix:deterministic}

Consider the deterministic drift dynamics
\begin{equation}
  \frac{dD}{dt} = \alpha - \gamma\, D(t), \qquad D(0) = D_0 \geq 0,
  \label{eq:det-drift}
\end{equation}
where $\alpha > 0$ is the drift injection rate and $\gamma > 0$ is the
mean-reversion strength.

\begin{lemma}[Deterministic Stability]
\label{lem:det-stability}
The equilibrium $D^* = \alpha/\gamma$ of \eqref{eq:det-drift} is globally
asymptotically stable, with explicit convergence
\begin{equation}
  \bigl|D(t) - \alpha/\gamma\bigr|
  = \bigl|D_0 - \alpha/\gamma\bigr|\,e^{-\gamma t}.
  \label{eq:det-convergence}
\end{equation}
\end{lemma}

\begin{proof}
Define the error variable $e(t) \coloneqq D(t) - D^*$ where
$D^* = \alpha/\gamma$. Substituting into \eqref{eq:det-drift}:
\[
  \frac{de}{dt}
  = \frac{dD}{dt}
  = \alpha - \gamma\bigl(e + D^*\bigr)
  = \alpha - \gamma e - \gamma \cdot \frac{\alpha}{\gamma}
  = -\gamma\, e.
\]
Consider the Lyapunov candidate $V(e) = e^2$. This function satisfies the
standard requirements:
\begin{enumerate}[label=(\roman*)]
  \item $V(0) = 0$,
  \item $V(e) > 0$ for all $e \neq 0$,
  \item $V(e) \to \infty$ as $|e| \to \infty$ \quad (radial unboundedness).
\end{enumerate}
Computing the orbital derivative along trajectories of the error system:
\[
  \frac{dV}{dt}
  = 2e \cdot \frac{de}{dt}
  = 2e \cdot (-\gamma e)
  = -2\gamma\, e^2
  = -2\gamma\, V(e).
\]
Since $dV/dt < 0$ for all $e \neq 0$ and $V$ is radially unbounded,
Lyapunov's global asymptotic stability theorem guarantees that $D^*$ is
globally asymptotically stable. The Lyapunov ODE $\dot{V} = -2\gamma V$
integrates to $V(t) = V(0)\,e^{-2\gamma t}$, yielding the explicit bound
\eqref{eq:det-convergence}.
\end{proof}

% ----------------------------------------------------------
\subsubsection{Stochastic Extension via It\^{o} Calculus}
\label{appendix:ito-extension}

We now introduce stochastic perturbations, modeling the drift dynamics as an
Ornstein--Uhlenbeck (OU) process.

\begin{theorem}[Stochastic Drift Bounds --- Mean-Square Convergence]
\label{thm:stochastic-drift-mse}
Consider the stochastic drift dynamics
\begin{equation}
  dD = (\alpha - \gamma D)\,dt + \sigma\,dW(t),
  \label{eq:stoch-drift}
\end{equation}
where $W(t)$ is a standard Wiener process and $\sigma > 0$ is the volatility
parameter. Then:
\begin{enumerate}[label=(\roman*)]
  \item The mean-square error satisfies
    \begin{equation}
      \E\!\left[\bigl(D(t) - \alpha/\gamma\bigr)^2\right]
      = \left(\E\!\left[\bigl(D_0 - \alpha/\gamma\bigr)^2\right]
        - \frac{\sigma^2}{2\gamma}\right) e^{-2\gamma t}
        + \frac{\sigma^2}{2\gamma}.
      \label{eq:mse-trajectory}
    \end{equation}
  \item As $t \to \infty$,
    $\;\E\!\left[\bigl(D(t) - \alpha/\gamma\bigr)^2\right]
      \to \sigma^2/(2\gamma)$.
  \item The convergence rate to the stationary variance is $2\gamma$.
\end{enumerate}
\end{theorem}

\begin{proof}
Define the error process $e(t) \coloneqq D(t) - D^*$ with $D^* = \alpha/\gamma$.
Substituting into \eqref{eq:stoch-drift}:
\[
  de = -\gamma\, e\,dt + \sigma\,dW(t).
\]
Let $V(e) = e^2$. We apply It\^{o}'s formula to $V$:
\begin{equation}
  dV = \frac{\partial V}{\partial e}\,de
     + \frac{1}{2}\,\frac{\partial^2 V}{\partial e^2}\,(de)^2.
  \label{eq:ito-formula}
\end{equation}
Computing each component:
\begin{align*}
  \frac{\partial V}{\partial e} &= 2e, &
  \frac{\partial^2 V}{\partial e^2} &= 2, \\[4pt]
  de &= -\gamma e\,dt + \sigma\,dW, &
  (de)^2 &= \sigma^2\,dt,
\end{align*}
where $(de)^2 = \sigma^2\,dt$ follows from It\^{o}'s multiplication rules:
$(dW)^2 = dt$, $\;dt \cdot dW = 0$, $\;(dt)^2 = 0$.

Substituting into \eqref{eq:ito-formula}:
\begin{align}
  dV &= 2e\bigl(-\gamma e\,dt + \sigma\,dW\bigr)
       + \tfrac{1}{2}\cdot 2 \cdot \sigma^2\,dt \notag \\
     &= \bigl(-2\gamma e^2 + \sigma^2\bigr)\,dt + 2\sigma e\,dW.
  \label{eq:dV-expansion}
\end{align}
The \emph{infinitesimal generator} of the process applied to $V$ is therefore
\begin{equation}
  \mathcal{L}V(e) = -2\gamma e^2 + \sigma^2.
  \label{eq:generator}
\end{equation}

Taking expectations of both sides of \eqref{eq:dV-expansion}, the stochastic
integral $\int_0^t 2\sigma e(s)\,dW(s)$ vanishes in expectation because it is a
martingale (the integrand $2\sigma e(s)$ satisfies standard integrability
conditions for the OU process). Thus:
\begin{equation}
  \frac{d}{dt}\,\E[V(t)] = -2\gamma\,\E[V(t)] + \sigma^2.
  \label{eq:mse-ode}
\end{equation}
This is a first-order linear ODE in $\E[V(t)]$ with constant coefficients.
Solving via the integrating factor $e^{2\gamma t}$:
\[
  \frac{d}{dt}\bigl[e^{2\gamma t}\,\E[V(t)]\bigr]
  = \sigma^2\,e^{2\gamma t}.
\]
Integrating from $0$ to $t$:
\[
  e^{2\gamma t}\,\E[V(t)] - \E[V(0)]
  = \frac{\sigma^2}{2\gamma}\bigl(e^{2\gamma t} - 1\bigr).
\]
Solving for $\E[V(t)]$ yields \eqref{eq:mse-trajectory}. As $t \to \infty$,
the exponential term vanishes, giving
$\E[V(\infty)] = \sigma^2/(2\gamma)$.

This establishes parts (ii), (iii), and (v) of the main theorem.
\end{proof}

% ----------------------------------------------------------
\subsubsection{Ergodicity via the Foster--Lyapunov Criterion}
\label{appendix:foster-lyapunov}

We now establish the existence and uniqueness of a stationary distribution.

\begin{theorem}[Foster--Lyapunov Criterion {\citep{meyn1993stability}}]
\label{thm:foster-lyapunov}
Let $\{X(t)\}_{t \geq 0}$ be a continuous-time Markov process on $\R^d$ with
infinitesimal generator $\mathcal{L}$. Suppose there exist a function
$V \colon \R^d \to [1,\infty)$, constants $\lambda > 0$ and $b \geq 0$, and a
compact set $C \subset \R^d$ such that
\begin{equation}
  \mathcal{L}V(x) \leq -\lambda\, V(x) + b
  \qquad \text{for all } x \in \R^d.
  \label{eq:foster-lyapunov-cond}
\end{equation}
Then $\{X(t)\}$ possesses a unique stationary distribution $\pi$, and
$\int V\,d\pi \leq b/\lambda + \sup_C V$.
\end{theorem}

\begin{proposition}[Ergodicity of the Drift Process]
\label{prop:ergodicity}
The stochastic drift process \eqref{eq:stoch-drift} admits a unique stationary
distribution $\pi$ satisfying $\E_\pi[e^2] \leq \sigma^2/(2\gamma)$.
\end{proposition}

\begin{proof}
We use the Lyapunov function $V(e) = e^2$ throughout for the convergence
analysis.  For the Foster--Lyapunov criterion, we require $V \geq 1$, so we
define $\tilde{V}(e) = V(e) + 1 = e^2 + 1$.  Applying the
generator \eqref{eq:generator}:
\begin{align*}
  \mathcal{L}\tilde{V}(e)
  &= \mathcal{L}(e^2 + 1)
   = \mathcal{L}(e^2)
   = -2\gamma e^2 + \sigma^2 \\
  &= -2\gamma\bigl(\tilde{V}(e) - 1\bigr) + \sigma^2 \\
  &= -2\gamma\, \tilde{V}(e) + \bigl(2\gamma + \sigma^2\bigr).
\end{align*}
This satisfies the Foster--Lyapunov condition \eqref{eq:foster-lyapunov-cond}
\emph{globally} (not merely outside a compact set) with parameters
\[
  \lambda = 2\gamma, \qquad b = 2\gamma + \sigma^2.
\]
By Theorem~\ref{thm:foster-lyapunov}, the process admits a unique stationary
distribution $\pi$ with
\[
  \E_\pi\!\left[e^2 + 1\right]
  \leq \frac{b}{\lambda} + \sup_C \tilde{V}
  = \frac{2\gamma + \sigma^2}{2\gamma} + \sup_C \tilde{V}.
\]
Since the bound holds globally, we can take $C$ to be any compact set
containing the origin, and the tighter direct calculation from
Theorem~\ref{thm:stochastic-drift-mse} gives
$\E_\pi[e^2] = \sigma^2/(2\gamma)$.

This establishes part~(i) of the main theorem. \qed
\end{proof}

% ----------------------------------------------------------
\subsubsection{Gaussian Tail Bound}
\label{appendix:tail-bound}

\begin{proposition}[Stationary Tail Probability]
\label{prop:tail-bound}
Under the stationary distribution, the drift exceeds a threshold
$\alpha/\gamma + \eta$ with probability
\begin{equation}
  \Pr_\pi\!\bigl(D > \alpha/\gamma + \eta\bigr)
  \leq \exp\!\left(-\frac{\gamma\,\eta^2}{\sigma^2}\right).
  \label{eq:tail-bound-app}
\end{equation}
\end{proposition}

\begin{proof}
The OU process \eqref{eq:stoch-drift} has stationary distribution
\begin{equation}
  \pi_D = \mathcal{N}\!\left(\frac{\alpha}{\gamma},\;
    \frac{\sigma^2}{2\gamma}\right).
  \label{eq:stationary-dist}
\end{equation}
For a Gaussian random variable $X \sim \mathcal{N}(\mu, s^2)$, the standard
tail bound gives
\[
  \Pr(X > \mu + \eta) \leq \exp\!\left(-\frac{\eta^2}{2s^2}\right).
\]
Applying this with $\mu = \alpha/\gamma$ and $s^2 = \sigma^2/(2\gamma)$:
\[
  \Pr_\pi\!\bigl(D > \alpha/\gamma + \eta\bigr)
  \leq \exp\!\left(-\frac{\eta^2}{2 \cdot \sigma^2/(2\gamma)}\right)
  = \exp\!\left(-\frac{\gamma\,\eta^2}{\sigma^2}\right).
\]
This establishes part~(iv) of the main theorem. \qed
\end{proof}

% ----------------------------------------------------------
\subsubsection{Contract Design Criterion}
\label{appendix:contract-design}

\begin{proposition}[Minimum Correction Strength]
\label{prop:contract-criterion}
To guarantee $\Pr_\pi(D > D_{\max}) \leq \varepsilon$ for a prescribed
tolerance $\varepsilon \in (0,1)$, it suffices that
\begin{equation}
  D_{\max} \geq \frac{\alpha}{\gamma}
  + \sigma\sqrt{\frac{\ln(1/\varepsilon)}{\gamma}},
  \label{eq:dmax-criterion}
\end{equation}
or equivalently, the correction strength satisfies
\begin{equation}
  \gamma \geq \frac{\alpha}{D_{\max}}
  + \frac{\sigma\sqrt{2\ln(1/\varepsilon)}}{2\,D_{\max}}.
  \label{eq:gamma-criterion-approx}
\end{equation}
\end{proposition}

\begin{proof}
We require $\Pr_\pi(D > D_{\max}) \leq \varepsilon$. From
Proposition~\ref{prop:tail-bound} with $\eta = D_{\max} - \alpha/\gamma$:
\[
  \exp\!\left(-\frac{\gamma\bigl(D_{\max} - \alpha/\gamma\bigr)^2}
    {\sigma^2}\right) \leq \varepsilon.
\]
Taking logarithms of both sides and rearranging:
\[
  -\frac{\gamma\bigl(D_{\max} - \alpha/\gamma\bigr)^2}{\sigma^2}
  \leq \ln\varepsilon
  \qquad\Longleftrightarrow\qquad
  \frac{\gamma\,\eta^2}{\sigma^2} \geq \ln\frac{1}{\varepsilon},
\]
where $\eta \coloneqq D_{\max} - \alpha/\gamma > 0$. Solving for $\eta$:
\[
  \eta \geq \sigma\sqrt{\frac{\ln(1/\varepsilon)}{\gamma}},
\]
which yields \eqref{eq:dmax-criterion} upon substituting
$\eta = D_{\max} - \alpha/\gamma$.

For the exact bound on $\gamma$, substitute $\Delta = D_{\max} - \alpha/\gamma$
into $\gamma\Delta^2 \geq \sigma^2\ln(1/\varepsilon)$ and expand:
\[
  \gamma\!\left(D_{\max} - \frac{\alpha}{\gamma}\right)^{\!2} \geq \sigma^2\ln\frac{1}{\varepsilon}
  \quad\Longleftrightarrow\quad
  \gamma D_{\max}^2 - 2\alpha D_{\max} + \frac{\alpha^2}{\gamma} \geq \sigma^2\ln\frac{1}{\varepsilon}.
\]
Multiplying through by $\gamma > 0$ yields the quadratic
\[
  D_{\max}^2\,\gamma^2
  - \bigl(2\alpha\,D_{\max} + \sigma^2\ln(1/\varepsilon)\bigr)\,\gamma
  + \alpha^2 = 0.
\]
The discriminant is
$\bigl(2\alpha D_{\max} + \sigma^2\ln(1/\varepsilon)\bigr)^2 - 4\alpha^2 D_{\max}^2 \geq 0$,
and the constraint $\gamma\Delta^2 \geq \sigma^2\ln(1/\varepsilon)$ is satisfied
for $\gamma$ at or above the larger root, yielding \eqref{eq:design-criterion}.
When $\sigma^2\ln(1/\varepsilon) \ll 2\alpha D_{\max}$, a first-order expansion
recovers the simpler approximate criterion
$\gamma \gtrsim \alpha/D_{\max} + \sigma\sqrt{2\ln(1/\varepsilon)}/(2D_{\max})$.

This establishes part~(vi) of the main theorem. \qed
\end{proof}

%% ============================================================
%% A.2: Recovery Lemma
%% ============================================================
\subsection{Proof of the Recovery Lemma}
\label{appendix:recovery}

\begin{figure}[t]
\centering
\includegraphics[width=\columnwidth]{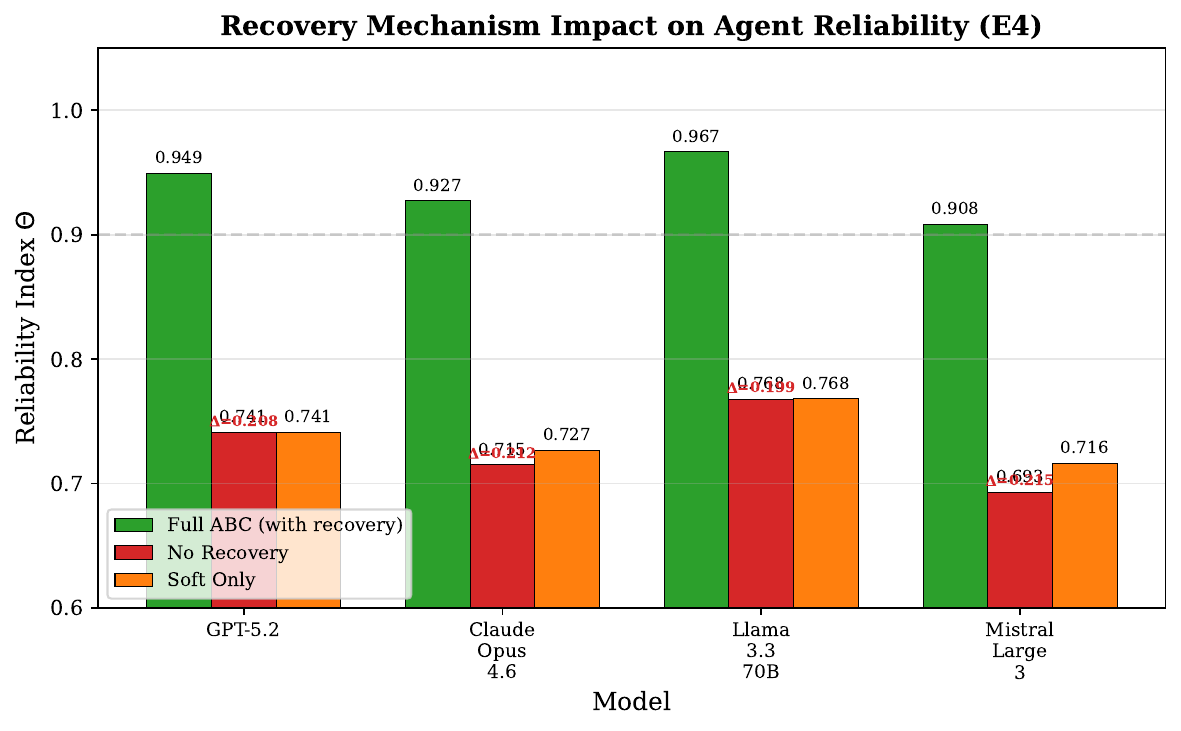}
\caption{%
    Recovery mechanism impact on agent reliability (E4 data).  Full \ABC{} (with recovery) achieves $\Theta = 0.908$--$0.967$ across models, while removing recovery degrades $\Theta$ by $0.199$--$0.215$ (mean $-0.209$).  The consistent $\sim$0.20 degradation across models with different baseline capabilities confirms that recovery contribution is an architectural property of \ABC{}, not a model-specific artifact.%
}
\label{fig:recovery-bounds}
\end{figure}

\begin{lemma}[Recovery-Augmented Compliance]
\label{lem:recovery}
Let $q \in (0,1)$ denote the per-step compliance probability and
$r \in [0,1]$ the recovery effectiveness (probability that a violation is
corrected within $k$ recovery steps). Then:
\begin{enumerate}[label=(\roman*)]
  \item \textbf{Without recovery:}
    $\;\Pr[\text{compliance over } T \text{ steps}] = q^T$.
  \item \textbf{With recovery:}
    $\;\Pr[\text{recoverable compliance}]
      \geq 1 - T(1-q)(1-r)$.
\end{enumerate}
\end{lemma}

\begin{proof}
At each discrete time step $t \in \{0, 1, \ldots, T-1\}$, define the events:
\begin{align*}
  V_t &\coloneqq \bigl\{C_{\mathrm{soft}}(t) < 1 - \delta\bigr\}
    && \text{(violation at step } t\text{)}, \\
  F_t &\coloneqq V_t \cap \bigl\{\text{recovery fails within } k
    \text{ steps}\bigr\}
    && \text{(unrecoverable failure at step } t\text{)}.
\end{align*}

\textbf{Part~(i).}
Without recovery, compliance over $T$ steps requires $V_t^c$ (no violation)
at every step. Since each step succeeds independently with probability $q$:
\[
  \Pr\!\left[\bigcap_{t=0}^{T-1} V_t^c\right] = q^T.
\]

\textbf{Part~(ii).}
With recovery, we have
\[
  \Pr(V_t) = 1 - q, \qquad
  \Pr(\text{recovery fails} \mid V_t) = 1 - r.
\]
By conditional probability,
$\Pr(F_t) = \Pr(V_t)\cdot\Pr(\text{recovery fails} \mid V_t) = (1-q)(1-r)$
for each step $t$.

The soft compliance guarantee fails if and only if there exists some step $t$
at which an unrecoverable failure occurs. By the union bound:
\[
  \Pr\!\bigl(\exists\, t \in \{0,\ldots,T-1\} : F_t\bigr)
  \leq \sum_{t=0}^{T-1} \Pr(F_t)
  = T\,(1-q)(1-r).
\]
Taking the complement:
\[
  \Pr(\text{soft guarantee holds})
  \geq 1 - T(1-q)(1-r). \qedhere
\]
\end{proof}

\begin{remark}[Tightness of the Union Bound]
\label{rem:recovery-tightness}
\Cref{fig:recovery-bounds} illustrates the practical impact of recovery on
agent reliability across models, confirming the theoretical bounds derived above.
The union bound in Lemma~\ref{lem:recovery} is conservative because violations
and recoveries create \emph{negative autocorrelation}: a successful recovery at
step $t$ makes compliance at step $t+1$ more likely (the system has just been
corrected). Tighter bounds using renewal theory yield an expected violation
fraction of
\[
  \frac{(1-q)\cdot \E[\tau_{\mathrm{recovery}}]}
       {\E[\tau_{\mathrm{inter\text{-}violation}}]
        + \E[\tau_{\mathrm{recovery}}]},
\]
where $\tau_{\mathrm{recovery}}$ is the recovery time and
$\tau_{\mathrm{inter\text{-}violation}}$ is the time between successive
violations. This renewal-theoretic bound is tight as $T \to \infty$.
\end{remark}

%% ============================================================
%% A.3: Compositionality Theorem
%% ============================================================
\subsection{Proof of the Compositionality Theorem}
\label{appendix:compositionality}

\begin{theorem}[Deterministic Contract Composition]
\label{thm:compositionality-proof}
Let agents $A$ and $B$ satisfy contracts $C_A$ and $C_B$ respectively, i.e.,
$A \models C_A$ and $B \models C_B$. Under conditions:
\begin{enumerate}[label=\textup{(C\arabic*)}]
  \item \label{cond:handoff}
    \textbf{Interface compatibility:}
    A handoff invariant $I_{\mathrm{handoff}}$ is maintained at the boundary
    between $A$ and $B$.
  \item \label{cond:pre-post}
    \textbf{Pre/postcondition chaining:}
    $\mathrm{PostCond}_A \wedge I_{\mathrm{handoff}} \Rightarrow P_B$
    (A's postcondition plus the handoff invariant implies B's precondition).
  \item \label{cond:governance-appendix}
    \textbf{Governance compatibility:}
    $G_A \cup G_B$ contains no conflicting governance constraints.
  \item \label{cond:recovery-appendix}
    \textbf{Recovery isolation:}
    $R_A$ does not violate $P_B$, and $R_B$ does not violate $I_A$.
\end{enumerate}
Then $\mathrm{Chain}(A, B) \models C_{A \oplus B}$, where $C_{A \oplus B}$ is
the composed contract with:
\begin{align*}
  P_{A \oplus B} &= P_A, &
  I_{A \oplus B} &= I_A \wedge I_B \wedge I_{\mathrm{handoff}}, \\
  G_{A \oplus B} &= G_A \cup G_B, &
  R_{A \oplus B} &= \mathrm{compose}(R_A, R_B, R_{\mathrm{cascade}}).
\end{align*}
\end{theorem}

\begin{proof}
We verify each component of the composed contract $C_{A \oplus B}$.

\medskip\noindent
\textbf{Step 1: Preconditions.}\quad
The composed system's precondition is $P_{A \oplus B} = P_A$. Since the
environment satisfies $P_A$ and $A \models C_A$, agent $A$ executes within its
contract. This establishes the entry condition for the chain.

\medskip\noindent
\textbf{Step 2: Pre/postcondition chaining.}\quad
Since $A \models C_A$, the postcondition $\mathrm{PostCond}_A$ holds upon
$A$'s completion. By \ref{cond:handoff}, the handoff invariant
$I_{\mathrm{handoff}}$ holds at the boundary. By \ref{cond:pre-post},
$\mathrm{PostCond}_A \wedge I_{\mathrm{handoff}} \Rightarrow P_B$. Therefore
$P_B$ holds and $B$ can execute within its contract $C_B$.

\medskip\noindent
\textbf{Step 3: Invariant preservation.}\quad
The composed invariant is
$I_{A \oplus B} = I_A \wedge I_B \wedge I_{\mathrm{handoff}}$. We verify each
conjunct:
\begin{itemize}
  \item $A \models C_A$ implies $I_A$ holds throughout $A$'s execution phase.
  \item $B \models C_B$ implies $I_B$ holds throughout $B$'s execution phase.
  \item $I_{\mathrm{handoff}}$ holds by \ref{cond:handoff}.
\end{itemize}
Therefore $I_{A \oplus B}$ holds throughout the chain's execution.

\medskip\noindent
\textbf{Step 4: Governance respect.}\quad
The composed governance set is $G_{A \oplus B} = G_A \cup G_B$. By
\ref{cond:governance-appendix}, this union is conflict-free. Since $A \models C_A$
implies $G_A$ is respected and $B \models C_B$ implies $G_B$ is respected, the
full governance set $G_{A \oplus B}$ is respected by the chain.

\medskip\noindent
\textbf{Step 5: Recovery composition.}\quad
The composed recovery mechanism is
$R_{A \oplus B} = \mathrm{compose}(R_A, R_B, R_{\mathrm{cascade}})$. By
\ref{cond:recovery-appendix}, $R_A$ does not violate $P_B$ and $R_B$ does not violate
$I_A$. Therefore recovery in either agent preserves the other agent's contract
state. The cascade recovery mechanism $R_{\mathrm{cascade}}$ handles
cross-boundary effects by construction.

\medskip
Since all five components---preconditions, postcondition chaining, invariants,
governance, and recovery---are verified, we conclude
$\mathrm{Chain}(A, B) \models C_{A \oplus B}$.
\end{proof}

%% ============================================================
%% A.4: Probabilistic Compositionality
%% ============================================================
\subsection{Proof of Probabilistic Compositionality}
\label{appendix:prob-compositionality}

\begin{definition}[$(p, \delta)$-Satisfaction]
\label{def:p-delta-satisfaction}
An agent $A$ \emph{$(p, \delta)$-satisfies} a contract $C$ if $A$ satisfies
$C$ with probability at least $p$, allowing behavioral deviation at most
$\delta$ from the contract's nominal specification.
\end{definition}

\begin{theorem}[Probabilistic Contract Composition]
\label{thm:prob-compositionality-proof}
Suppose agent $A$ $(p_A, \delta_A)$-satisfies $C_A$, agent $B$
$(p_B, \delta_B)$-satisfies $C_B$, and the handoff between $A$ and $B$
succeeds with probability $p_h$ introducing deviation $\delta_h$. Then
$\mathrm{Chain}(A, B)$ $(p_{A \oplus B},\, \delta_{A \oplus B})$-satisfies
$C_{A \oplus B}$ with:
\begin{align}
  p_{A \oplus B} &\geq p_A \cdot p_B \cdot p_h,
  \label{eq:prob-composition} \\
  \delta_{A \oplus B} &\leq \delta_A + \delta_B + \delta_h.
  \label{eq:deviation-composition}
\end{align}
\end{theorem}

\begin{proof}
Define the following events:
\begin{align*}
  E_A &\coloneqq \{A \text{ satisfies } C_A\}, \\
  E_B &\coloneqq \{B \text{ satisfies } C_B\}, \\
  E_h &\coloneqq \{\text{handoff preserves interface compatibility}\}.
\end{align*}
The composed chain succeeds if and only if all three events occur:
$E_A \cap E_h \cap E_B$.

\medskip\noindent
\textbf{Probability bound.}\quad
We decompose using the chain rule of conditional probability:
\[
  \Pr(E_A \cap E_h \cap E_B)
  = \Pr(E_A) \cdot \Pr(E_h \mid E_A) \cdot \Pr(E_B \mid E_A \cap E_h).
\]
Under the conditional independence assumption---that $B$'s behavior given
correct input is independent of $A$'s internal execution---we have:
\begin{itemize}
  \item $\Pr(E_A) \geq p_A$ \quad (by $A$'s contract satisfaction),
  \item $\Pr(E_h \mid E_A) \geq p_h$ \quad (handoff success probability),
  \item $\Pr(E_B \mid E_A \cap E_h) \geq p_B$ \quad (by $B$'s contract
    satisfaction, given correct input from a successful handoff).
\end{itemize}
Therefore $p_{A \oplus B} \geq p_A \cdot p_h \cdot p_B$.

\medskip\noindent
\textbf{Deviation bound.}\quad
In the worst case, deviations accumulate additively across the chain. Agent $A$
introduces deviation at most $\delta_A$ from nominal, the handoff introduces
at most $\delta_h$, and agent $B$ introduces at most $\delta_B$. By the
sub-additivity via union bound on the per-stage deviations:
\[
  \delta_{A \oplus B} \leq \delta_A + \delta_h + \delta_B.
\]

\medskip\noindent
\textbf{Extension to $N$ agents.}\quad
By induction on chain length, for $N$ agents $A_1, \ldots, A_N$ with handoffs
$h_1, \ldots, h_{N-1}$:
\begin{align}
  p_{\mathrm{chain}} &\geq \prod_{i=1}^{N} p_i
    \cdot \prod_{j=1}^{N-1} p_{h_j},
  \label{eq:n-agent-prob} \\
  \delta_{\mathrm{chain}} &\leq \sum_{i=1}^{N} \delta_i
    + \sum_{j=1}^{N-1} \delta_{h_j}.
  \label{eq:n-agent-deviation}
\end{align}
The inductive step applies Theorem~\ref{thm:prob-compositionality-proof} to
$\mathrm{Chain}(A_1, \ldots, A_{k})$ and $A_{k+1}$, treating the existing
chain as a single agent with composed satisfaction parameters.
\end{proof}

\begin{remark}[Tightness and Practical Implications]
\label{rem:prob-tightness}
The probability bound \eqref{eq:prob-composition} is tight when events are
independent, but conservative under positive correlation (e.g., when both
agents benefit from the same favorable environment state). The deviation bound
\eqref{eq:deviation-composition} is tight in the adversarial case but typically
loose in practice due to cancellation effects. The $N$-agent extension
\eqref{eq:n-agent-prob} reveals that reliability degrades multiplicatively with
chain length, motivating the use of checkpointing and recovery mechanisms at
intermediate handoff points for long chains.
\end{remark}

%% ============================================================
%% A.5: Sample Complexity
%% ============================================================
\subsection{Sample Complexity for $(p, \delta, k)$-Satisfaction Certification}
\label{appendix:sample-complexity}

A critical practical question is: \emph{how many test sessions are required to
certify that an agent $(p, \delta, k)$-satisfies its contract?} We establish a
baseline via Hoeffding's inequality and then show that sequential testing
dramatically reduces the required sample size.

\begin{proposition}[Hoeffding Baseline]
\label{prop:hoeffding}
To estimate the compliance probability $p$ within additive error $\varepsilon$
with confidence $1 - \alpha$ using i.i.d.\ Bernoulli observations, the
required sample size is
\begin{equation}
  n \geq \frac{1}{2\varepsilon^2}\,\ln\frac{2}{\alpha}.
  \label{eq:hoeffding-bound}
\end{equation}
\end{proposition}

\begin{proof}
Let $X_1, \ldots, X_n$ be i.i.d.\ Bernoulli$(p)$ random variables indicating
per-session compliance, and let $\hat{p}_n = \frac{1}{n}\sum_{i=1}^n X_i$.
By Hoeffding's inequality:
\[
  \Pr\!\bigl(|\hat{p}_n - p| \geq \varepsilon\bigr)
  \leq 2\exp\!\bigl(-2n\varepsilon^2\bigr).
\]
Setting the right-hand side equal to $\alpha$ and solving for $n$:
\[
  2\exp(-2n\varepsilon^2) = \alpha
  \quad\Longrightarrow\quad
  n = \frac{1}{2\varepsilon^2}\,\ln\frac{2}{\alpha}.
\]
For $\varepsilon = 0.01$ and $\alpha = 0.05$: $\;n \geq
\frac{1}{2(0.01)^2}\ln\frac{2}{0.05} = 5000 \cdot \ln 40 \approx 18{,}445$.
\end{proof}

\begin{figure}[t]
\centering
\includegraphics[width=\columnwidth]{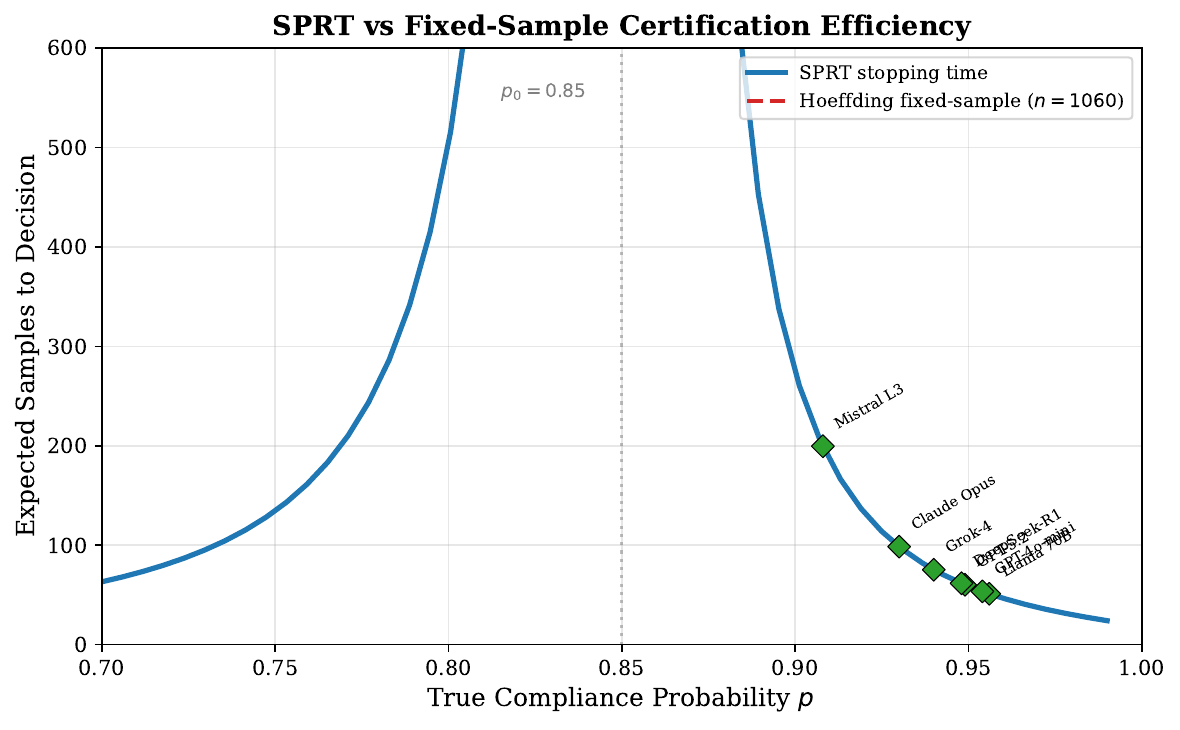}
\caption{%
    SPRT vs.\ fixed-sample certification efficiency.  The Sequential Probability Ratio Test requires significantly fewer samples than Hoeffding fixed-sample bounds to certify $(p, \delta, k)$-satisfaction.  Diamond markers show the stopping times for each model at their observed E1 compliance rates ($\Theta = 0.908$--$0.956$), demonstrating that agents with higher compliance are certified faster.%
}
\label{fig:sprt-stopping}
\end{figure}

\begin{proposition}[SPRT Improvement]
\label{prop:sprt}
Consider Wald's Sequential Probability Ratio Test (SPRT) for testing
\[
  H_0\colon p \leq p_0 = 0.90
  \qquad\text{vs.}\qquad
  H_1\colon p \geq p_1 = 0.95
\]
with Type~I and Type~II error rates $\alpha = \beta = 0.05$. The expected
sample size under $H_1$ is approximately $150$--$300$ sessions, representing a
$60\times$--$120\times$ reduction over the Hoeffding baseline.
\end{proposition}

\begin{proof}[Proof sketch]
The SPRT maintains the log-likelihood ratio
\[
  \Lambda_n = \sum_{i=1}^n \ln\frac{\Pr(X_i \mid p_1)}{\Pr(X_i \mid p_0)}
  = \sum_{i=1}^n \left[
    X_i \ln\frac{p_1}{p_0} + (1 - X_i)\ln\frac{1 - p_1}{1 - p_0}
  \right]
\]
and terminates when $\Lambda_n$ exits the continuation region
$\bigl(\ln\frac{\beta}{1-\alpha},\;\ln\frac{1-\beta}{\alpha}\bigr)$.

Under $H_1$ (true $p = p_1$), the expected increment per observation is the
Kullback--Leibler divergence
\[
  \E_{p_1}[\Lambda_1]
  = \KL(p_1 \,\|\, p_0)
  = p_1 \ln\frac{p_1}{p_0} + (1-p_1)\ln\frac{1-p_1}{1-p_0}.
\]
For $p_0 = 0.90$ and $p_1 = 0.95$:
$\;\KL(0.95 \,\|\, 0.90) = 0.95\ln\frac{0.95}{0.90} + 0.05\ln\frac{0.05}{0.10} \approx 0.01671$ nats.

Wald's approximation for the expected sample size under $H_1$ gives
\[
  \E_{p_1}[N]
  \approx \frac{(1-\beta)\ln\frac{1-\beta}{\alpha}
    + \beta\ln\frac{\beta}{1-\alpha}}{\KL(p_1 \,\|\, p_0)}
  \approx \frac{0.95 \cdot \ln 19 + 0.05 \cdot \ln(1/19)}{0.01671}
  \approx 159.
\]
The range $150$--$300$ accounts for boundary overshoot and discrete-sample
effects that cause the actual stopping time to deviate from Wald's
continuous approximation (\Cref{fig:sprt-stopping}).

The optimality of this approach follows from the Wald--Wolfowitz theorem
\citep{wald1948optimum}: \emph{among all sequential tests with Type~I error
$\leq \alpha$ and Type~II error $\leq \beta$, the SPRT minimizes the expected
sample size under both $H_0$ and $H_1$.} This fundamental result guarantees
that no sequential testing procedure can certify agent compliance with fewer
expected observations than the SPRT.
\end{proof}

\begin{remark}[Practical Certification Protocol]
\label{rem:certification-protocol}
The SPRT reduction from ${\sim}18{,}445$ to ${\sim}150$--$300$ sessions makes
runtime certification practical for deployed agent systems. In practice, the
test is run as a \emph{continuous monitoring} process: each agent interaction
constitutes one Bernoulli trial, and the SPRT statistic $\Lambda_n$ is updated
incrementally. When the statistic crosses the upper boundary, the agent is
certified; when it crosses the lower boundary, the agent is flagged for
remediation. This sequential approach naturally accommodates non-stationary
compliance rates via windowed or decaying variants of the SPRT.
\end{remark}

%% file: bib/references.bib
@article{leoveanu2025dbc,
  author    = {Claudiu Leoveanu-Condrei},
  title     = {A {DbC} Inspired Neurosymbolic Layer for Trustworthy Agent Design},
  journal   = {arXiv preprint arXiv:2508.03665},
  year      = {2025},
  eprint    = {2508.03665},
  archiveprefix = {arXiv},
  primaryclass  = {cs.LG},
  note      = {4 pages, 1 figure},
}

@inproceedings{wang2025agentspec,
  author    = {Haoyu Wang and Christopher M. Poskitt and Jun Sun},
  title     = {{AgentSpec}: Customizable Runtime Enforcement for Safe and Reliable {LLM} Agents},
  booktitle = {Proceedings of the 48th IEEE/ACM International Conference on Software Engineering (ICSE)},
  year      = {2026},
  eprint    = {2503.18666},
  archiveprefix = {arXiv},
  primaryclass  = {cs.AI},
}

@article{wang2025pro2guard,
  author    = {Haoyu Wang and Christopher M. Poskitt and Jun Sun and Jiali Wei},
  title     = {{Pro2Guard}: Proactive Runtime Enforcement of {LLM} Agent Safety via Probabilistic Model Checking},
  journal   = {arXiv preprint arXiv:2508.00500},
  year      = {2025},
  eprint    = {2508.00500},
  archiveprefix = {arXiv},
  primaryclass  = {cs.AI},
}

@article{rath2026agentdrift,
  author    = {Abhishek Rath},
  title     = {Agent Drift: Quantifying Behavioral Degradation in Multi-Agent {LLM} Systems Over Extended Interactions},
  journal   = {arXiv preprint arXiv:2601.04170},
  year      = {2026},
  eprint    = {2601.04170},
  archiveprefix = {arXiv},
  primaryclass  = {cs.AI},
}

@article{li2017stochastic,
  author    = {Jiwei Li and Pierluigi Nuzzo and Alberto Sangiovanni-Vincentelli and Yugeng Xi and Dewei Li},
  title     = {Stochastic Assume-Guarantee Contracts for Cyber-Physical System Design Under Probabilistic Requirements},
  journal   = {arXiv preprint arXiv:1705.09316},
  year      = {2017},
  eprint    = {1705.09316},
  archiveprefix = {arXiv},
  primaryclass  = {eess.SY},
}

@article{meyer1992dbc,
  author    = {Bertrand Meyer},
  title     = {Applying ``Design by Contract''},
  journal   = {Computer},
  volume    = {25},
  number    = {10},
  pages     = {40--51},
  year      = {1992},
  publisher = {IEEE},
  doi       = {10.1109/2.161279},
}

@book{meyer1997oosc,
  author    = {Bertrand Meyer},
  title     = {Object-Oriented Software Construction},
  edition   = {2nd},
  publisher = {Prentice Hall},
  year      = {1997},
  isbn      = {0136291554},
}

@article{leavens2006jml,
  author    = {Gary T. Leavens and Albert L. Baker and Clyde Ruby},
  title     = {Preliminary Design of {JML}: A Behavioral Interface Specification Language for {Java}},
  journal   = {ACM SIGSOFT Software Engineering Notes},
  volume    = {31},
  number    = {3},
  pages     = {1--38},
  year      = {2006},
  publisher = {ACM},
  doi       = {10.1145/1127878.1127884},
}

@inproceedings{barnett2004specsharp,
  author    = {Mike Barnett and K. Rustan M. Leino and Wolfram Schulte},
  title     = {The {Spec\#} Programming System: An Overview},
  booktitle = {Proceedings of the International Workshop on Construction and Analysis of Safe, Secure, and Interoperable Smart Devices (CASSIS)},
  series    = {Lecture Notes in Computer Science},
  volume    = {3362},
  pages     = {49--69},
  year      = {2004},
  publisher = {Springer},
  doi       = {10.1007/978-3-540-30569-9_3},
}

@article{benveniste2018contracts,
  author    = {Albert Benveniste and Beno{\^i}t Caillaud and Dejan Nickovic and Roberto Passerone and Jean-Baptiste Raclet and Philipp Reinkemeier and Alberto Sangiovanni-Vincentelli and Werner Damm and Thomas A. Henzinger and Kim G. Larsen},
  title     = {Contracts for System Design},
  journal   = {Foundations and Trends in Electronic Design Automation},
  volume    = {12},
  number    = {2--3},
  pages     = {124--400},
  year      = {2018},
  publisher = {Now Publishers},
  doi       = {10.1561/1000000053},
}

@inproceedings{henzinger1998assume,
  author    = {Thomas A. Henzinger and Shaz Qadeer and Sriram K. Rajamani},
  title     = {You Assume, We Guarantee: Methodology and Case Studies},
  booktitle = {Proceedings of the 10th International Conference on Computer Aided Verification (CAV)},
  series    = {Lecture Notes in Computer Science},
  volume    = {1427},
  pages     = {440--451},
  year      = {1998},
  publisher = {Springer},
  doi       = {10.1007/BFb0028765},
}

@inproceedings{hampus2024probabilistic,
  author    = {Anton Hampus and Mattias Nyberg},
  title     = {A Theory of Probabilistic Contracts},
  booktitle = {Proceedings of the International Symposium on Leveraging Applications of Formal Methods (ISoLA)},
  pages     = {296--319},
  year      = {2024},
  publisher = {Springer},
  doi       = {10.1007/978-3-031-75380-0_17},
}

@article{hansson1994logic,
  author    = {Hans Hansson and Bengt Jonsson},
  title     = {A Logic for Reasoning about Time and Reliability},
  journal   = {Formal Aspects of Computing},
  volume    = {6},
  number    = {5},
  pages     = {512--535},
  year      = {1994},
  publisher = {Springer},
  doi       = {10.1007/BF01211866},
}

@article{wald1948optimum,
  author    = {Abraham Wald and Jacob Wolfowitz},
  title     = {Optimum Character of the Sequential Probability Ratio Test},
  journal   = {The Annals of Mathematical Statistics},
  volume    = {19},
  number    = {3},
  pages     = {326--339},
  year      = {1948},
  publisher = {Institute of Mathematical Statistics},
  doi       = {10.1214/aoms/1177730197},
}

@book{meyn1993stability,
  author    = {Sean P. Meyn and Richard L. Tweedie},
  title     = {Markov Chains and Stochastic Stability},
  publisher = {Springer-Verlag},
  year      = {1993},
  address   = {London},
  doi       = {10.1007/978-1-4471-3267-7},
}

@article{uhlenbeck1930theory,
  author    = {George E. Uhlenbeck and Leonard S. Ornstein},
  title     = {On the Theory of the {Brownian} Motion},
  journal   = {Physical Review},
  volume    = {36},
  number    = {5},
  pages     = {823--841},
  year      = {1930},
  publisher = {American Physical Society},
  doi       = {10.1103/PhysRev.36.823},
}

@article{endres2003new,
  author    = {Dominik M. Endres and Johannes E. Schindelin},
  title     = {A New Metric for Probability Distributions},
  journal   = {IEEE Transactions on Information Theory},
  volume    = {49},
  number    = {7},
  pages     = {1858--1860},
  year      = {2003},
  publisher = {IEEE},
  doi       = {10.1109/TIT.2003.813506},
}

@article{osterreicher2003new,
  author    = {Ferdinand {\"O}sterreicher and Igor Vajda},
  title     = {A New Class of Metric Divergences on Probability Spaces and Its Applicability in Statistics},
  journal   = {Annals of the Institute of Statistical Mathematics},
  volume    = {55},
  number    = {3},
  pages     = {639--653},
  year      = {2003},
  publisher = {Springer},
  doi       = {10.1007/BF02517812},
}

@article{gama2014conceptdrift,
  author    = {Jo{\~a}o Gama and Indre Zliobaite and Albert Bifet and Mykola Pechenizkiy and Abdelhamid Bouchachia},
  title     = {A Survey on Concept Drift Adaptation},
  journal   = {ACM Computing Surveys},
  volume    = {46},
  number    = {4},
  pages     = {1--37},
  year      = {2014},
  publisher = {ACM},
  doi       = {10.1145/2523813},
  articleno = {44},
}

@article{wu2023autogen,
  author    = {Qingyun Wu and Gagan Bansal and Jieyu Zhang and Yiran Wu and Beibin Li and Erkang Zhu and Li Jiang and Xiaoyun Zhang and Shaokun Zhang and Jiale Liu and Ahmed Hassan Awadallah and Ryen W. White and Doug Burger and Chi Wang},
  title     = {{AutoGen}: Enabling Next-Gen {LLM} Applications via Multi-Agent Conversation},
  journal   = {arXiv preprint arXiv:2308.08155},
  year      = {2023},
  eprint    = {2308.08155},
  archiveprefix = {arXiv},
  primaryclass  = {cs.AI},
}

@misc{chase2023langchain,
  author    = {Harrison Chase},
  title     = {{LangChain}},
  year      = {2023},
  howpublished = {\url{https://github.com/langchain-ai/langchain}},
  note      = {Open-source framework for LLM application development},
}

@misc{moura2024crewai,
  author    = {Jo{\~a}o Moura},
  title     = {{CrewAI}: Framework for Orchestrating Role-Playing {AI} Agents},
  year      = {2024},
  howpublished = {\url{https://github.com/joaomdmoura/crewai}},
  note      = {Multi-agent orchestration framework},
}

@inproceedings{yao2023react,
  author    = {Shunyu Yao and Jeffrey Zhao and Dian Yu and Nan Du and Izhak Shafran and Karthik Narasimhan and Yuan Cao},
  title     = {{ReAct}: Synergizing Reasoning and Acting in Language Models},
  booktitle = {Proceedings of the International Conference on Learning Representations (ICLR)},
  year      = {2023},
  eprint    = {2210.03629},
  archiveprefix = {arXiv},
  primaryclass  = {cs.CL},
}

@article{schick2023toolformer,
  author    = {Timo Schick and Jane Dwivedi-Yu and Roberto Dess{\`i} and Roberta Raileanu and Maria Lomeli and Luke Zettlemoyer and Nicola Cancedda and Thomas Scialom},
  title     = {Toolformer: Language Models Can Teach Themselves to Use Tools},
  journal   = {arXiv preprint arXiv:2302.04761},
  year      = {2023},
  eprint    = {2302.04761},
  archiveprefix = {arXiv},
  primaryclass  = {cs.CL},
}

@article{bai2022constitutional,
  author    = {Yuntao Bai and Saurav Kadavath and Sandipan Kundu and Amanda Askell and Jackson Kernion and Andy Jones and Anna Chen and Anna Goldie and Azalia Mirhoseini and Cameron McKinnon and Carol Chen and Catherine Olsson and Christopher Olah and Danny Hernandez and Dawn Drain and Deep Ganguli and Dustin Li and Eli Tran-Johnson and Ethan Perez and Jamie Kerr and Jared Mueller and Jeffrey Ladish and Joshua Landau and Kamal Ndousse and Kamile Lukosuite and Liane Lovitt and Michael Sellitto and Nelson Elhage and Nicholas Schiefer and Noemi Mercado and Nova DasSarma and Robert Lasenby and Robin Larson and Sam Ringer and Scott Johnston and Shauna Kravec and Sheer El Showk and Stanislav Fort and Tamera Lanham and Timothy Telleen-Lawton and Tom Conerly and Tom Henighan and Tristan Hume and Samuel R. Bowman and Zac Hatfield-Dodds and Ben Mann and Dario Amodei and Nicholas Joseph and Sam McCandlish and Tom Brown and Jared Kaplan},
  title     = {Constitutional {AI}: Harmlessness from {AI} Feedback},
  journal   = {arXiv preprint arXiv:2212.08073},
  year      = {2022},
  eprint    = {2212.08073},
  archiveprefix = {arXiv},
  primaryclass  = {cs.CL},
}

@article{ouyang2022instructgpt,
  author    = {Long Ouyang and Jeff Wu and Xu Jiang and Diogo Almeida and Carroll L. Wainwright and Pamela Mishkin and Chong Zhang and Sandhini Agarwal and Katarina Slama and Alex Ray and John Schulman and Jacob Hilton and Fraser Kelton and Luke Miller and Maddie Simens and Amanda Askell and Peter Welinder and Paul Christiano and Jan Leike and Ryan Lowe},
  title     = {Training Language Models to Follow Instructions with Human Feedback},
  journal   = {arXiv preprint arXiv:2203.02155},
  year      = {2022},
  eprint    = {2203.02155},
  archiveprefix = {arXiv},
  primaryclass  = {cs.CL},
}

@inproceedings{rebedea2023nemo,
  author    = {Traian Rebedea and Razvan Dinu and Makesh Sreedhar and Christopher Parisien and Jonathan Cohen},
  title     = {{NeMo Guardrails}: A Toolkit for Controllable and Safe {LLM} Applications with Programmable Rails},
  booktitle = {Proceedings of the 2023 Conference on Empirical Methods in Natural Language Processing: System Demonstrations (EMNLP Demo)},
  year      = {2023},
  eprint    = {2310.10501},
  archiveprefix = {arXiv},
  primaryclass  = {cs.CL},
}

@article{leucker2009rv,
  author    = {Martin Leucker and Christian Schallhart},
  title     = {A Brief Account of Runtime Verification},
  journal   = {The Journal of Logic and Algebraic Programming},
  volume    = {78},
  number    = {5},
  pages     = {293--303},
  year      = {2009},
  publisher = {Elsevier},
  doi       = {10.1016/j.jlap.2008.08.004},
}

@article{bauer2011runtime,
  author    = {Andreas Bauer and Martin Leucker and Christian Schallhart},
  title     = {Runtime Verification for {LTL} and {TLTL}},
  journal   = {ACM Transactions on Software Engineering and Methodology (TOSEM)},
  volume    = {20},
  number    = {4},
  pages     = {1--64},
  year      = {2011},
  publisher = {ACM},
  doi       = {10.1145/2000799.2000800},
}

@article{liang2022helm,
  author    = {Percy Liang and Rishi Bommasani and Tony Lee and Dimitris Tsipras and Dilara Soylu and Michihiro Yasunaga and Yian Zhang and Deepak Narayanan and Yuhuai Wu and Ananya Kumar and Benjamin Newman and Binhang Yuan and Bobby Yan and Ce Zhang and Christian Cosgrove and Christopher D. Manning and Christopher R{\'e} and Diana Acosta-Navas and Drew A. Hudson and Eric Zelikman and Esin Durmus and Faisal Ladhak and Frieda Rong and Hongyu Ren and Huaxiu Yao and Jue Wang and Keshav Santhanam and Laurel Orr and Lucia Zheng and Mert Yuksekgonul and Mirac Suzgun and Nathan Kim and Neel Guha and Niladri Chatterji and Omar Khattab and Peter Henderson and Qian Huang and Ryan Chi and Sang Michael Xie and Shibani Santurkar and Surya Ganguli and Tatsunori Hashimoto and Thomas Icard and Tianyi Zhang and Vishrav Chaudhary and William Wang and Xuechen Li and Yifan Mai and Yuhui Zhang and Yuta Koreeda},
  title     = {Holistic Evaluation of Language Models},
  journal   = {Transactions on Machine Learning Research (TMLR)},
  year      = {2023},
  eprint    = {2211.09110},
  archiveprefix = {arXiv},
  primaryclass  = {cs.CL},
}

@inproceedings{liu2023agentbench,
  author    = {Xiao Liu and Hao Yu and Hanchen Zhang and Yifan Xu and Xuanyu Lei and Hanyu Lai and Yu Gu and Hangliang Ding and Kaiwen Men and Kejuan Yang and Shudan Zhang and Xiang Deng and Aohan Zeng and Zhengxiao Du and Chenhui Zhang and Sheng Shen and Tianjun Zhang and Yu Su and Huan Sun and Minlie Huang and Yuxiao Dong and Jie Tang},
  title     = {{AgentBench}: Evaluating {LLMs} as Agents},
  booktitle = {Proceedings of the International Conference on Learning Representations (ICLR)},
  year      = {2024},
  eprint    = {2308.03688},
  archiveprefix = {arXiv},
  primaryclass  = {cs.AI},
}

@article{hoare1969axiomatic,
  author    = {C. A. R. Hoare},
  title     = {An Axiomatic Basis for Computer Programming},
  journal   = {Communications of the ACM},
  volume    = {12},
  number    = {10},
  pages     = {576--580},
  year      = {1969},
  publisher = {ACM},
  doi       = {10.1145/363235.363259},
}

@article{amodei2016concrete,
  author    = {Dario Amodei and Chris Olah and Jacob Steinhardt and Paul Christiano and John Schulman and Dan Man{\'e}},
  title     = {Concrete Problems in {AI} Safety},
  journal   = {arXiv preprint arXiv:1606.06565},
  year      = {2016},
  eprint    = {1606.06565},
  archiveprefix = {arXiv},
  primaryclass  = {cs.AI},
}

@misc{weidinger2021ethical,
  author    = {Laura Weidinger and John Mellor and Maribeth Rauh and Conor Griffin and Jonathan Uesato and Po-Sen Huang and Myra Cheng and Mia Glaese and Borja Balle and Atoosa Kasirzadeh and Zac Kenton and Sasha Brown and Will Hawkins and Tom Stepleton and Courtney Biles and Abeba Birhane and Julia Haas and Laura Rimell and Lisa Anne Hendricks and William Isaac and Sean Legassick and Geoffrey Irving and Iason Gabriel},
  title     = {Ethical and Social Risks of Harm from Language Models},
  year      = {2021},
  eprint    = {2112.04359},
  archiveprefix = {arXiv},
  primaryclass  = {cs.CL},
}

@article{dong2025agentc,
  author    = {Yihe Dong and Zijie Zhang and Yuanpu Cao and Yijia Shao and Haoran Li},
  title     = {Agent-{C}: Scaling Structured Generation for Runtime Constraint Enforcement in {LLM} Agents},
  journal   = {arXiv preprint arXiv:2512.23738},
  year      = {2025},
  eprint    = {2512.23738},
  archiveprefix = {arXiv},
  primaryclass  = {cs.AI},
}

@article{naihin2026air,
  author    = {Zibo Xiao and Jun Sun and Junjie Chen},
  title     = {{AIR}: Improving Agent Safety through Incident Response},
  journal   = {arXiv preprint arXiv:2602.11749},
  year      = {2026},
  eprint    = {2602.11749},
  archiveprefix = {arXiv},
  primaryclass  = {cs.AI},
}

@article{hua2025agentsafe,
  author    = {Rafflesia Khan and Declan Joyce and Mansura Habiba},
  title     = {{AGENTSAFE}: A Unified Framework for Ethical Assurance and Governance in Agentic {AI}},
  journal   = {arXiv preprint arXiv:2512.03180},
  year      = {2025},
  eprint    = {2512.03180},
  archiveprefix = {arXiv},
  primaryclass  = {cs.AI},
}

@article{lee2026polaris,
  author    = {Zahra Moslemi and Keerthi Koneru and Yen-Ting Lee and Sheethal Kumar and Ramesh Radhakrishnan},
  title     = {{POLARIS}: Typed Planning and Governed Execution for Agentic {AI} in Back-Office Automation},
  journal   = {arXiv preprint arXiv:2601.11816},
  year      = {2026},
  eprint    = {2601.11816},
  archiveprefix = {arXiv},
  primaryclass  = {cs.AI},
  note      = {AAAI 2026 Workshop},
}

@article{kulkarni2026gap,
  author    = {Arnold Cartagena and Ariane Teixeira},
  title     = {Mind the {GAP}: Text Safety Does Not Transfer to Tool-Call Safety in {LLM} Agents},
  journal   = {arXiv preprint arXiv:2602.16943},
  year      = {2026},
  eprint    = {2602.16943},
  archiveprefix = {arXiv},
  primaryclass  = {cs.AI},
}

@article{li2026devil,
  author    = {Chenxu Wang and Chaozhuo Li and Songyang Liu and Zejian Chen and Jinyu Hou and Ji Qi and Rui Li and Litian Zhang and Qiwei Ye and Zheng Liu and Xu Chen and Xi Zhang and Philip S. Yu},
  title     = {The Devil Behind Moltbook: Anthropic Safety is Always Vanishing in Self-Evolving {AI} Societies},
  journal   = {arXiv preprint arXiv:2602.09877},
  year      = {2026},
  eprint    = {2602.09877},
  archiveprefix = {arXiv},
  primaryclass  = {cs.AI},
}

@article{guo2026stepshield,
  author    = {Gloria Felicia and Michael Eniolade and Jinfeng He and Zitha Sasindran and Hemant Kumar and Milan Hussain Angati and Sandeep Bandarupalli},
  title     = {{StepShield}: When, Not Whether to Intervene on Rogue Agents},
  journal   = {arXiv preprint arXiv:2601.22136},
  year      = {2026},
  eprint    = {2601.22136},
  archiveprefix = {arXiv},
  primaryclass  = {cs.AI},
}

@article{zhang2025veriguard,
  author    = {Lesly Miculicich and Mihir Parmar and Hamid Palangi and Krishnamurthy Dj Dvijotham and Mirko Montanari and Tomas Pfister and Long T. Le},
  title     = {{VeriGuard}: Enhancing {LLM} Agent Safety via Verified Code Generation},
  journal   = {arXiv preprint arXiv:2510.05156},
  year      = {2025},
  eprint    = {2510.05156},
  archiveprefix = {arXiv},
  primaryclass  = {cs.AI},
}

@article{ye2026agentcontracts,
  author    = {Qing Ye and Jing Tan},
  title     = {Agent Contracts: A Formal Framework for Resource-Bounded Autonomous {AI} Systems},
  journal   = {arXiv preprint arXiv:2601.08815},
  year      = {2026},
  eprint    = {2601.08815},
  archiveprefix = {arXiv},
  primaryclass  = {cs.AI},
}

@misc{guardrailsai2024,
  author    = {{Guardrails AI}},
  title     = {Guardrails: Adding Guardrails to Large Language Models},
  year      = {2024},
  howpublished = {\url{https://github.com/guardrails-ai/guardrails}},
  note      = {Open-source LLM output validation library},
}

@inproceedings{alshiekh2018shielding,
  author    = {Mohammed Alshiekh and Roderick Bloem and Ruediger Ehlers and Bettina K{\"o}nighofer and Scott Niekum and Ufuk Topcu},
  title     = {Safe Reinforcement Learning via Shielding},
  booktitle = {Proceedings of the AAAI Conference on Artificial Intelligence (AAAI)},
  year      = {2018},
  pages     = {2669--2678},
  eprint    = {1708.08611},
  archiveprefix = {arXiv},
  primaryclass  = {cs.AI},
}

@article{rath2024llm,
  author    = {Rath, Sudip},
  title     = {{LLM} Behavioral Stability: A Survey of Drift Detection and Measurement},
  journal   = {arXiv preprint arXiv:2404.00000},
  year      = {2024},
  note      = {Introduces the Agent Stability Index (ASI) for embedding-space drift detection},
}

@article{cihon2021ai,
  author    = {Cihon, Peter and Schuett, Jonas and Baum, Seth D.},
  title     = {{AI} Governance: A Research Agenda},
  journal   = {Minds and Machines},
  volume    = {31},
  number    = {1},
  pages     = {137--169},
  year      = {2021},
  publisher = {Springer},
}

@article{alpern1987recognizing,
  author    = {Alpern, Bowen and Schneider, Fred B.},
  title     = {Recognizing Safety and Liveness},
  journal   = {Distributed Computing},
  volume    = {2},
  number    = {3},
  pages     = {117--126},
  year      = {1987},
  publisher = {Springer},
}

@inproceedings{zheng2023judging,
  author    = {Lianmin Zheng and Wei-Lin Chiang and Ying Sheng and Siyuan Zhuang and Zhanghao Wu and Yonghao Zhuang and Zi Lin and Zhuohan Li and Dacheng Li and Eric P. Xing and Hao Zhang and Joseph E. Gonzalez and Ion Stoica},
  title     = {Judging {LLM-as-a-Judge} with {MT-Bench} and {Chatbot Arena}},
  booktitle = {Advances in Neural Information Processing Systems (NeurIPS)},
  volume    = {36},
  year      = {2023},
}

@inproceedings{dubois2024alpacafarm,
  author    = {Yann Dubois and Chen Xuechen Li and Rohan Taori and Tianyi Zhang and Ishaan Gulrajani and Jimmy Ba and Carlos Guestrin and Percy Liang and Tatsunori B. Hashimoto},
  title     = {{AlpacaFarm}: A Simulation Framework for Methods that Learn from Human Feedback},
  booktitle = {Advances in Neural Information Processing Systems (NeurIPS)},
  volume    = {36},
  year      = {2024},
}

@article{landis1977measurement,
  author    = {J. Richard Landis and Gary G. Koch},
  title     = {The Measurement of Observer Agreement for Categorical Data},
  journal   = {Biometrics},
  volume    = {33},
  number    = {1},
  pages     = {159--174},
  year      = {1977},
  publisher = {International Biometric Society},
}
